\documentclass[10pt,oneside]{article}

\usepackage[utf8]{inputenc} %
\usepackage[T1]{fontenc}    %
\usepackage{url}            %
\usepackage{booktabs}       %
\usepackage{amsfonts}       %
\usepackage{nicefrac}       %
\usepackage{microtype}      %

\usepackage{lmodern}
\usepackage{times}

\usepackage{amssymb,amsmath,amsthm}
\usepackage{bbold}
\usepackage[short]{optidef}

\usepackage{framed,multirow,multicol}

\usepackage{xcolor,xspace}
\usepackage{lscape}
\usepackage{subfigure,graphicx,epsfig,tikz,caption}
\usepackage[normalem]{ulem}
\usepackage{enumerate}
\usepackage{verbatim}
\usepackage{makeidx,latexsym}
\usepackage[colorlinks=true,citecolor=blue]{hyperref}%

\usepackage{bm}
\usepackage{stmaryrd}
\usepackage{lscape}

\usepackage[english]{babel}
\usepackage{bbm}
\usepackage{pgfplots}

\usepackage{parskip}
\usepackage[letterpaper,margin=1.1in]{geometry}

\usepackage[numbers,sort&compress,square,comma]{natbib}

\usepackage[utf8]{inputenc} %
\usepackage[T1]{fontenc}    %
\usepackage{hyperref}       %
\usepackage{url}            %
\usepackage{booktabs}       %
\usepackage{amsfonts}       %
\usepackage{nicefrac}       %
\usepackage{microtype}      %
\usepackage{xcolor}         %

\usepackage{algorithm,algpseudocode}
\usepackage{amssymb,amsmath,amsthm}
\newtheorem{theorem}{Theorem}[section]
\newtheorem{proposition}[theorem]{Proposition}

\newtheorem{lemma}[theorem]{Lemma}

\newcommand{\scale}{{S}}
\newcommand{\barn}{\bar{N}}

\newcommand{\I}{\mathcal{I}}
\newcommand{\J}{\mathcal{J}}
\newcommand{\R}{\mathcal{U}}
\newcommand{\RP}{\mathcal{P}}
\newcommand{\Tc}{T_{\text{max}}}
\newcommand{\Ts}{\tau}
\newcommand{\imax}{i^*}
\newcommand{\procname}{get\_job\_class}%
\newcommand{\pnalg}{{\sc PN $c\mu$ rule}}

\title{Scheduling jobs with stochastic holding costs}

\author{
	Dabeen Lee\thanks{Department of Industrial and Systems Engineering, KAIST, Daejeon 34141, Republic of Korea, \url{dabeenl@kaist.ac.kr}}
	\and
	Milan Vojnovic\thanks{Department of Statistics, London School of Economics, London, United Kingdom, \url{m.vojnovic@lse.ac.uk}}
}	

\date{\today}

\begin{document}

\maketitle

\begin{abstract}
	
	We study a single-server scheduling problem for the objective of minimizing the expected cumulative holding cost incurred by jobs, where parameters defining stochastic job holding costs are unknown to the scheduler. We consider a general setting allowing for different job classes, where jobs of the same class have statistically identical holding costs and service times, with an arbitrary number of jobs across classes. In each time step, the server can process a job and observes random holding costs of the jobs that are yet to be completed. We consider a learning-based $c\mu$ rule scheduling which starts with a preemption period of fixed duration, serving as a learning phase, and having gathered data about jobs, it switches to nonpreemptive scheduling. Our algorithms are designed to handle instances with large and small gaps in mean job holding costs and achieve near-optimal performance guarantees. The performance of algorithms is evaluated by regret, where the benchmark is the minimum possible total holding cost attained by the $c\mu$ rule scheduling policy when the parameters of jobs are known. We show regret lower bounds and algorithms that achieve nearly matching regret upper bounds. Our numerical results demonstrate the efficacy of our algorithms and show that our regret analysis is nearly tight.
\end{abstract}

\tableofcontents

\section{Introduction}

We consider a scheduling problem for jobs with stochastic holding costs which is described as follows: given a set of jobs, each %
incurring random cost over time steps until its completion with unknown mean value, make scheduling decisions of which job to process in each time step, with the objective of minimizing the expected total cumulative cost. Here, we need an algorithm that seamlessly integrates learning of mean job holding costs and scheduling.

The problem of scheduling jobs with time-varying holding costs arises in several different applications. In online social media platforms, content moderation requires scheduling of content review jobs, which have holding costs driven by the number of accumulated views as views of harmful content items represent a community-integrity cost \cite{fb}.
In data processing platforms, complex jobs are processed whose characteristics are often unknown in advance, but as %
the system learns more about the jobs' features, it may flexibly adjust scheduling decisions to serve jobs with high priority first \cite{decima}.
Another application is in optimizing energy consumption of servers in data centers, where a job waiting to be served uses energy-consuming resources \cite{power,energy-survey}. %
In emergency medical departments, patients undergo triage while being treated, and schedules for serving patients are flexibly adjusted depending on their conditions~\cite{stillman,triage-two-classes}. Note that patients' conditions may get worse while waiting, which corresponds to holding costs in our problem. For aircraft maintenance, diagnosing the conditions of parts and applying the required measures to repair them are conducted in a combined way~\cite{diagnostic}.

\begin{figure}[t!]
	\begin{center}
		\includegraphics[width=0.4\linewidth]{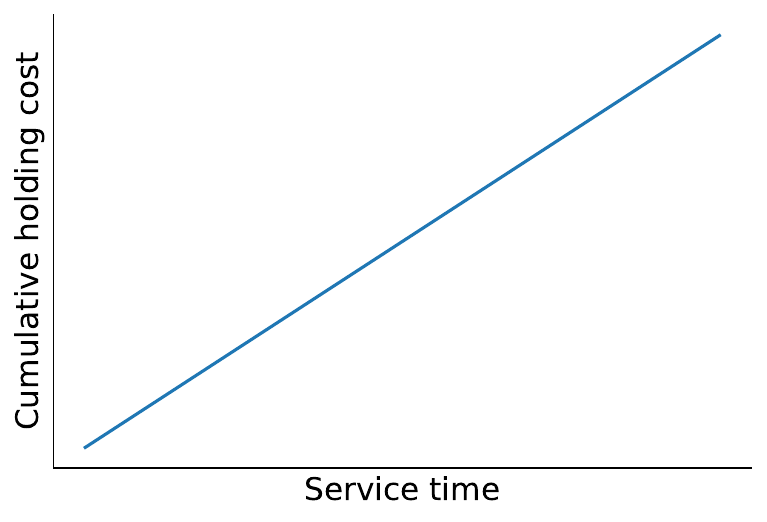}\hspace*{1.5cm}
		\includegraphics[width=0.4\linewidth]{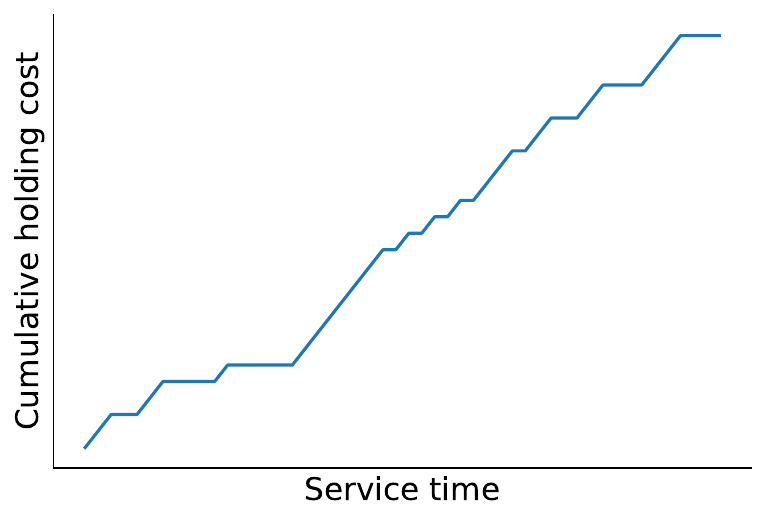}\hspace*{0.5cm}
		\caption{An illustration of cumulative job holding cost: (left) deterministic, (right) stochastic.}\label{fig:illustration}
	\end{center}
\end{figure}

We study a single server scheduling system where jobs incur independent holding costs, with each job having a time-varying holding cost according to a stochastic process with independent and identically distributed increments with unknown mean value, and independent service times with known mean values. Recent works, e.g. \cite{learning-cmu-rule}, started investigating queuing system control policies under uncertainty about jobs' mean service time parameters, where job holding costs are deterministic (linear) functions of stochastic job waiting times. In our problem setting, job holding costs are stochastic in a different way in that job holding costs themselves are according to some exogenous stochastic processes. For the aforementioned application scenarios, it is natural to model a job's holding costs by a stochastic process. As the first step towards understanding the case of stochastic job holding costs, we consider a single-server scheduling for a given set of jobs. An illustration of deterministic and stochastic job holding costs is shown in Figure~\ref{fig:illustration}. On the one hand, classic queuing system literature assumes stochastic job service times and deterministic job holding costs, which are proportional to job waiting times. On the other hand, we study systems where job service times are either deterministic or stochastic, and job holding costs are stochastic.

We consider a setting in which each job belongs to a job class with all jobs of the same class having identical mean holding costs and mean service times. The classes of jobs are known to the scheduler, and the information about job classes can be leveraged by the scheduler for learning mean job holding costs and making scheduling decisions. The number of distinct job classes is allowed to be arbitrary and so is the distribution of jobs over distinct job classes. Our framework also covers the case when the scheduler has no access to the information about job classes as a special case where jobs are of distinct classes. In some situations in practice, information about job classes can be available to the scheduler. For example, in online social media platforms, the scheduler may have access to features of content such as information about the author, content, and other content-item specific information.   

We consider service systems where jobs can be served preemptively, meaning that the scheduler can switch the server from serving the job that is currently being served to serving another job before the current job is completed. This is unlike non-preemptive scheduling where the server must complete serving an assigned job before switching to serving another job. Preemption allows the scheduler to adjust decisions at any timescale based on gathered observations. Because the parameters of stochastic job holding cost processes are unknown to the scheduler, the main challenge is to efficiently learn these parameters in order to realize a near-accurate priority ranking of jobs for minimizing the total accumulated cost.

We consider a learning-based $c\mu$ rule scheduling policy, under which jobs are selected according to the $c\mu$ index estimates obtained from observed data, first according to preemptive scheduling and then according to non-preemptive scheduling discipline. We show theoretical results on regret defined as the difference between the expected total holding cost achieved by an algorithm and the expected total cost achieved by the $c\mu$-rule scheduling policy when the marginal holding costs and mean service time parameters are known. We show a worst-case regret bound in terms of the total number of jobs and a scaling factor for mean job service times. We may think of this scaling factor to represent the rate at which information about stochastic job holding costs is observed by the scheduler. We show lower bounds on regret for any algorithm, which show that our regret upper bounds are nearly optimal.

Previous works \cite{learning-cmu-rule,scheduling-testing,triage-two-classes} 
considered the problem of minimizing the expected total holding cost under different assumptions about uncertainties, either assuming that marginal holding costs of jobs are deterministic and known and mean job service times are unknown, or that both marginal holding costs of jobs and mean service times are a-priori unknown and become known after testing a job.

\subsection*{Related work}

The scheduling problem asking to minimize the sum of weighted completion times for a given set of jobs, with weights $c_i$ and service times $1/\mu_i$, was studied in the seminal paper by Smith \cite{smith}, showing that serving jobs in decreasing order of indices $c_i\mu_i$ is optimal. This policy is often referred to as the \emph{Smith's rule} or $c\mu$ rule. This policy corresponds to the weighted shortest processing time first (WSPT) policy in the literature on machine scheduling. The Smith's rule is also optimal for the objective of minimizing the expected sum of weighted completion times when job service times are random with mean values $1/\mu_i$. We refer the reader to \cite{pinedo} for a comprehensive coverage of various results.

Serving jobs by using $c_i\mu_i$ as the priority index is known to be an optimal scheduling policy for \emph{multi-class, single-server queuing systems}, with arbitrary job arrivals and random independent, geometrically distributed job service times with mean values $1/\mu_i$ \cite{cmu-rule-revisited}. %
A generalized $c\mu$ rule is known to be asymptotically optimal for convex job holding cost functions in a heavy-traffic regime, which corresponds to using a dynamic index defined as the product of the current marginal job holding cost and the job service rate \cite{M95}. This generalized $c\mu$ rule is also known to be asymptotically optimal in a heavy-traffic limit for \emph{multi-server} queuing systems under a certain resource polling condition \cite{MS04}.

The work discussed above on the performance of Smith's or $c\mu$ rule assumes that the marginal job holding costs and the mean job service times are known parameters to the scheduler. Only some recent work considered the performance of these rules when some of these parameters are unknown. In the line of work on \emph{scheduling with testing} \cite{scheduling-testing,triage-two-classes}, marginal job holding costs and the mean job service times are a-priori unknown, but their values for a job become known by \emph{testing} this job. The question there is about how to allocate the single server to processing or testing activities, which cannot be done simultaneously. The optimal policy combines testing the jobs up to certain time and serving the jobs according to the $c\mu$ rule policy thereafter. In \cite{learning-cmu-rule}, a multi-class queuing system is considered under %
assumption that %
marginal job holding costs are known and %
mean job service times are unknown to the scheduler. The authors established that using the empirical $c\mu$ rule in the single-server case guarantees a finite regret with respect to the $c\mu$ rule with known parameters as a benchmark. Similar result is established for the multi-server case by using the empirical $c\mu$ rule combined with an exploration mechanism. %

Our work is related to \emph{permutation} or \emph{learning to rank} problems, e.g. see \cite{Fogel15,ltr} and the references therein, where the goal is to find a linear order of items based on some observed information about individual items, or relations among them. Indeed, the objective of our problem can be seen as finding a permutation $\pi$ that minimizes the cost function $\sum_{i=1}^N i c_{\pi(i)}$, for the special case of identical mean processing time parameter values. For example, we may interpret $c_i$ as a measure of dissimilarity between item $i$ and a reference item, and the goal is to sort items in decreasing order of these dissimilarity indices. The precise objective is defined for a sequential learning setting where irrevocable ranking decisions for items need to be made over time and the cost in each time step is the sum of dissimilarity indices of items which are still to be ranked. 

Although this paper focuses on the objective of minimizing the total cumulative holding cost and equivalently the sum of weighted completion times, there are other types of scheduling problems where the goal is to control the queue length or to maximize the total throughput. Several works considered such scheduling problems under uncertain system parameters and developed algorithms that serve jobs while learning the uncertain parameters. For example, \cite{KriSenJohSha2016Learning,KriSenJohSha2020Learning} proposed a multi-armed bandit framework to model multi-server queuing systems where the servers' mean service rates are unknown, and they analyzed the notion of \emph{queue regret} defined as the difference between the queue lengths obtained by their algorithm and the optimal queue lengths. In~\cite{adaptive-matching}, jobs have unknown types, the posterior distributions of which are updated while attempting to serve them, and the goal is to maximize the system's throughput.

\subsection*{Summary of contributions}

We present an algorithm based on the \emph{empirical $c\mu$ rule}, that is, the $c\mu$ rule applied by using the current sample mean estimates of the mean job holding costs. Since the ranking of jobs based on the $\hat c_{i,t}\mu_i$ values, where $\hat c_{i,t}$ denotes the empirical mean of job $i$'s holding cost in time step $t$, may change over time, it is natural to consider two types of the empirical $c\mu$ rule, preemptive and nonpreemptive. Under the \emph{preemptive} empirical $c\mu$ rule, the server selects a job in every time step from the set of jobs which are not yet completed. In contrast, under the \emph{nonpreemptive} version, once a job is selected in a certain time step, the server has to commit to serving this job until its completion, and then, it may select the next job based on the empirical $c\mu$ rule. The preemptive empirical $c\mu$ rule works well for instances with large gaps between the jobs' mean holding costs, whereas the nonpreemptive one is better for cases where the jobs' mean holding costs are close. We show that if either preemptive or nonpreemptive scheduling is used exclusively, the expected regret can grow linearly in the scaling parameter $\scale$ of job service times in the worst case. The preemptive case may result in undesired delays especially for jobs with similar mean holding cost parameters, whereas the nonpreemptive case may suffer from early commitment to a job with low priority.

Our policy, the \emph{preemptive-then-nonpreemptive} empirical $c\mu$ rule, is a combination of the preemptive and nonpreemptive empirical $c\mu$ rules. This variant of empirical $c\mu$ rule has a fixed length of preemption phase followed by nonpreemptive scheduling of jobs. The preemption period is long enough to separate jobs with large gaps in their mean holding costs, while it is not too long so that we can control delay costs from the preemption phase to be small, thereby avoiding undesired delays from continuous preemption and the risk of early commitment.

In Section~\ref{sec:deterministic}, we give a theoretical analysis of our algorithm for the case of deterministic service times.  We prove that the expected regret of our empirical $c\mu$ rule is sublinear in the scaling factor $\scale$ and subquadratic in the total number of jobs $N$. We also show that this is near-optimal by providing a lower bound on the expected regret of any algorithm, which has the same scaling in $\scale$, and a small gap in terms of the dependence on $N$, when the largest job class has at most $cN$ jobs for some constant $0< c < 1$. For the case when the largest job class has $N - o(N)$ jobs, there is a substantial gap with respect to the dependence on $N$ between our upper and lower bounds. For this case, we propose a refined algorithm, which augments our empirical $c\mu$ rule with a prioritization of the largest job class. We show that this refined algorithm has the expected regret that is near-optimal with respect to both the dependence on $S$ and $N$.

In Section~\ref{sec:extensions}, we consider various extensions including allowing for mean job service times to be non-identical across job classes, instance-dependent regret upper bounds, and stochastic job service times. Our analysis shows that when the service time of each job is stochastic and geometrically distributed, the expected regret of our algorithm is also sublinear in $\scale$ and subquadratic in $N$.

Our regret bounds in Sections~\ref{sec:deterministic} and~\ref{sec:extensions} are obtained based on an equivalent representation of the expected regret that decomposes the regret into the delay costs due to preemption and the regret terms incurred by choosing a low priority job while there exists another job that has priority over the low priority one. For upper bounds, the key part is to argue that even if our algorithm chooses a lower priority job, the gap between the job and the job with the highest priority is not too large. For lower bounds, we consider problem instances where two classes of jobs are statistically so close that any algorithm cannot avoid making suboptimal selection of jobs.

Finally, in Section~\ref{sec:experiments}, we presents results of numerical experiments that demonstrate the performance of our algorithms and validate our theoretical results.

\section{Problem formulation}\label{sec:problem}

We consider a discrete-time single-server scheduling system with one or more job classes. Let $\I=\{1,\ldots, I\}$ denote the set of job classes, where $I\geq 1$. All jobs are present in the system from the beginning, and we assume no further job arrivals. For each class $i\in \I$, we denote by $N_i$ the number of jobs of class $i$ and let $\J_i$ denote the set of jobs of class $i$ at the beginning. Let $N=\sum_{i\in\I}N_i$ denote the total number of jobs to be served, and $\J = \cup_{i\in \I} \J_i$. Notice that it suffices to consider $1\leq I \leq N$ where $I = N$ corresponds to the case when jobs are of distinct classes.  

Every job incurs a random holding cost until its completion according to a stochastic process with independent and identically distributed (i.i.d.) increments with a sub-Gaussian distribution. A random variable $X$ with mean $c$ is sub-Gaussian with parameter $\sigma$ if $\mathbb{E}[X^{\lambda(X-c)}]\leq\exp(\sigma^2\lambda^2/2)$ for all $\lambda\in \mathbb{R}$. %
The class of sub-Gaussian distributions accommodate different parametric distributions, e.g. Bernoulli and Poisson distributions, which are suitable for modeling stochastic holding costs. %
The mean holding costs per unit time of jobs of different classes are of values $c_1, \ldots, c_I$, which are unknown to the scheduler. We assume that the values of $c_1, \ldots, c_I$ are in a bounded interval. Note that if $X$ is sub-Gaussian with parameter $\sigma$, then for any $\sigma^\prime\geq \sigma$, it is sub-Gaussian also with parameter $\sigma^\prime$. Moreover, if $X$ is sub-Gaussian with parameter $\sigma$, then $X/\beta$ is sub-Gaussian with parameter $\sigma/\beta$ for any $\beta>0$. As the total holding cost depends linearly on $c_1,\ldots, c_I$, we assume that $c_i\in[0,1]$ and $\sigma_i=1$ for $i\in\I$ without loss of generality.

The number of service time steps to complete a job of class $i\in\I$ is assumed to be deterministic of value $\scale/\mu_i$ for each $i$, where $\scale$ is a scaling parameter. The larger the value of $\scale$, the larger the number of observations of stochastic costs for each job. Note that a large value of $\scale$ does not necessarily mean that the mean job service times are large in real time. The scaling parameter $S$ may reflect the frequency of scheduling decisions and the rate at which holding costs change in real time. In addition to the case of deterministic job service times, we also consider the case of stochastic job service times, assumed to be according to geometric distributions, which is a standard case studied in the queueing systems literature.

We analyze the performance of our scheduling policy against the minimum (expected) cumulative holding cost that can be achieved when the decision-maker has complete knowledge about the jobs' mean holding costs. The famous $c\mu$ \emph{rule}, which sequentially processes jobs in the decreasing order of their $c_i\mu_i$ values, is known to guarantee the minimum cumulative holding cost, so we use this as our benchmark. Assuming $c_1\mu_1\geq c_2\mu_2\geq\cdots\geq c_I\mu_I$, it is optimal to serve the $N_1$ jobs of class 1 first, the $N_2$ jobs of class 2 next, and so on. Note again that the $c\mu$ rule can be implemented only when the values of $c_1,\ldots,c_I$ are fully known. One can measure the performance of a scheduling algorithm based on partial information about the jobs' mean holding costs by analyzing the following notion of \emph{regret}. Given a (randomized) scheduling policy $\pi$, %
the cumulative holding cost under $\pi$ up to time $T\geq 1$ is given by
$$\sum_{t=1}^T\sum_{i\in\I}\sum_{n\in\J_{i,t}^\pi}X_{n,t}$$
where $\J_{i,t}^\pi$ is the set of remaining jobs in class $i$ at time $t$ under policy $\pi$ and $X_{n,t}$ is the random holding cost incurred by job $n$ at time $t$. Note that the cumulative holding cost depends on the randomness in the holding costs of jobs and the scheduling policy $\pi$ that determines $\J_{i,t}^\pi$ for $i\in\I$ and $t\geq 1$. 
Then we define the \emph{expected regret} of  scheduling policy $\pi$ up to time $T\geq 1$ as
$$\text{Regret}^\pi(T):=\mathbb{E}\left[\sum_{t=1}^T\sum_{i\in\I}\sum_{n\in\J_{i,t}^\pi}X_{n,t}\right]- C^*(T)$$
where $C^{*}(T)$ denotes the expected cumulative holding cost under the $c\mu$ rule.

Recall that  $\mathbb{E}\left[X_{n,t}\right]=c_i$ for any job $n\in \J_{i}$ at any time $t$. Moreover, the distribution of $X_{n,t}$ is determined once $\left\{\J_{i,t}^\pi\right\}_{i\in\I,t\in[T]}$ is fixed. Then the expectation of the cumulative holding cost conditioned on $\left\{\J_{i,t}^\pi\right\}_{i\in\I,t\in[T]}$ can be expressed in terms of the mean holding costs of jobs. We denote by $C^\pi(T)$ the conditional expectation under policy $\pi$ for time $T\geq 1$, so we have
$$C^\pi(T):=\mathbb{E}\left[\sum_{t=1}^T\sum_{i\in\I}\sum_{n\in\J_{i,t}^\pi}X_{n,t}\mid \left\{\J_{i,t}^\pi\right\}_{i\in \I, t\in[T]}\right]=\sum_{t=1}^T\sum_{i\in\mathcal{I}}c_i |\J_{i,t}^\pi|$$
where the second equality follows from $\mathbb{E}\left[X_{n,t}\right]=c_i$ for any job $n\in \J_{i,t}^\pi$ and $t\geq 1$.
Here, as $\J_{i,t}^\pi$'s are random, $C^\pi(T)$ is also a random variable. Nevertheless, by the law of iterated expectations, the expectation of the cumulative holding cost is given by $\mathbb{E}\left[ C^\pi(T)\right]$. Based on this, we obtain the following equivalent definition of the expected regret.
$$\text{Regret}^\pi(T)=\mathbb{E}\left[C^\pi(T)\right]- C^*(T).$$
Under the $c\mu$ rule, the number of remaining class $i$ jobs at time $t$ under the $c\mu$ rule, denoted $\J_{i,t}^*$, is deterministic as the service time of each job is deterministic. This implies that the expectation of the cumulative holding cost under the $c\mu$ rule, for which we introduced the notation $C^*(T)$, is given by $\sum_{t=1}^T \sum_{i\in\I}c_i|\J_{i,t}^*|$.

In this paper, we are interested in the regret at a time step at which all jobs have been served. Any work conserving policy, not letting the server idle whenever there is a job waiting to be served, completes all jobs after precisely $\Tc:=\sum_{i\in\I}N_iS/\mu_i$ time steps. After $\Tc$, as there is no job waiting to be served, $\text{Regret}^{\pi}(T)$ remain the same as  $\text{Regret}^{\pi}(\Tc)$. For this reason, we focus on characterizing $\text{Regret}^{\pi}(\Tc)$. In particular, we provide strong upper and lower bounds on $\text{Regret}^{\pi}(\Tc)$, for which it is sufficient to obtain bounds on $C^\pi:=C^\pi(\Tc)$ and $R^\pi:=C^\pi-C^*(\Tc)$ because $\text{Regret}^{\pi}(\Tc)=\mathbb{E}\left[C^\pi\right]-C^*(\Tc)=\mathbb{E}\left[R^\pi\right]$. Throughout the paper, we refer to $R^\pi$ as the \emph{regret at completion}.

The regret at completion is directly related to the jobs' \emph{completion times}. The completion time of a job is basically the number of time slots in which the job remains in the system.
Note that the term $C^{\pi}$ can be expressed in terms of jobs' completion times as follows:
$$C^{\pi}=\sum_{i\in\I}\sum_{n\in\J_i} c_i T_{i,n}^\pi$$
where $T_{i,n}^\pi$ is the completion time of job $n$ of class $i$ served under $\pi$. %
Under the $c\mu$ rule, the $n$th job of class $i$ stays in the system for $n\scale/\mu_i+\sum_{j\in[i-1]}N_j\scale/\mu_j$ time steps. %
Moreover, $C^{*}:=C^*(\Tc)$ is given by 
\begin{equation}\label{min-cost}
	C^*=\sum_{i\in\I}\sum_{n\in\J_i}c_i\left(\frac{n\scale}{\mu_i}+\sum_{j=1}^{i-1}\frac{N_j\scale}{\mu_j}\right)=\sum_{i\in\I}c_i\left(\frac{N_i(N_i-1)}{2\mu_i}+\sum_{j=1}^{i-1}\frac{N_iN_j}{\mu_j}\right)\scale,
\end{equation}
which is equal to the minimum expected cumulative holding cost.
When $c_i$'s and $\mu_i$'s are fixed, %
$$C^*=O\left(\sum_{i\in\I}N_i^2S+\sum_{j=1}^{i-1}2N_iN_jS\right)=O(N^2\scale).$$
As $C^*=O(N^2\scale)$, our goal is to construct a scheduling policy $\pi$ under which $\text{Regret}^\pi(\Tc)$, the expected regret at completion under $\pi$, is sublinear in the scaling parameter $S$ and subquadratic in the number of jobs $N$. To do so, we focus on bounding the regret at completion under $\pi$, given by $R^\pi$, based on  $\text{Regret}^\pi(\Tc)=\mathbb{E}\left[R^\pi\right]$.

\section{Algorithms and regret bounds}\label{sec:deterministic}

In this section we present our algorithms and bounds on the expected regret. We first show an algorithm and establish upper bounds on the expected regret for this algorithm in Section~\ref{sec:uniform}. This algorithm selects jobs according to the empirical $c\mu$-rule, first selecting jobs preemptively and then switching to serving jobs non-preemptively. We then establish a general lower bound on the expected regret in Section~\ref{sec:lower-bound}. This lower bound identifies cases when the upper bound on the expected regret of the algorithm is nearly optimal and where there can be a substantial gap. To address the latter case, we present a refined algorithm in Section~\ref{sec:refined}. This refined algorithm extends the simple algorithm with giving priority to large job classes. 

The key concepts that underlie the design of our algorithms are the use of the empirical $c\mu$-rule for selecting jobs, switching from preeemptive to non-preeemptive scheduling of jobs, and, finally, giving priority to large class of jobs. We discuss these concepts next. 

\paragraph{Empirical $c\mu$ rule} Our algorithm is a learning variant of the well-known $c\mu$ rule. For each class $i\in\I$, an empirical estimate of value $c_i\mu_i$ is computed over time. Then, every time the algorithm decides which job to serve, a class with the current highest value is chosen. Initially, there is one or more jobs in each class, thus, one or more samples from each class's holding cost distribution are observed. As the number of remaining jobs in a class decreases, there are fewer observations for later time slots. Let $N_{i,t}$ denote the number of class $i$ jobs that exist in time slot $t$, and let $H_{i,t}$ be the total cumulative holding cost by the jobs of class $i$ up to time slot $t$. Note that $\sum_{s=1}^tN_{i,s}$ is the total number of realized i.i.d. random cost values for class $i$ and that $H_{i,t}$ is the sum of the random costs over all existing jobs of class $i$ through the first $t$ time slots. Then, the empirical estimate of class $i$'s mean holding cost at time $t$ is given by
$$
\hat c_{i,t}:=\frac{H_{i,t}}{\sum_{s=1}^tN_{i,s}}.
$$
Note that $\hat c_{i,t}\mu_i$ is an estimator for $c_i\mu_i$. Following the $c\mu$ rule choosing a job from a class in $\text{argmax}_{i\in\I}c_i\mu_i$, we serve a job from some class $i$ maximizing $\hat c_{i,t}\mu_i$.

\paragraph{%
	Preemptive and nonpreemptive scheduling} The next important component of our algorithms %
is deciding whether to serve jobs preemptively or non-preemptively. We consider settings where after providing a unit service to a job at time $t$, the server may switch to serving a different job at time $t+1$ even before the former job is completed. Our algorithms allow for preemption for some number of initial time slots, and then switches to non-preemptively schedules jobs, meaning that the server does not preempt until the current job finishes. Let us consider some problem instances to explain how the idea of combining preemptive and non-preemptive scheduling works for minimizing the regret.

Consider the simple example of two job classes, each with one job and unit mean service time, and job holding costs $c_1\geq c_2$. As each class has just one job and the total number of jobs is two, we say that job $i$ is of class $i$ for $i=1,2$. By the $c\mu$ rule, processing job 1 first and job 2 next is optimal, and the minimum expected cost is $c_1\scale+ c_2 \cdot 2\scale$. Recall that the empirical $c\mu$ rule selects whichever job that has a higher index while the preemptive and nonpreemptive scheduling differ in how frequently such selections are made. We will explain that restricting to preemption indefinitely or scheduling without preemption both fail in some instances.

The preemptive empirical $c\mu$ rule is more flexible in that scheduling decisions may be adjusted in every time slot as the estimators of $c_1$ and $c_2$ are updated. This is indeed favorable when $c_1$ is much greater than $c_2$, in which case, the empirical estimate of $c_1$ would get significantly larger than that of $c_2$ soon. However, we can imagine a situation where $c_1$ and $c_2$ are so close that the empirical estimates of $c_1$ and $c_2$ are almost identical for the entire duration of processing the jobs. Under this scenario, the two jobs are chosen with almost equal probabilities, in which case, they are completed around the same time. For example, job~1 stays in the system for $2\scale$ time periods, while job~2 remains for $2\scale-1$ time steps. Then the regret is $c_1\scale-c_2$, which may be linear in $\scale$.

The issue is that both jobs may remain in the system and incur holding costs for the entire duration of service $2\scale$. In contrast, one job leaves the system after $\scale$ time steps under the optimal policy. Therefore, there is an incentive in completing one job early instead of keeping both jobs longer. 

To avoid the aforementioned issue, we could instead consider the nonpreemptive version that selects a job in the beginning and commits to it. However, when preemption is not allowed, there is a high chance of committing to a suboptimal job. Under the nonpreemptive version, the probability of job 2 being selected first is at least $(1-c_1)c_2$, and therefore, the expected regret of this policy is at least $(1-c_1)c_2\cdot (c_1-c_2)\scale$ as $(c_1-c_2)\scale = (c_1\cdot 2\scale + c_2\scale)-(c_1 T+c_2\cdot 2\scale)$. When $c_1-c_2=\Omega(1)$, the expected regret is linear in $\scale$. Hence, the nonpreemptive version may suffer from undesired early commitment. 

\paragraph{Prioritizing jobs from a large class} 

The last key component of our algorithmic development is prioritizing jobs from a large class. To motivate the underlying idea, we consider a problem class with two classes, one of which includes all but one job. To make it concrete, we consider the setting where $I=2$, $N_1=N-1$, and $N_2=1$. We set the mean holding cost of one class to $1/2+\varepsilon$ and that of the other class to $1/2-\varepsilon$ for some $\varepsilon>0$. We choose a class uniformly at random for the one having mean holding cost $1/2+\varepsilon$. Assume that $\varepsilon$ is too small that no practical algorithm can decide which class is of higher mean holding cost in the first time slot.

Then we imagine an algorithm that follows the empirical $c\mu$ rule for the first time slot but follows the optimal policy afterward. Since $1/2+\varepsilon$ and $1/2-\varepsilon$ are so close, the algorithm makes a mistake in the first time slot with probability almost $1/2$. In particular, the algorithm selects class $2$ with probability around $1/2$ when $c_1=1/2+\varepsilon$ and $c_2=1/2-\varepsilon$, which occurs with probability $1/2$. This implies that when $c_1=1/2+\varepsilon$ and $c_2=1/2-\varepsilon$, the algorithm incurs a regret of $c_1(N-1)$ with probability at least $1/3$, and therefore, the expected regret of the algorithm is $\Omega(N)$. This is striking in that a single mistake from the empirical $c\mu$ rule results in a regret that grows linearly in $N$.

Next we consider another algorithm for the same class of instances. Instead of choosing a class based on the empirical $c\mu$ rule, the second algorithm always chooses to serve a job from class $1$ in the first time slot. The algorithm also follows the optimal policy from the second time slot and onward. Note that the algorithm reduces to the optimal $c\mu$ rule when $c_1=1/2+\varepsilon$ and $c_2=1/2-\varepsilon$, while it makes a mistake every time when $c_1=1/2-\varepsilon$ and $c_2=1/2+\varepsilon$. However, even when $c_1=1/2-\varepsilon$ and $c_2=1/2+\varepsilon$, its regret is only $c_2$. Therefore, the expected regret of the second algorithm is only $O(1)$.

The main takeaway here is that making a mistake when the large class has the priority, which corresponds to serving the small class, results in a significantly worse regret than making a mistake when the small class has the priority. Hence, it is reasonable to give extra prioritization to the large class at the expense of incurring some regret from the case when the small class is of higher cost. This is the underlying intuition for refining our algorithm to be more careful about the unbalanced case by giving extra priority to the largest class of jobs. To be more specific, our algorithm serves a job from the largest class until it figures out that another class has a significantly higher cost than the largest class.

\subsection{Preemptive-then-nonpreemptive empirical $c\mu$ rule}
\label{sec:uniform}

Our algorithm, which we call the preemptive-then-nonpreemptive empirical $c\mu$ rule (in short, \pnalg), is a combination of preemptive scheduling and nonpreemptive scheduling. The algorithm starts in a preemption phase in which the server may try different classes of jobs while learning the mean holding costs of classes, thereby circumventing the early commitment issue. The number of preemptions is limited, which allows avoiding the issue of unnecessary delays. Pseudo-code of our algorithm is given in Algorithm~\ref{preempt-then-nonpreemptive}.

\begin{algorithm}[h!]
	\caption{Preemptive-then-nonpreemptive empirical $c\mu$ rule (\pnalg)}\label{preempt-then-nonpreemptive}
	\begin{algorithmic}
		\State Input: $\Ts%
		$, $I$, $(\mu_i, i\in \I)$, $(N_i, i\in \I)$ %
		\State Initialize $\R \leftarrow \I$ \textit{\quad // $\R$ is the set of unfinished job classes}
		
		\State $n^* \leftarrow \mathsf{null}$ %
		\For{$t=1,\ldots,\Tc$}

		\If{$t\leq \Ts+1$ or $n^*=\mathsf{null}$}
		\State $n^*\leftarrow \text{a job from some class in $\arg\max_{i\in \R}\hat c_{i,t}\mu_i$.}$
		\EndIf
		
		\State Serve job $n^*$
		\If{job $n^*$ has been completed}
		\If{the class of job $n^*$ has no remaining job}
		\State $\R\leftarrow \R\setminus\{\text{the class of job $n^*$}\}$
		\EndIf
		\State $n^*\leftarrow\mathsf{null}$
		\EndIf
		
		\EndFor
	\end{algorithmic}
\end{algorithm}

This algorithm's performance heavily depends on the length of the preemption phase, denoted $\Ts$. We will decide the value of $\Ts$ to be strictly less than the minimum service time of jobs, and as a result, no job finishes during the preemption phase. %

In this section, we focus on the case when $\mu_1=\mu_2=\cdots=\mu_I=\mu$, in which all jobs have the same service time. Here we may assume that $\mu=1$, because for otherwise, we can replace $\scale/\mu$ by $\scale$. 

Recall that $N_i$ is the initial number of jobs of class $i$ for $i\in\I$ and $N=\sum_{i\in\I}N_i$. Then the following result provides an upper bound on the expected regret of Algorithm~\ref{preempt-then-nonpreemptive}. 

\begin{theorem}
	\label{thm:uniform-ub-1}
	The expected regret of Algorithm~\ref{preempt-then-nonpreemptive} is 
	\begin{equation}\label{regret:uniform-ub-1}
		O\left(\max\left\{N\scale^{2/3}(\log N\scale)^{1/3},\ N^{3/2}\scale^{1/2}(\log N\scale)^{1/2}\right\}\right).
	\end{equation}
\end{theorem}

Recall that the minimum expected cumulative holding cost, attainable by the $c\mu$ rule, is $O(N^2\scale)$. Algorithm~\ref{preempt-then-nonpreemptive} indeed achieves a regret that is sublinear in the scaling parameter $S$ and subquadratic in the total number of jobs $N$.

The regret bound in Theorem~\ref{thm:uniform-ub-1} is a worst-case bound. For any given number of job classes $I$ and the total number of jobs $N$, the bound allows for arbitrary initial distribution of jobs over job classes. For the case when the algorithm has no access to information about job classes, each job can be thought to belong to a distinct class, and in this case the total number of jobs corresponds to the total number of distinct classes.

While deferring proof of Theorem~\ref{thm:uniform-ub-1} to the appendix, we sketch proof ideas here. Theorem~\ref{thm:uniform-ub-1} is a consequence of the following lemma characterizing a regret upper bound that depends on additional parameters $I$,  $N_{\min}=\min_{i\in \I} N_i$, and $\Ts$ where $I$ is the number of classes, $N_{\min}$ is the minimum number of jobs in a class, and $\Ts$ is the length of the preemption phase.

\begin{lemma}
	\label{lemma:uniform-ub-1}
	The expected regret of Algorithm~\ref{preempt-then-nonpreemptive} is 
	\begin{equation}\label{regret:uniform-ub-1'}
		O\left(N \Ts + \frac{N\scale(\log N\scale)^{1/2}}{N_{\min}^{1/2}(\Ts+1)^{1/2}}+\min\left\{IN,I^{1/2}N{(\log N)^{1/2}},\frac{N^{3/2}}{N_{\min}^{1/2}}\right\}\scale^{1/2}({\log N\scale})^{1/2}\right).
	\end{equation}
\end{lemma}
As the regret upper bound given by Lemma~\ref{lemma:uniform-ub-1} has three terms, the expected regret of Algorithm~\ref{preempt-then-nonpreemptive} consists of three parts. %
The first part is due to delays caused by serving jobs from low priority classes, classes with low $c_i$ values, during the preemption phase. Here, the preemption phase may have length 0, i.e., $\Ts=0$. In this case, all jobs are processed without preemption, and in particular, the first job for nonpreemptive serving is selected based on $N_i$ observed cost values for each class $i\in \I$ at the beginning of the first time slot. The second part of the regret corresponds to the risk of suboptimal selection of the first job for nonpreemptive serving right after the preemption phase. The last part is the regret incurred from choosing a suboptimal sequence of the rest of jobs for nonpreemptive serving. Note that the bound~\eqref{regret:uniform-ub-1'} has terms with $\log NS$ factors. The $\log NS$ factors arise from estimating the mean job holding costs, for which we use a Hoeffding's bound \cite{hoeffding}.

The proof of Lemma~\ref{lemma:uniform-ub-1} and our regret analysis are based on a key technical lemma that provides a representation for the expected regret, by which we can decompose the regret to different terms that correspond to individual jobs. In particular, the server may switch between jobs during the preemption phase, which would result in jobs getting delayed. The representation given by the lemma captures this. Moreover, serving a job with a smaller mean holding cost before a job with a higher cost would contribute to the regret, and the lemma elucidates how the regret value depends on the gap between the mean holding costs of the two jobs.

Let us add more technical details as to how the three terms stated in Lemma~\ref{lemma:uniform-ub-1} appear in the regret analysis of Algorithm~\ref{preempt-then-nonpreemptive}. To explain the first term $N\tau$, we take a job from the lowest priority class. Under an optimal scheduling policy, this job is served after the jobs from the other classes are completed. Imagine a situation where the job from the lowest priority class gets served for the entire duration of the preemption phase but the job is taken back to the queue until all the other jobs are finished. Under this situation, all but this job are delayed for $\tau$ time units, which results in $O(N\tau)$ regret. Hence, $O(N\tau)$ is a worst-case bound on the regret incurred from the preemption phase.

As mentioned above, the second regret term is incurred while completing the job chosen right after the preemption phase. The job is selected from a class $i$ in $\arg\max_{i\in\I}\hat c_{i,\tau+1}$, but there can be other class $j$ such that $c_j > c_i$, meaning that class $j$ takes priority over class $i$. Here, the gap $c_j - c_i$ being strictly positive indicates that the choice of class $i$ is suboptimal, and this would contribute to regret. In fact, the regret term depends on the gap $c_j-c_i$, and therefore, we need an upper bound on $c_j-c_i$ to provide an upper bound on the regret term. We do this by constructing a confidence interval for the mean holding cost of each class based on a Hoeffding's bound. When selecting the job right after the preemption phase, each class $i\in\I$ collects $N_i(\tau+1)$ samples from its cost distribution, and $N_i(\tau+1)$ is greater than or equal to $N_{\min}(\tau+1)$. By Hoeffding's bound, the true mean holding cost $c_i$ of class $i$ belongs to a confidence interval around its empirical estimate $\hat c_{i,\tau+1}$ of radius $O((\log NS)^{1/2}/N_{\min}^{1/2}(\tau+1)^{1/2})$ with high probability. It follows that if $\hat c_{j,\tau+1}\leq \hat c_{i,\tau+1}$, then $c_j- c_i=O((\log NS)^{1/2}/N_{\min}^{1/2}(\tau+1)^{1/2})$. Lastly, all initial $N$ jobs remain in the system until finishing the first job, which takes up to $\scale$ time steps.

The third regret term is incurred after the first job is finished until completing the rest of jobs. The regret analysis is also based on bounding the value of $c_j-c_i$ for two classes $i$ and $j$ such that $\hat c_{j,t}\leq \hat c_{i,t}$ where $t$ is the moment when a job for nonpreemptive serving is selected. Until finishing the first job, each class $i\in\I$ collects $O(N_i\scale)$ samples from its cost distribution, so the radius of the confidecne interval of $c_i$ is $O((\log NS)^{1/2}/ N_{\min}^{1/2} S^{1/2})$. On the other hand, as some jobs from a class get completed and leave the system, the number of jobs in the class decreases. Hence, we need to carefully keep track of the number of samples obtained from the cost distrubution.

What remains is to decide the value of $\Ts$, that is, the length of the preemption phase. Note that $\Ts$ appears in the first two terms of the regret upper bound in Lemma~\ref{lemma:uniform-ub-1}. Setting
$$\Ts=\Theta\left(N_{\min}^{-1/3} \scale^{2/3}\left(\log N\scale\right)^{1/3}\right)$$
asymptotically minimizes the regret upper bound in Lemma~\ref{lemma:uniform-ub-1}. More precisely, we set the length $\Ts$ as follows:
\begin{equation}
	\label{eq:Ts-choice}
	\Ts=
	\begin{cases}
		\lfloor N_{\min}^{-1/3} \scale^{2/3}\left(\log N\scale\right)^{1/3}\rfloor, &\text{if $\scale > N_{\min}^{-1/3} \scale^{2/3}\left(\log (N\scale)\right)^{1/3}$} \\
		S-1, & \text{ otherwise}.
	\end{cases}
\end{equation}

Note that $\scale > N_{\min}^{-1/3} \scale^{2/3}\left(\log N\scale\right)^{1/3}$ is equivalent to $N_{\min} \scale > \log(N\scale)$.  Intuitively, for fixed values of the number of jobs over classes, for any large enough value of $\scale$, $\Ts$ is set to be roughly proportional to $\scale^{2/3}$ ignoring the logarithmic term and the rounding to an integer value. Otherwise, $\Ts$ is set to value $\scale-1$. Lastly, note that $\Ts<\scale$ under our choice in~\eqref{eq:Ts-choice}. Note that the condition $N_{\min} \scale \leq \log(N\scale)$ captures situations where $N$ is much larger than $\scale$, e.g., $N\geq 2^{N_{\min}S}$. Hence, roughly speaking, the value of $\tau$ is chosen depending on whether $N$ is much greater than $S$ or not.

Based on Lemma~\ref{lemma:uniform-ub-1}, with our choice of $\Ts$ given in~\eqref{eq:Ts-choice}, we can argue that the expected regret of Algorithm~\ref{preempt-then-nonpreemptive} is
\begin{equation}\label{regret:uniform-ub-1''}
O\left(\max\left\{\frac{N}{N_{\min}^{1/3}}\scale^{2/3}(\log N\scale)^{1/3},\ \min\left\{IN,I^{1/2}N{(\log N)^{1/2}},\frac{N^{3/2}}{N_{\min}^{1/2}}\right\}\scale^{1/2}({\log N\scale})^{1/2}\right\}\right).
\end{equation}
Since $N_{\min}\geq 1$ and $I\leq N$,~\eqref{regret:uniform-ub-1''} gives rise to the regret upper bound in Theorem~\ref{thm:uniform-ub-1}. The complete proof of Theorem~\ref{thm:uniform-ub-1} is given in the appendix.

\subsection{Regret lower bound}
\label{sec:lower-bound}

In this section we provide a lower bound on the expected regret of any scheduling policy. This lower bound establishes near optimality of the upper bound in Theorem~\ref{thm:uniform-ub-1} in the case of balanced job classes with respect to the number of jobs per class. The lower bound also covers the case of unbalanced job classes with respect to the number of jobs per class, for which there can be a large gap between the lower bound and the upper bound of Theorem~\ref{thm:uniform-ub-1} with respect on the dependence on $N$. The cases of balanced and unbalanced job classes are formally defined in the following. In Section~\ref{sec:refined}, we propose a refined algorithm that nearly achieves the lower bound for both cases.   

We denote with $\barn$ the number of jobs of all job classes except for excluding a job class with the largest number of jobs, i.e. $\barn := N-\max_{i\in\I}N_i$. Let $\imax$ be some class in $\arg\max_{i\in\I}N_i$. Then $\barn = N-N_{\imax}$ is equal to the number of jobs outside the class $\imax$. We can distinguish two cases with respect to the value of $\barn$: (a) \emph{balanced case} 
under which $\barn = \Omega(N)$ and (b) \emph{unbalanced case} 
under which $\barn = o(N)$. In the former case, the largest number of jobs of a class is at most a constant fraction of the total number of jobs. In the latter case, all but a diminishing small fraction of jobs are of the same class. The lower bound in the following theorem applies to all cases.

\begin{theorem}
\label{thm:uniform-lb}
For any (randomized) scheduling algorithm, there is a family of instances under which the expected regret is
\begin{equation}\label{regret:uniform-lb}
	\Omega\left(\max\left\{ \barn^{2/3} \scale^{2/3}, \ N^{1/2}\barn^{1/2}\scale^{1/2}\right\}\right),
\end{equation}
where the expectation is over the random choice of an instance, the randomness in holding costs, and the algorithm.
\end{theorem}

In the balanced case, the regret lower bound in Theorem~\ref{thm:uniform-lb} is equivalent to 
\begin{equation}
\Omega\left(\max\left\{ N^{2/3} \scale^{2/3}, \ N\scale^{1/2}\right\}\right).
\label{equ:lb1}
\end{equation}
Note that this lower bound nearly matches the upper bound in Theorem~\ref{regret:uniform-ub-1}. Ignoring the logarithmic factors in the upper bound, the upper and lower bounds have the same dependence on parameter $\scale$ while there exists some small gap in the dependence on parameter $N$. 

In fact, when the distribution of $N$ jobs over  classes is balanced, the expected regret of Algorithm~\ref{preempt-then-nonpreemptive} matches the lower bound in Theorem~\ref{thm:uniform-lb}. To elaborate, we have $N_{\min}= N/\kappa$ for some $\kappa>0$. Then, as~\eqref{regret:uniform-ub-1''} is an upper bound on the expected regret, it gives rise to an upper bound
$$\tilde O\left(\max\left\{\kappa^{1/3} N^{2/3} \scale^{2/3},\ \kappa^{1/2} N\scale^{1/2}\right\}\right).$$
Hence, if $\kappa$ is bounded by a fixed constant, $N_{\min}$ is a constant fraction of $N$ which means that each class has at least a constant fraction of the $N$ jobs. Moreover, in such a case, the bound reduces to (\ref{equ:lb1}), as desired. Moreover, if the number of classes is bounded by a fixed constant, the expected regret of Algorithm~\ref{preempt-then-nonpreemptive} nearly matches the lower bound. Since $I$ is some constant and $N_{\min}\geq 1$,~\eqref{regret:uniform-ub-1''} reduces to
$$O\left(\max\left\{N \scale^{2/3}(\log NS)^{1/3},\ N\scale^{1/2}(\log NS)^{1/2}\right\}\right),$$
whose second term equals the second term of~\eqref{regret:uniform-lb} up to a logarithmic factor.

To prove that the expected regret of any (randomized) scheduling algorithm has a lower bound given in Theorem~\ref{thm:uniform-lb}, we show that $\Omega\left(\barn^{2/3}\scale^{2/3}\right)$ and $\Omega\left(N^{1/2}\barn^{1/2}\scale^{1/2}\right)$ are two lower bounds on the expected regret. To explain our proof strategy, let us take some nonempty sets $\I_1$ and $\I_2$ partitioning $\I$, the set of all classes. Assume that
$M_1 \geq M_2$ where $M_1:=\sum_{i\in \I_1}N_i$ and $M_2:=\sum_{i\in \I_2}N_i$. Then we set the mean holding cost of each class $i\in \I_1$ to $c_i=1/2$ while for some $\epsilon>0$, for each $j\in \I_2$, we set $c_j=(1+\epsilon)/2$ with probability $1/2$ and $c_j=(1-\epsilon)/2$ otherwise. Note that the mean holding cost of a job from $\I_2$ is greater than that of a job from $\I_1$ with probability $1/2$ and is smaller than that with probability $1/2$. At the same time, if $\epsilon$ is sufficiently small, it is difficult to determine which of $\I_1$ and $\I_2$ has a higher mean holding cost than the other. We formally argue this using the notion of Kullback–Leibler (KL) divergence. Based on this, we show that the expected regret of any (randomized) scheduling algorithm is bounded below by $\Omega\left(M_2^{2/3}\scale^{2/3}\right)$ under some mild condition. We also prove that the expected regret is bounded below by $\Omega\left(M_1^{1/2}M_2^{1/2}\scale^{1/2}\right)$. These two lower bounds hold true for any partition $\I_1, \I_2$ of $\I$. In particular, we can argue that there always exist a partition $\I_1, \I_2$ such that $M_1=\Omega(\barn)$ and a partition $\I_1, \I_2$ such that $M_1M_2=\Omega(N\barn)$. This in turn gives us the desired lower bound on the expected regret.

\subsection{Refined algorithm}
\label{sec:refined}

In this section, we present and analyse an algorithm that guarantees a better scaling of regret with the number of jobs than the simple %
\pnalg, in the unbalanced case when all but a diminishing fraction of jobs are of the same class.

Recall that the lower bound for the regret in Theorem~\ref{thm:uniform-lb} depends on $N$, $\bar{N}$ and $\scale$, where $\bar{N}=N-\max_{i\in\I} N_i$. In the balanced case, i.e. when $\barn=\Omega(N)$, the lower bound becomes $\Omega\left(\max\left\{ N^{2/3} \scale^{2/3}, N\scale\right\}\right)$. In this case, the upper bound on regret of Algorithm~\ref{preempt-then-nonpreemptive} in Theorem~\ref{thm:uniform-ub-1} nearly matches the lower bound. However, in the unbalanced case, i.e. when $\barn=o(N)$, in which case $\max_{i\in \I} N_i=N-o(N)$, there can be a substantial gap between the upper bound in Theorem~\ref{thm:uniform-ub-1} and the lower bound.

For example, consider scenarios when one class takes all but a small number of jobs so that the distribution of the jobs over classes is extremely unbalanced such that $\barn$ is a fixed constant. In this case, the lower bound in Theorem~\ref{thm:uniform-lb} reduces to $\Omega\left(\max\left\{ \scale^{2/3}, N^{1/2}\scale^{1/2}\right\}\right)$. This resulting lower bound has a gap of factor $N$ from the upper bound in Theorem~\ref{thm:uniform-ub-1}. In fact, for Algorithm~\ref{preempt-then-nonpreemptive}, it seems hard to avoid the large $N$ factors from the regret upper bound $O\left(\max\left\{N\scale^{2/3}, N^{3/2}\scale^{1/2}\right\}\right)$, as described by the following two examples.

\paragraph{Example 1} Consider an instance $\mathcal{P}_1$ with $I=2$, $N_1 = N-1$ and $N_2=1$. Then we have $\barn=N-N_1=N_2=1$. Let $c_1=1/8$ and $c_2=(1-\epsilon)/8$. We set  $\epsilon=\sqrt{\ln 2/\scale}$ if $\scale\leq 3$, and $\epsilon=S^{-1/3}\sqrt{\ln 2}$, otherwise. Under this instance, we can argue that the following holds.

\begin{proposition}\label{prop:ex1}
Under instance $\mathcal{P}_1$ with $N\geq 8$ and $\scale\geq1$, the expected regret of Algorithm~\ref{preempt-then-nonpreemptive} is $\Omega( N\scale^{2/3})$.
\end{proposition}

The main reason for Algorithm~\ref{preempt-then-nonpreemptive} incurring a regret of $\Omega(N\scale^{2/3})$ under instance $\mathcal{P}_1$ is as follows. The $N-1$ jobs in class 1 have a higher mean holding cost than the job of class 2, so processing the class 2 job while a class 1 job is waiting incurs a delay cost from the class 1 job. In fact, we can argue that, since the gap between $c_1$ and $c_2$ is small, the empirical estimate of class 2's mean holding cost $\hat c_{2,t}$ is higher than that of class 1, $\hat c_{1,t}$, for a constant fraction of times during the preemption phase in expectation. Hence, Algorithm~\ref{preempt-then-nonpreemptive} spends a constant fraction on the preemption phase serving the class 2 job in expectation, which costs delay costs from the $N-1$ class 1 jobs.

\paragraph{Example 2} Let $0<\delta<1$ be some fixed constant. Consider an instance $\mathcal{P}_2$ with $I=\lceil N^{\delta}\rceil+1$, $N_1=N-\lceil N^{\delta}\rceil$, and $N_i=1$ for $2\leq i\leq I$. Here, $\bar N=\lceil N^{\delta}\rceil = o(N)$. Let $c_1=1/8$ and $c_i=(1-\epsilon)/8$ for all $2\leq i\leq I$. Under this instance, we can argue that the following holds.

\begin{proposition}\label{prop:ex2}
Under instance $\mathcal{P}_2$ with $N\geq 8^{1/(1-\delta)}$ and $\scale\geq 1$, the expected regret of Algorithm~\ref{preempt-then-nonpreemptive} is $\Omega\left(N^{1+\delta/2}\scale^{1/2}\right)$.
\end{proposition}

Note that $\delta$ can be fixed to a number close $1$, in which case, the exponent $1+\delta/2$ is close to $3/2$. As in Example 1, the intuition for why Algorithm 1 cannot avoid such a high regret is that $\hat c_{2,t}$ is higher than $\hat c_{1,t}$ for significantly many time steps $t$ in expectation, as the gap between $c_1$ and $c_2$ is small. We can argue that under instance $\mathcal{P}_2$, a constant fraction of the first $\lceil N^{\delta}\rceil$ jobs completed by Algorithm~\ref{preempt-then-nonpreemptive} belong to class 2.

From Propositions~\ref{prop:ex1} and~\ref{prop:ex2}, it follows that for any $0<\gamma\leq 1$, there exists some instance with $\bar N=O(N^{1-\gamma})$, under which the expected regret of Algorithm~\ref{preempt-then-nonpreemptive} is
$$\Omega\left(\max\{N\scale^{2/3},\ N^{(3-\gamma)/2}\scale^{1/2}\right).$$
That being said, the gap between the expected regret of Algorithm~\ref{preempt-then-nonpreemptive} and the lower bound given by Theorem~\ref{thm:uniform-lb} can be large in general, especially when $\barn$ is small compared to $N$. Then it is natural to ask if we can find a better algorithm or improve the lower bound.

In the remainder of this section, we provide a refinement of Algorithm~\ref{preempt-then-nonpreemptive} to reduce the dependence on parameter $N$ in the regret upper bound. As suggested by Examples 1 and 2, Algorithm~\ref{preempt-then-nonpreemptive} suffers from a regret that has a high dependence on $N$ when the class with the largest number of jobs has a high holding cost but some of the other jobs is chosen instead. To remedy this, a refined algorithm, given as Algorithm~\ref{preempt-then-nonpreemptive-refined}, prioritizes the largest class. To be specific, Algorithm~\ref{preempt-then-nonpreemptive-refined} gives a priority to serving jobs of the largest class, unless some other class turns out to have a significantly higher holding cost than the largest class. Once all jobs of the largest class are completed, any uncompleted jobs are served according to the empirical $c\mu$-rule. This empirical $c\mu$-rule augmented with prioritization is defined in the procedure {\sc \procname}\ in Algorithm~\ref{preempt-then-nonpreemptive-refined}.

Recall that $\hat c_{i,t}$ denotes $H_{i,t}/\sum_{s=1}^t N_{i,s}$ where $H_{i,t}$ is the total cumulative holding cost incurred by the jobs of class $i$ up to time slot $t$.%
We define two other statistics for each class
$$
\text{UCB}_{i,t} := \hat c_{i,t} +\sqrt{\frac{3}{\sum_{s=1}^t N_{i,s}}\log \frac{N\scale}{\mu_{\min}}} \hbox{ and }
\text{LCB}_{i,t} := \hat c_{i,t} -\sqrt{\frac{3}{\sum_{s=1}^t N_{i,s}}\log \frac{N\scale}{\mu_{\min}}}.
$$
where UCB stands for \emph{upper confidence bound} and LCB stands for \emph{lower confidence bound}. Recall our assumption that $\mu_{\min}=\mu_1=\cdots=\mu_I=1$ throughout this section. These confidence bounds are defined such that $c_i\in [\text{LCB}_{i,t},\text{UCB}_{i,t}]$ for every $i$ and $t$ with high probability. This means that if $\text{UCB}_{i,t}<\text{LCB}_{j,t}$, then $c_i<c_j$ with high probability. We use this to compare the mean holding cost of the largest class and those of other classes.

\begin{algorithm}[t!]
\caption{Refined preemptive-then-nonpreemptive empirical $c\mu$ rule (Refined \pnalg)}\label{preempt-then-nonpreemptive-refined}
\begin{algorithmic}
	\State Input: $\Ts%
	$, $I$, $(\mu_i, i\in \I)$, $(N_i, i\in \I)$
	\State Initialize $\R\leftarrow \I$ and $\RP\leftarrow \emptyset$ \textit{// $\R$ is the set of unfinished job classes, $\RP$ is the set of job classes with priority}
	\State For each class $i\in \I$, designate a job for serving when class $i$ is selected
	\State Set $i_{\max}$ to some class in $\arg\max_{i\in \I}N_i$
	
	\State $n^* \leftarrow \mathsf{null}$ %
	\For{$t=1,\ldots,\Tc$}
	\State $\RP\leftarrow \RP\cup \left\{i\in \R\setminus (\RP\cup\{i_{\max}\}):\ \text{LCB}_{i,t}> \text{UCB}_{j,t}\text{ for some }j\in \RP\cup\{i_{\max}\}\right\}$

	\If{$t\leq \Ts+1$ or $n^*=\mathsf{null}$} 
	\State $n^*\leftarrow \text{the designated job of the class returned by~\textproc{\procname($i_{\max}$, $\R$, $\RP$, ($\hat c_{i,t}\mu_i$, $i\in \R$))}}$
	\EndIf
	
	\State Serve job $n^*$
	\If{job $n^*$ has been completed}
	\If{the class of job $n^*$ has a remaining job}
	\State Choose a new designated job for the class 
	\Else
	\State $\R\leftarrow \R\setminus\{\text{the class of job $n^*$}\}$ and $\RP\leftarrow \RP\setminus\{\text{the class of job $n^*$}\}$
	\EndIf
	\State $n^*\leftarrow\mathsf{null}$
	\EndIf

	\EndFor

	\Procedure{\procname}{$i_{\max}$, $\R$, $\RP$, ($\hat c_{i,t}\mu_i$, $i\in \R$)}
	\If{$i_{\max}\in \R$}
	\If{$\RP = \emptyset$}
	\State Return class $i_{\max}$
	\Else
	\State Return some class $i\in\arg\max_{i\in \RP}\hat{c}_{i,t}\mu_i$
	\EndIf
	\Else
	\State Return some class $i\in\arg\max_{i\in \R}\hat{c}_{i,t}\mu_i$
	\EndIf
	\EndProcedure
\end{algorithmic}
\end{algorithm}

The following gives an upper bound on the expected regret of Algorithm~\ref{preempt-then-nonpreemptive-refined}.

\begin{theorem}
\label{thm:uniform-ub-2}
The expected regret of Algorithm~\ref{preempt-then-nonpreemptive-refined} is
\begin{equation}\label{regret:uniform-ub-2}
	O\left(\max\left\{\barn\scale^{2/3}(\log N\scale)^{1/3},\ N^{1/2}\barn\scale^{1/2}(\log N\scale)^{1/2}\right\}\right).
\end{equation}
\end{theorem}

Recall the lower bound in Theorem~\ref{thm:uniform-lb}, which reads as $\Omega\left(\max\left\{ \barn^{2/3} \scale^{2/3}, \ N^{1/2}\barn^{1/2}\scale^{1/2}\right\}\right)$. Comparing this lower bound with the upper bound in Theorem~\ref{thm:uniform-ub-2}, we note that they differ for factors $\bar N^{1/3}$ in the first term and $\bar N^{1/2}$ in the second term. Therefore, the gap between the lower bound and the upper bound has no explicit dependence on $N$. This is favorable  in the unbalanced case when $\bar N = o(N)$. On the other hand, in the balance case when $\bar N = \Omega(N)$, %
the upper bound of Algorithm~\ref{preempt-then-nonpreemptive-refined} in Theorem~\ref{thm:uniform-ub-2} and the upper bound of Algorithm~\ref{preempt-then-nonpreemptive} in Theorem~\ref{thm:uniform-ub-1} are equivalent to each other up to constant factors.

Theorem~\ref{thm:uniform-ub-2} is a consequence of the following lemma, which is similar in spirit to Lemma~\ref{lemma:uniform-ub-1}.

\begin{lemma}
\label{lemma:uniform-ub-2}
The expected regret of Algorithm~\ref{preempt-then-nonpreemptive-refined} is 
\begin{equation}
	\label{regret:uniform-ub-2'}
	O\left(\barn \Ts + \frac{\barn\scale(\log N\scale)^{1/2}}{N_{\min}^{1/2}(\Ts+1)^{1/2}}+\left(\min\left\{I\barn,I^{1/2}\barn{(\log \barn)^{1/2}},\frac{\barn^{3/2}}{N_{\min}^{1/2}}\right\}+\sqrt{IN\barn}\right)\scale^{1/2}({\log N\scale})^{1/2}\right).
\end{equation}
\end{lemma}

Similarly to the regret upper bound in Lemma~\ref{lemma:uniform-ub-1}, the regret upper bound in Lemma~\ref{lemma:uniform-ub-2} has three terms. The first term bounds the delay costs incurred in the preemption phase. The second term bounds the regret due to suboptimal selection of the first job in the nonpreemptive phase, which starts right after the preemption phase. The factors $\barn$ in the first two terms of the bound in Lemma~\ref{lemma:uniform-ub-2} correspond to $N$ in Lemma~\ref{lemma:uniform-ub-1}. This improvement is obtained by a preferential treatment of the largest class of jobs in Algorithm~\ref{preempt-then-nonpreemptive-refined}. Basically, while the largest class is still active, we make a decision to serve a job from some other class only if we are sure that the class has a higher mean holding cost than the largest class. Then no delay cost is paid for the largest class, and situations as in Examples~1 and 2 are prevented.

The third term in the upper bound of Lemma~\ref{lemma:uniform-ub-2} has two parts, one of which is similar to the third term of the upper bound in Lemma~\ref{lemma:uniform-ub-1} while the other is an extra term. The part with the minimization term has factors with $\barn$, instead of $N$. Again, this comes from a preferential treatment of the largest class by  Algorithm~\ref{preempt-then-nonpreemptive-refined}. The last term, involving factor $\sqrt{IN\barn}$, is new. Algorithm~\ref{preempt-then-nonpreemptive-refined} has to pay this extra regret because a job from the largest class is selected even if there is another class with a higher empirical mean holding cost.

The preemption threshold value $\Ts$ can be set as in~\eqref{eq:Ts-choice}. With this choice of $\Ts$, we can show that the upper bound in Lemma~\ref{lemma:uniform-ub-2}  is upper bounded by
\begin{equation}
\label{regret:uniform-ub-2''}
O\left(\frac{\barn}{N_{\min}^{1/3}}\scale^{2/3}(\log N\scale)^{1/3}+\left(\min\left\{I\barn,I^{1/2}\barn{(\log \barn)^{1/2}},\frac{\barn^{3/2}}{N_{\min}^{1/2}}\right\}+\sqrt{IN\barn}\right)\scale^{1/2}({\log N\scale})^{1/2}\right).
\end{equation}
By noting that $N_{\min}\geq 1$ and $I\leq \barn +1$,~\eqref{regret:uniform-ub-2''} gives rise to the upper bound in Theorem~\ref{thm:uniform-ub-2}.

We conclude this section by pointing out to two cases when we can have a better bound from \eqref{regret:uniform-ub-2''} than the upper bound asserted in Theorem~\ref{thm:uniform-ub-2}.
First, consider the case when $\barn$ is bounded by some fixed constant, then \eqref{regret:uniform-ub-2''} reduces to $\tilde O\left(\max\left\{\scale^{2/3},\ N^{1/2}\scale^{1/2}\right\}\right)$,
which coincides with the lower bound in Theorem~\ref{thm:uniform-lb}. Second, consider the case when $N_{\min}$ is at least a constant fraction of $\barn$, then \eqref{regret:uniform-ub-2''} reduces to $\tilde O\left(\max\left\{\barn^{2/3}\scale^{2/3},\ N^{1/2}\barn^{1/2}\scale^{1/2}\right\}\right)$,
which coincides with the lower bound in Theorem~\ref{thm:uniform-lb}.

\section{Extensions}
\label{sec:extensions}

In Section~\ref{sec:heterogeneous}, we provide regret upper and lower bounds of the refined \pnalg\ given by Algorithm~\ref{preempt-then-nonpreemptive-refined} for the case of heterogeneous service times. In Section~\ref{sec:instance-dep}, we provide an instance-dependent regret upper bound that delineates how the regret depends on the gap between the $c\mu$ index values. Lastly, in Section~\ref{sec:stochastic}, we consider the setting where the service time of each job is random and follows a geometric distribution.

\subsection{Heterogeneous service times}\label{sec:heterogeneous}

In the previous section, we focused on the case where the service time of each job is equal to $\scale$. In this section, we allow the service rates $\mu_1,\ldots,\mu_I$ to be heterogeneous and the service time of a job is one of $S/\mu_1,\ldots, S/\mu_I$. For simplicity of notation, we use notation
$$\bar\scale = S/\mu_{\min}$$
where $\mu_{\min}=\min_{i\in\I}\mu_i$. Let $\mu_{\max}=\max_{i\in\I}\mu_i$.
Then the service time of a job is at most $\bar S$ and greater than or equal to $(\mu_{\max}/\mu_{\min})^{-1}\bar S= S/\mu_{\max}$. We study settings where the following condition is satisfied.
\begin{equation}\label{hetero-assumption}
	\frac{\mu_{\max}}{\mu_{\min}} <\frac{N_{\min}^{1/3}}{\left(\log N\scale\right)^{1/3}}
	\bar \scale^{1/3}.
\end{equation}
Here,~\eqref{hetero-assumption} bounds the ratio of $\mu_{\max}$ and $\mu_{\min}$. We analyze the expected regret of Algorithm~\ref{preempt-then-nonpreemptive-refined} under~\eqref{hetero-assumption}. Note that, up to scaling, we may also assume that $\mu_1,\ldots,\mu_I\geq1$ without loss of generality. We further assume that $\scale/\mu_i$ for $i\in\I$ are all integers. If not, one may replace $\mu_i$ by  $\mu_i^\prime$ such that $\lceil \scale /\mu_i\rceil=\scale/\mu_i^\prime$ since a job of class $i$ needs "at least" this many time steps to be completed.

Then, for the general case, we set the length $\Ts$ of the preemption period to 
\begin{equation}\label{eq:Ts-choice-hetero}
	\Ts=\lfloor N_{\min}^{-1/3} \bar\scale^{2/3}\left(\log N\bar\scale\right)^{1/3}\rfloor.
\end{equation}
By~\eqref{hetero-assumption}, we have
\begin{equation}\label{hetero-assumption'}
	\Ts\leq N_{\min}^{-1/3} \bar\scale^{2/3}\left(\log N\scale\right)^{1/3}< \scale /\mu_{\max}.
\end{equation}
Note that~\eqref{hetero-assumption'} implies that no job finishes until the end of time slot $\Ts+1$.
\begin{theorem}\label{thm:hetero-ub}
	The expected regret of Algorithm~\ref{preempt-then-nonpreemptive-refined} is
	$$O\left(\bar N \bar \scale^{2/3}\left(\log  N\bar \scale\right)^{1/3}+(\mu_{\max}/\mu_{\min})^{1/2}N^{1/2}\barn \bar \scale^{1/2}\left(\log N\bar \scale\right)^{1/2}\right).$$
\end{theorem}

The upper bound given in Theorem~\ref{thm:hetero-ub} is expressed in terms of $\bar\scale$ instead of $\scale$. For the homogeneous case, $\scale$ is the mean service time of each job as $\mu_1=\cdots=\mu_I=1$, but $\scale$ is a scaling factor and not necessarily the mean service time of a job. For the case of heterogeneous service times, $\bar\scale$ is the parameter that corresponds to the mean service time of some job. Moreover, the upper bound has an additional factor $\mu_{\max}/\mu_{\min}$, which is the ratio of the longest service time and the shortest service time. When there is a large gap between the longest service time and the shortest service time, the ratio is large, and thus, the upper bound becomes large as well. The next theorem provides a lower bound on the expected regret of any algorithm.

\begin{theorem}\label{thm:hetero-lb}
	For any (randomized) scheduling algorithm, there is a family of instances under which the expected regret is 
	$$\Omega\left(\max\left\{(\mu_{\max}/\mu_{\min})^{-4/3}\barn^{2/3}\bar \scale^{2/3},\ (\mu_{\max}/\mu_{\min})^{-1}\barn\bar\scale^{1/2} \right\}\right)$$
	where the expectation is taken over the choice of an instance and the randomness in holding costs and the algorithm.
\end{theorem}

The lower bound given in Theorem~\ref{thm:hetero-lb} also has dependence on the ratio $\mu_{\max}/\mu_{\min}$ and the longest mean service time $\bar\scale$. The lower bound and the upper bound given by Theorem~\ref{thm:hetero-ub} has some gap with respect to the ratio as well as the parameter $\bar\scale$.

Lastly, we remark that the upper and lower bounds given by Theorems~\ref{thm:hetero-ub} and~\ref{thm:hetero-lb} recover the bounds~\eqref{regret:uniform-ub-2} and~\eqref{regret:uniform-lb} for the homogeneous case as this case corresponds to setting $\mu_{\max}=\mu_{\min}=1$ and $\bar\scale=\scale$.

The proof of Theorem~\ref{thm:hetero-ub} is an adaptation of the proof of Theorem~\ref{thm:uniform-ub-2} to the heterogeneous service time case. In particular, we compare the empirical $c\mu$ values given by $\{\hat c_{i,t}\mu_i\}_{i\in\I}$, and the confidence interval of the true value $c_i\mu_i$ for each $i\in \I$ has dependence on the mean service time $\mu_i$ as well as the number of samples obtained from the cost distribution. Moreover, the regret depends on the gap $c_j\mu_j - c_i\mu_i$ for some distinct classes $i,j$ with $\hat c_{j,t}\leq \hat c_{i,t}$, not $c_j-c_i$.

Theorem~\ref{thm:hetero-lb} can be proved similarly as in Theorem~\ref{thm:uniform-lb}. The key difference is in the design of problem instances used to provide lower bounds. For the uniform case, we partition the set of classes into two sets $\I_1$ and $\I_2$ and compare $M_1=\sum_{i\in\I_1}N_i$ and $M_2=\sum_{i\in\I_2}N_i$ where $M_\ell$ equals the total number of jobs that belong to a class in $\I_\ell$ for $\ell\in\{1,2\}$. For the heterogeneous case, we compare $M_1=\sum_{i\in\I_1}N_i/\mu_i$ and $M_2=\sum_{i\in\I_2}N_i/\mu_i$ where $M_\ell$ for $\ell\in\{1,2\}$ collects the number of jobs that belongs to a class $i$ in $\I_\ell$ normalized by the mean service time $\mu_i$. Furthermore, the mean holding cost of class $i$ is set to $\mu_{\min}/2\mu_i$, instead of $1/2$, or its perturbation given by $(1\pm\epsilon)\mu_{\min}/2\mu_i$. The rest of the proof is similar to that of Theorem~\ref{thm:uniform-lb}.

\subsection{Instance-dependent regret upper bounds}\label{sec:instance-dep}

The upper and lower bounds on the expected regret in the previous section are independent of the values of $c_1,\ldots, c_I$. However, it is intuitive to expect that Algorithm~\ref{preempt-then-nonpreemptive-refined}'s performance depends on the gaps between the values of $c_1\mu_1,\ldots, c_I\mu_I$, as it would be difficult to separate jobs $i$ and $j$ with $c_i\mu_i$ and $c_j\mu_j$ being close. Motivated by this, we give regret upper bounds that have an explicit dependence on the gaps between $c_1\mu_1,\ldots, c_I\mu_I$.

We will show that the expected regret depends on the quantity $\Delta$ defined as
\begin{equation*}\label{eq:delta-gap}
	\Delta:=\min\left\{\frac{|c_i\mu_i-c_j\mu_j|}{\mu_i+\mu_j}:\ i,j\in\I, i\neq j\right\}
\end{equation*}
that captures the gap between $c_i\mu_i$ and $c_j\mu_j$ values for distinct $i,j$.

\begin{theorem}\label{instance:N}
	The expected regret of Algorithm~\ref{preempt-then-nonpreemptive-refined} is
	\begin{equation*}
		O\left(\frac{\barn}{N_{\min}}\left(\frac{1}{\Delta^2}+\frac{\mu_{\max}\log N}{\mu_{\min}\Delta}\right)\log N\bar\scale\right).
	\end{equation*}
\end{theorem}

Notice that the upper bound has a logarithmic dependence on $\scale$ when the gaps between $c_1\mu_1,\ldots,c_I\mu_I$ are fixed. On the other hand, the largest factor in $\barn$ is still $\barn$ as in the instance-independent upper bound~\eqref{regret:uniform-ub-2} while the explicit dependence on $N$ is poly-logarithmic.

\subsection{Stochastic service times}\label{sec:stochastic}

Algorithm~\ref{preempt-then-nonpreemptive-refined} works for the case of stochastic service times as well. We assume that for $i\in\I$, the mean of a class $i$ job's service time is given by $\scale/\mu_i$ and known to the decision-maker. This incorporates the setting of deterministic service times as a special case. Unlike the deterministic case, some jobs may be finished during the preemption phase in the stochastic case.

Although the service time of each job has randomness unlike the deterministic setting, the original definition of the expected regret extends to the stochastic setting of this section. Recall that $\J_{i,t}^\pi$ is the set of remaining class $i$ jobs at time $t$ under scheduling policy $\pi$ and the cumulative holding cost under $\pi$ up to time $T$ is given by $\sum_{t=1}^T \sum_{i\in\I} \sum_{n\in\J_{i,t}^\pi} X_{n,t}$. Here, $\J_{i,t}^\pi$ depends on not only the random holding costs incurred by the jobs but also the random service times of jobs. Nevertheless, it still holds under the stochastic setting that $\mathbb{E}[X_{n,t}]=c_i$ for any $n\in\J_i$ at any time $t$ and. Therefore, $C^\pi$ and $R^\pi$ can be properly defined as in Section~\ref{sec:problem}.

We prove that when the service time of each job is geometrically distributed, the expected regret of Algorithm~\ref{preempt-then-nonpreemptive-refined} can be still sublinear in $\scale$ and subquadractic in $N$. The probability that each job of class $i$ is completed when it is served in a time slot is $\mu_i/\scale$. For this setting, we set $\Ts$ to
$$\Ts = \lfloor \barn^{2/3}\bar\scale^{2/3}\left(\log N\bar\scale\right)^{1/3}\rfloor.$$
Based on the memoryless property of the geometric distribution, we obtain the following regret upper bound.
\begin{theorem}\label{thm:stochastic}
	When the service time of each job of class $i$ is geometrically distributed with mean $\mu_i/\scale$, the expected regret of Algorithm~\ref{preempt-then-nonpreemptive-refined} is $O(N^{2/3}\barn \bar\scale^{2/3}(\log N\bar\scale)^{1/3})$.
\end{theorem}
The upper bound for the case of geometrically distributed service times has a subquadratic dependence on $N$ and a sublinear dependence on $\bar \scale$, although the dependence on the parameter $N$ is worse than the upper bound~\eqref{regret:uniform-ub-2} for the deterministic homogeneous case.

\section{Experiments}\label{sec:experiments}

We ran experiments to assess the numerical performance of the %
\pnalg, given by Algorithm~\ref{preempt-then-nonpreemptive}, and its refined version, given by Algorithm~\ref{preempt-then-nonpreemptive-refined}. We designed three sets of experiments, described as follows. The first set of experiments is to test the efficiency of {\pnalg} %
against the preemptive empirical $c\mu$ rule and the empirical $c\mu$ rule without preemption. The second set of experiments is for evaluating the tightness of the proposed upper and lower bounds on the regret of {\pnalg} %
by measuring how the expected regret behaves as a function of parameters $N$ and $\scale$. The third set of experiments is designed to compare {\pnalg} %
and the refined {\pnalg}, given by Algorithm~\ref{preempt-then-nonpreemptive-refined}. We explain the details of each set of experiments and discuss the results in the following subsections.
Our code for running the experiments and obtained data are publicly available in \url{https://github.com/learning-to-schedule/learning-to-schedule}. %

\subsection{{\pnalg} versus the pure preemptive and nonpreemptive $c\mu$ rules}

In Figure~\ref{fig:uniform}, we show the results for comparing {\pnalg} against the preemptive and nonpreemptive versions. We use instances with $N=20$, $N_i=1$ for $i\in\I$, $\scale=2000$, $\mu_i=1$ for $i\in\I$, and $c_1,\ldots,c_I$ being sampled from the uniform distribution on $[0.5-\varepsilon,0.5+\varepsilon)$, where $\varepsilon$ is a parameter that we vary.
\begin{figure}[t!]
	\begin{center}
		\includegraphics[width=0.45\linewidth]{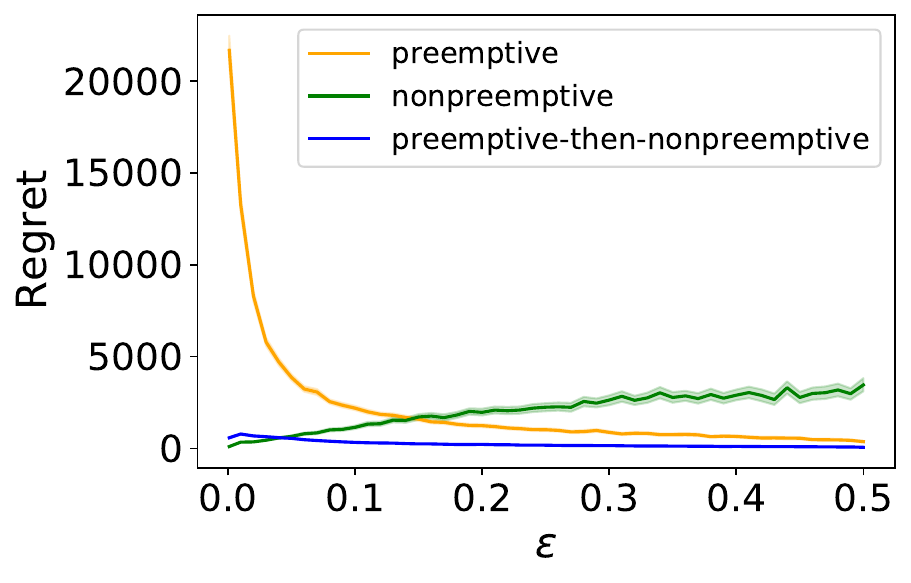}\hspace*{0.5cm}
		\includegraphics[width=0.45\linewidth]{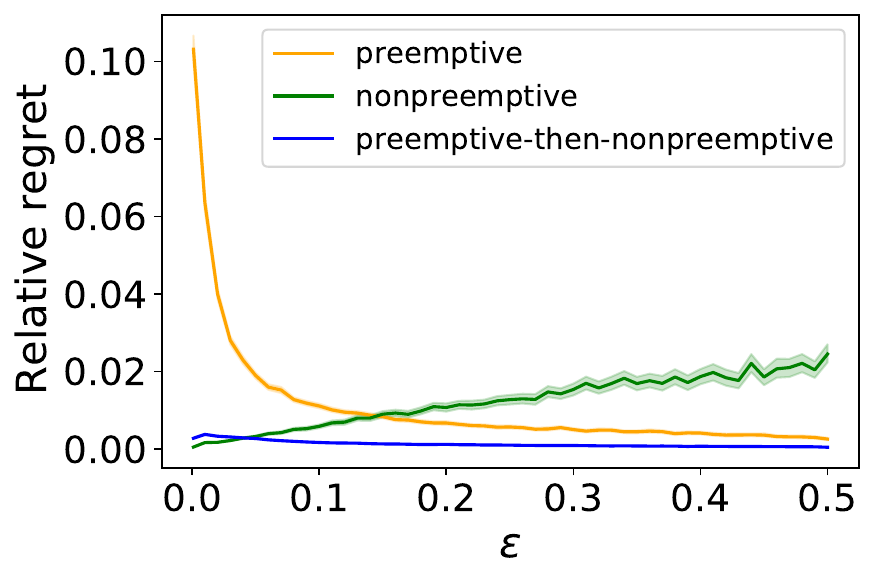}
		\caption{Comparing the three versions of the empirical $c\mu$ rule for the case of deterministic service times and equal service times: (left) regret and (right) relative regret.}\label{fig:uniform}
	\end{center}
\end{figure}
For each value of $\varepsilon$, we generate 100 instances, and for each of which, we record the expected regret of each algorithm where the expectation is taken over the randomness in holding costs. The left plot in Figure~\ref{fig:uniform} shows how the (expected) regret changes by varying the value of $\varepsilon$, and the right plot shows the (expected) \emph{relative} regret, defined as the regret divided by the minimum expected cumulative cost. As expected, the preemptive version suffers for instances of small $\varepsilon$ where the mean holding costs of jobs are close to each other, whereas the nonpreemptive $c\mu$ rule's regret does seem to increase for instances of large $\varepsilon$ where there may be large gaps between the jobs' mean holding costs. Compared to these two algorithms, our {\pnalg} performs uniformly well over different values of $\varepsilon$. 
\begin{figure}[t!]
	\begin{center}
		\includegraphics[width=0.45\linewidth]{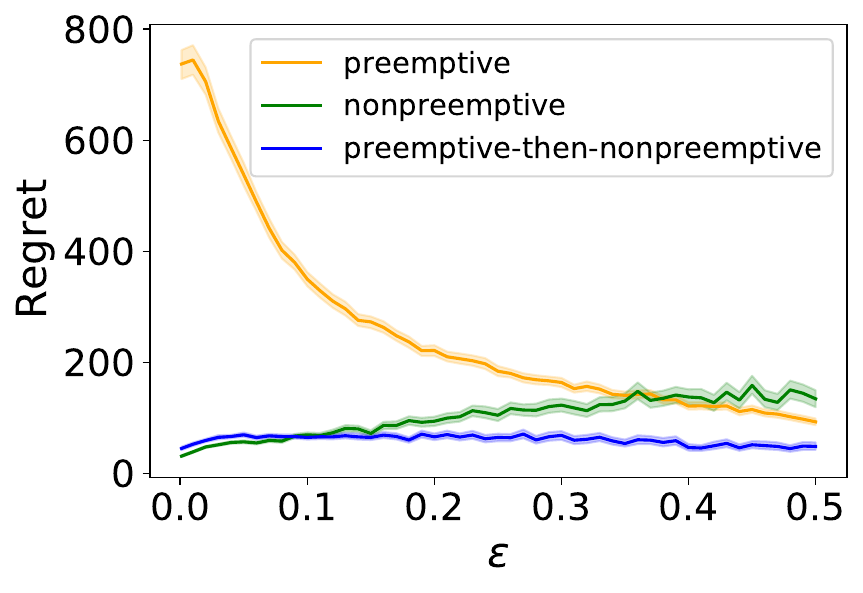}\hspace*{0.5cm}
		\includegraphics[width=0.45\linewidth]{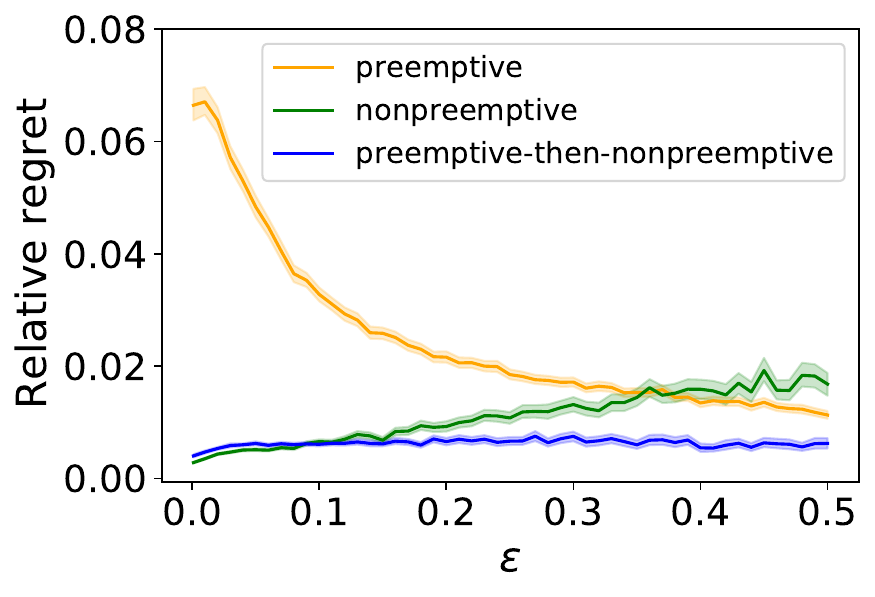}
		\caption{Comparing the three versions of the empirical $c\mu$ rule for the case of deterministic and heterogeneous service times: (left) regret and (right) relative regret.}\label{fig:heterogeneous}
	\end{center}
\end{figure}
This trend continues even when jobs have heterogeneous service times. For the second set of results, we use the same setup as in the first experiment, but following \cite{borg}, we sample the mean service times $\scale/\mu_1,\ldots,\scale/\mu_I$ using a translated (heavy-tailed) Pareto distribution so that $\scale/\mu_i\geq 100$ for $i\in\I$. More precisely, for each $\scale/\mu_i$, we sample a number $x_i$ from the distribution with the density function $f(x)= \frac{0.7}{x^{1.7}}\quad\text{for}\ x\in[1,\infty)$,
and then, we set $\scale/\mu_i=99 + \lfloor x_i\rfloor$.
Here, the density function corresponds to the Pareto distribution with shape parameter $0.7$\footnote{According to~\cite{borg}, Google's 2019 workload data shows that the resource-usage-hours, corresponding to the service times, of jobs follow the Pareto distribution with shape parameter 0.69 (see Figure 12 in~\cite{borg}).}, which has infinite mean. As we assumed that each $\scale/\mu_i$ is an integer, we take $\lfloor x_i\rfloor$, to which we add 99 to ensure that $\scale/\mu_i$ is at least 100. Figure~\ref{fig:heterogeneous} shows that our algorithm achieves small regrets for all values of $\varepsilon$ even for the case of heterogeneous service times.

\subsection{%
	Dependence of the regret of {\pnalg} %
	on parameters $\scale$ and $N$}

To examine how the expected regret of {\pnalg}
grows as a function of $\scale$, we test instances with $N=20$, $N_i=1$ for $i\in\I$, and different values of $S$ from 20 to 1,000,000. To understand how the expected regret depends on $N$, we test instances with $\scale=1000$, $N_i=1$ for $i\in\I$, and different values of $N$ from 2 to 1000. For both kinds of experiments, we set $\mu_i=1$ for $i\in\I$ and $\varepsilon=0.001$, and the reason for this choice is that the family of instances used for providing the regret lower bound~\eqref{regret:uniform-lb} have jobs whose mean holding costs are concentrated around $1/2$ when $\mu_i=1$ for $i\in\I$. For each setup, we generate 100 random instances by sampling $c_1,\ldots, c_I$ from $[0.5-\varepsilon,0.5+\varepsilon)$ uniformly at random.
\begin{figure}[t!]
	\begin{center}
		\includegraphics[width=0.45\linewidth]{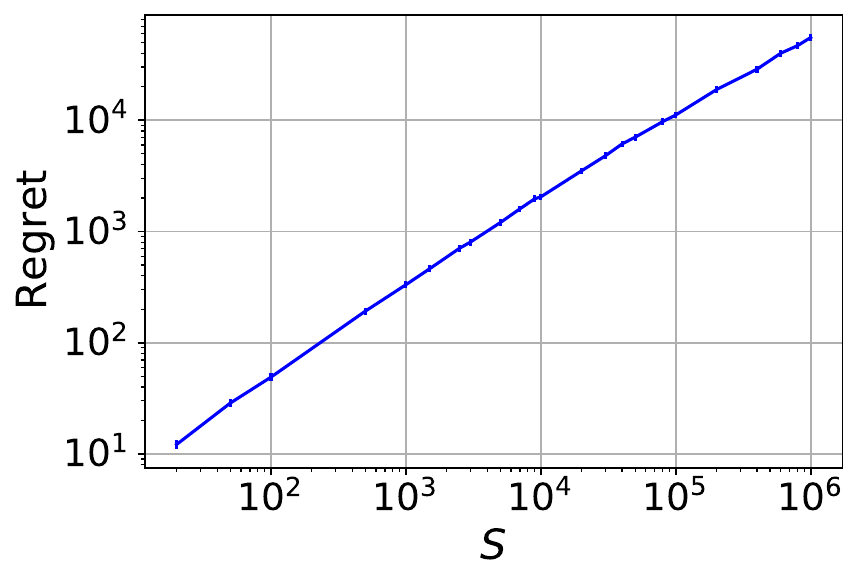}\hspace*{0.5cm}
		\includegraphics[width=0.45\linewidth]{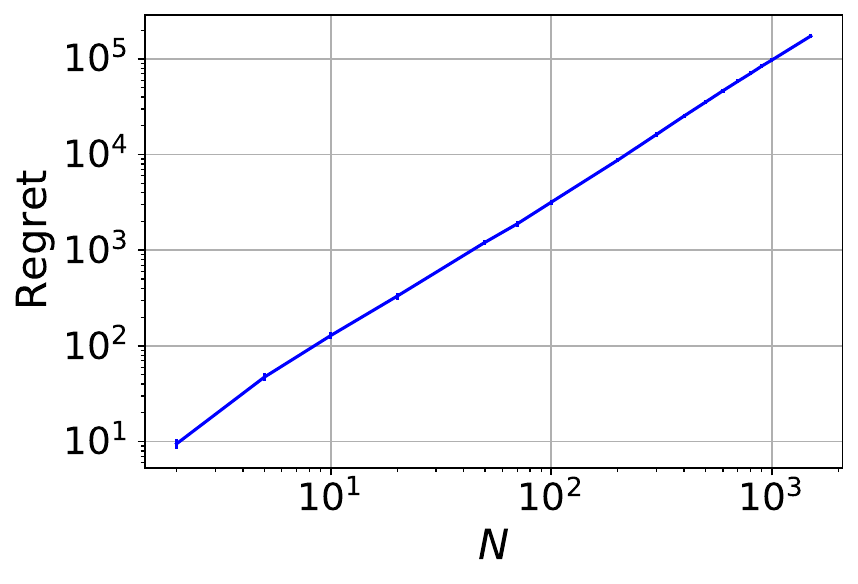}
		\caption{Examining how the regret grows as a function of $S$ (left) and $N$ (right).}\label{fig:dependence}
	\end{center}
\end{figure}
The left plot in Figure~\ref{fig:dependence} shows the regret's dependence on $\scale$ in logaritmic scales of the axes. The plot is almost linear, and its slope is roughly $3.4/4.9\simeq 0.69$, which is close to the exponent $2/3$ for the $\scale$ factors in both the upper bound~\eqref{regret:uniform-ub-1} and the lower bound~\eqref{equ:lb1}. The right plot in Figure~\ref{fig:dependence} shows the regret's dependence on $N$, also in logarithmic scales of the axes. As the left one, the plot is also almost linear, and its slope is approximately $4.1/2.9\simeq 1.41$. This result suggests that the upper bound~\eqref{regret:uniform-ub-1} is close to being exact and that there may be a larger room for improving the lower bound~\eqref{equ:lb1}.

\subsection{Superiority of the refined {\pnalg} for the case of unbalanced job classes}

To consider the case of unbalanced job classes, we generate instances of unbalanced job classes, where $I=2$, $N_1=N-1$, $N_2=1$, $\scale=100$, and $\mu_1=\mu_2=1$. Under this setting, one class contains all but one job, which means that we have $\bar N = N-N_{\max}=1$. We assign $1/2+\varepsilon$ to the mean holding cost value of a class and $1/2-\varepsilon$ to that of the other class where $\varepsilon$ is set to a value in $\{0.001,0.002,\ldots,0.01\}$. %
More precisely, for each instance with a fixed $\varepsilon$, we have $c_1=1/2+\varepsilon$ and $c_2=1/2-\varepsilon$ with probability $1/2$ and $c_1=1/2-\varepsilon$ and $c_2=1/2+\varepsilon$ with probability $1/2$. Following this, we generate 100 instances for each value of $\varepsilon$, and for each instance, we ran {\pnalg} %
and the refined {\pnalg}
and compare their performances measured by regret values. 

Proposition~\ref{prop:ex1} shows that the expected regret of {\pnalg}
is $\Omega(N\scale^{2/3})$ while it follows from~\eqref{thm:uniform-ub-2} that the expected regret of the refined {\pnalg}
is $O\left(\max\left\{\scale^{2/3}(\log N\scale)^{1/3}, N^{1/2}\scale^{1/2}(\log N\scale)^{1/2}\right\}\right)$ as $\bar N=1$. Hence, it is expected that the refined {\pnalg}
gives rise to a smaller regret than 
{\pnalg}.
\begin{figure}[t!]
	\begin{center}
		\includegraphics[width=0.45\linewidth]{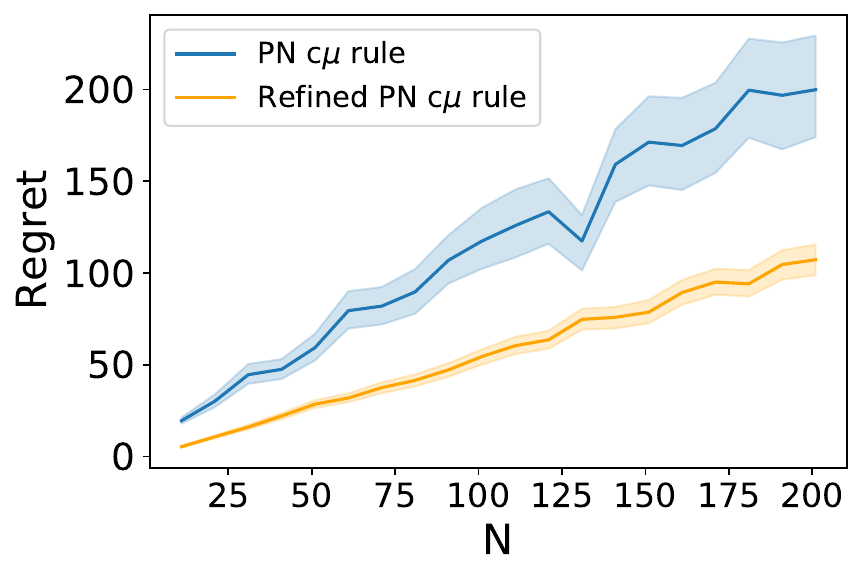}
		\includegraphics[width=0.46\linewidth]{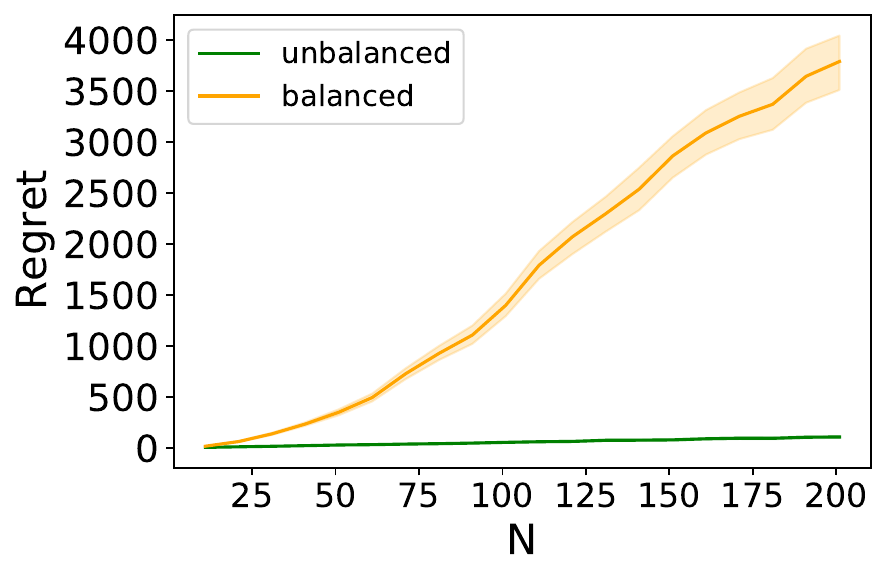}\hspace*{0.5cm}
		\caption{Comparing the \pnalg and its refined version (left) and testing the regret of the refined algorithm for the balanced and the unbalanced cases (right). }\label{fig:refined_supremacy}
	\end{center}
\end{figure}
The left of Figure~\ref{fig:refined_supremacy} shows a numerical result that meets our expectation deduced from the theoretical results. Note that {\pnalg} %
exhibits a steeper growth of regret as $N$ grows than the refined {\pnalg}.

We ran another type of experiments to see how %
the refined {\pnalg}'s performance behaves depending on whether the jobs are equally distributed among classes or not. We generate instances of balanced job classes, where $I=2$, $N_1=N-\lfloor N/2\rfloor$, $N_2=\lfloor N/2\rfloor$, $\scale=100$, and $\mu_1=\mu_2=1$. The mean holding costs $c_1$ and $c_2$ of the two classes are set in the same way as the unbalanced case. The right plot of Figure~\ref{fig:refined_supremacy} depicts how the regret of Algorithm~\ref{preempt-then-nonpreemptive-refined} grows as $N$ increases under each case. The regret increases at a significantly faster rate under the balanced case than the unbalanced case. This observation aligns with our theoretical founding indeed. The class of instances used for the balanced case is precisely the ones used for proving a lower bound on the expected regret, $\Omega\left(\max\left\{N^{2/3}\scale^{2/3},\ N \scale^{1/2}\right\}\right)$ given by Theorem~\ref{thm:uniform-lb}. However, we observed that the regret of the refined {\pnalg} 
under the unbalanced case is $O\left(\max\left\{\scale^{2/3}(\log N\scale)^{1/3}, N^{1/2}\scale^{1/2}(\log N\scale)^{1/2}\right\}\right)$, which has a significantly smaller dependence on parameter $N$.

\section{Conclusion and future work}

This paper %
studies the problem of finding a learning and scheduling algorithm to find a schedule of jobs minimizing the expected cumulative holding cost %
in the setting of stochastic job holding costs with mean job holding costs being unknown to the scheduler.
We give bounds on the expected regret of our algorithm for both the case of deterministic service times and the setting of geometrically distributed stochastic service times. Lastly, we provide numerical results that support our theoretical findings.

One open question is about improving our analysis for the case of heterogeneous service times. The regret upper and lower bounds that we provided for the heterogeneous case have some gaps with respect to the ratio $\mu_{\max}/\mu_{\min}$. We leave as an open question to improve upper and lower bounds on the expected regret of the preemptive-then-nonpreemptive empirical $c\mu$ rule for the case of large gaps in $\mu_1,\ldots,\mu_I$.

Another open question concerns the case of geometrically distributed stochastic service times. Although we have proved that the expected regret of our algorithm is sublinear in the scaling factor $S$ and subquadratic in $N$, we think that there exists a more refined regret analysis. Our argument is based on the observation that the jobs remaining after the preemption phase will have generated $\tau$ instantiated holding costs. However, as the service times of jobs are stochastic, the number of observations for a job is also a random variable, but we could not take this into account in our analysis. 

One may also consider some variations of our problem by allowing for partial or delayed feedback. We can imagine a situation where the learner observes stochastic holding costs only for a subset of items in each time step, or another scenario in which realized job holding costs are observed by the learner after some delay. %
This may be of interest in real-world systems where only a limited information about stochastic holding costs is accessible by the learner due to computation or communication constraints in each time step.

Lastly, it is left for future work to study cases when \emph{both} mean job holding costs and mean job service times are unknown parameters. \cite{learning-cmu-rule} considers unknown mean service times, whereas our work studies the case of unknown mean job holding costs. Combining these two frameworks may be an interesting problem to study.

\paragraph{Acknowledgements}
This research is supported, in part, by the Institute for Basic Science (IBS-R029-C1, Y2) and the Facebook Systems for ML Research Award.

\bibliographystyle{plainnat} 
\bibliography{mybibfile}

\appendix

\section{Clean event}\label{sec:clean-event}

Henceforth, we use notation $[m]$ for any positive integer $m$ to denote $\{1,\ldots, m\}$, the set of all positive integers less than or equal to $m$. Recall that the service time of a job takes a value in $\{S/\mu_i: i\in\I\}$ where $\I$ denotes the set of classes and that $\mu_{\min}=\min_{i\in\I}\mu_i$. Moreover, $N_i$ is the number of initial jobs of class $i\in\I$, $N_{i,t}$ is the number of class $i\in\I$ jobs that remain in time slot $t\geq 1$, and $N=\sum_{i\in\I} N_i$.

We define the notion of "clean event" to analyze the performance of the preemptive-then-nonpreemptive empirical $c\mu$ rule. %
Recall that $\hat c_{i,t}$ denotes $H_{i,t}/\sum_{s=1}^t N_{i,s}$ where $H_{i,t}$ is the total cumulative holding cost incurred by the jobs of class $i$ up to time slot $t$, which is the sum of $\sum_{s=1}^t N_{i,s}$ i.i.d. sub-Gaussian random variables (per-time holding costs). As the number $\sum_{s=1}^t N_{i,s}$ itself is a random variable, we apply the "reward tape" argument from~\cite{MAL-068}. The total number of realized per-time holding costs incurred by class $i$ jobs is at most $N_iN\scale/\mu_{\min}$, because class $i$ has $N_i$ jobs initially and the algorithm must complete all jobs by time $N\scale/\mu_{\min}$. For each class $i$, we obtain $N^2\scale/\mu_{\min}$ samples from the per-time holding cost distribution of class $i$ and record them in a tape with $N^2\scale/\mu_{\min}$ cells. Here, $N_iN\scale/\mu_{\min}$ cells suffice, but we take $N^2\scale/\mu_{\min}(\geq N_iN\scale/\mu_{\min})$ cells for technicality. Then for a job of class $i$ remaining at time $t$, its holding cost for the time slot is taken from a cell in the tape. For $m=1,2,\ldots, N^2\scale/\mu_{\min}$, let $X_{i,m}$ be the $m$th cost value recorded on the tape. 

We say that the \emph{clean event} holds when the following condition is satisfied:
\begin{equation*}
	\left|c_i - \frac{1}{m}\sum_{s=1}^mX_{i,s}\right| \leq x_m \ \text{for} \ m\in[N^2\scale/\mu_{\min}]\ \text{and for} \ i\in \I,\ \text{where}\ x_m=\sqrt{\frac{3}{m}\log \frac{NS}{\mu_{\min}}}.
\end{equation*}

Recall that $S/\mu_i$ is assumed to be an integer for each $i\in \I$, in which case $N^2S/\mu_{\min}$ is an integer.
Since $X_{i,m}$ for all $m$ are sub-Gaussian %
with mean $c_i$ and variance proxy parameter\footnote{This is equivalent to the variance when the distribution is Gaussian.} $1$, %
by Hoeffding's inequality \cite{hoeffding}, 
$$\mathbb{P}\left[\left|c_i - \frac{1}{m}\sum_{s=1}^mX_{i,s}\right| > x_m\right]\leq 2\exp(-2m x_m^2)$$
for any $m\geq 1$ and $x_m>0$ since $c_i\in[0,1]$. Then we obtain the following by using the union bound:
\begin{align}\label{clean-event:bound}
	\begin{aligned}
		\mathbb{P}\left [\text{clean event}\right] &=\mathbb{P}\left[\left|c_i - \frac{1}{m}\sum_{s=1}^mX_{i,s}\right| \leq x_m\ \text{for} \ m\in[N^2\scale/\mu_{\min}]\ \text{and for} \ i\in \I\right]\\&\geq 1-\sum_{i\in \I}\sum_{m\in[N^2\scale/\mu_{\min}]}\mathbb{P}\left[\left|c_i - \frac{1}{m}\sum_{s=1}^mX_{i,s}\right| > x_m\right]\\
		&\geq 1-2I\sum_{m\in[N^2\scale/\mu_{\min}]}\exp(-2m x_m^2)\\
		&=1-2I\frac{N^2\scale}{\mu_{\min}}\frac{\mu_{\min}^6}{N^6\scale^6}\\
		&\geq 1-\frac{2}{{N^3\bar \scale^5}}
	\end{aligned}
\end{align}
where the last inequality is because $I\leq N$ and $\bar\scale = \scale/\mu_{\min}$. %
Hence, under the clean event, we have that 
\begin{equation}\label{conf-interval}
	c_i\mu_i\in\left[\hat c_{i,t}\mu_i-\mu_i\sqrt{\frac{3}{\sum_{s=1}^tN_{i,s}}\log \frac{N\scale}{\mu_{\min}}},\ \hat c_{i,t}\mu_i+\mu_i\sqrt{\frac{3}{\sum_{s=1}^tN_{i,s}}\log \frac{N\scale}{\mu_{\min}}}\right].
\end{equation}
for all $i\in \I$ and $t\in[N\scale/\mu_{\min}]$. %

Consider two classes $i$ and $j$ such that $c_i\mu_i\geq c_j\mu_j$. If $\hat c_{i,t}\mu_i\leq \hat c_{j,t}\mu_j$, under the clean event, the following holds:
\begin{align}\label{clean-event:gap}
	\begin{aligned}
		c_i\mu_i-c_j\mu_j &\leq \left(\hat c_{i,t}\mu_i +\mu_i\sqrt{\frac{3}{\sum_{s=1}^tN_{i,s}}\log \frac{N\scale}{\mu_{\min}}}\right)- \left(\hat c_{j,t}\mu_j-\mu_j\sqrt{\frac{3}{\sum_{s=1}^tN_{j,s}}\log \frac{N\scale}{\mu_{\min}}}\right) \\
		&\leq \mu_i\sqrt{\frac{3}{\sum_{s=1}^tN_{i,s}}\log \frac{N\scale}{\mu_{\min}}}+\mu_j\sqrt{\frac{3}{\sum_{s=1}^tN_{j,s}}\log \frac{N\scale}{\mu_{\min}}}.
	\end{aligned}
\end{align}

\section{A basic tool for understanding the expected regret}\label{sec:regret-basic}

In this section, we prove Lemma~\ref{lemma:regret-1} that provides an equivalent representation of the expected regret, which our regret analysis later crucially relies on. The representation given by Lemma~\ref{lemma:regret-1} allows us to decompose the expected regret to smaller terms that correspond to individual jobs. In particular, the representation unravels how the regret depends on the delay costs and the gaps between jobs' mean holding costs.

Without loss of generality, we assume that $$c_1\geq c_2\geq \cdots \geq c_I.$$ 
There are total $N=\sum_{i\in \I} N_i$ jobs in $\J$ that are initially present to be served. We enumerate the $N$ jobs from $1$ to $N$ so that jobs $1+\sum_{j\in[i-1]}N_j,\ldots, \sum_{j\in[i]}N_j$ are the ones in $\J_i$ of class $i$. When the values of $c_1,\ldots, c_I$ are known, we may serve jobs from 1 to $N$, minimizing the total cumulative holding cost. Let $d_n$ denote the %
mean holding cost per unit time
of job $n\in[N]$. Then, if job $n$ is of class $i$, we have $d_n=c_i$. Moreover, we introduce notation $\hat d_{n,t}$ for $n\in[N]$ and $t\geq1$ which is equivalent to $\hat c_{i,t}$ assuming that job $n$ is of class $i$.

Now let $\sigma:[N]\rightarrow [N]$ be the permutation of $[N]$ that corresponds to the sequence of jobs completed by an algorithm $\pi$, i.e., $\pi$ finishes jobs in the order $\sigma(1),\sigma(2),\ldots, \sigma(N)$. For $n\in[N]$, let us count the number of time steps where job $\sigma(n)$ stays in the system. For job $\sigma(n)$ to be completed, the system needs to process jobs $\sigma(1),\ldots,\sigma(n-1)$ first and then job $\sigma(n)$, for which the server needs to spend $n\scale$ time steps. At the same time, the server may spend some number of time steps, denoted $W_n$, to serve jobs other than $\sigma(1),\ldots, \sigma(n)$ before completing job $\sigma(n)$. Then job $\sigma(n)$ stays in the system for precisely $W_n+n\scale$ time steps. Let $C^\pi$ and $R^\pi$ denote the cumulative holding cost and the regret incurred up to $\Tc$, the time at which all jobs are completed, respectively. Then $C^\pi$ is precisely,
$$
C^\pi=\sum_{n\in[N]}d_{\sigma(n)}W_n+\sum_{n\in[N]}d_{\sigma(n)}\cdot n\scale,$$
and since the minimum holding cost is $\sum_{n\in[N]} d_n\cdot n\scale$, we have
\begin{equation}\label{prelim-regret'}
	R^\pi=\sum_{n\in[N]}d_{\sigma(n)}W_n+\sum_{n\in[N]}\left(d_{\sigma(n)}-d_n\right)n\scale.\end{equation}
Here, $d_{\sigma(n)}-d_n$ can be negative. Nonetheless, we will show that $\sum_{n\in[N]}\left(d_{\sigma(n)}-d_n\right)\cdot n\scale$ can be rewritten as a sum of nonnegative terms only. Let $E_n$ be defined as 
\begin{equation}\label{eq:E_n}
	E_n:=\left\{\ell\in[N]:\ \ell>n,\ \sigma(\ell) < \sigma(n) \right\}.
\end{equation}
Then we know that
$$
E_n\supseteq\left\{\ell\in[N]:\ \text{$\sigma(n)$ finishes before $\sigma(\ell)$ and
	$d_{\sigma(\ell)}> d_{\sigma(n)}$}\right\}.
$$
Note that for any $n\in[N]$ and $\ell\in E_n$, we know that $d_{\sigma(\ell)}-d_{\sigma(n)}\geq 0$. 
\begin{lemma}\label{lemma:regret-1}
	Let $E_n$ be defined as in~\eqref{eq:E_n}. Then 
	\begin{equation}\label{prelim-regret}
		R^\pi=\sum_{n\in[N]}d_{\sigma(n)}W_n+\sum_{n\in[N]}\sum_{\ell\in E_n}(d_{\sigma(\ell)} - d_{\sigma(n)})\scale.
	\end{equation}
\end{lemma}
\begin{proof}
	Due to~\eqref{prelim-regret'}, it is sufficient to show that
	$$ \sum_{n\in[N]}\left(d_{\sigma(n)}-d_n\right)\cdot n\scale=\sum_{n\in[N]}\sum_{\ell\in E_n}(d_{\sigma(\ell)} - d_{\sigma(n)})\scale$$
	holds. The first sum can be rewritten as
	\begin{equation}\label{eq:sum1}
		\sum_{n\in[N]}\left(d_{\sigma(n)}-d_n\right)\cdot n\scale = \sum_{n\in[N]}d_{\sigma(n)}\left(n-\sigma(n)\right)\scale.
	\end{equation} 
	Now let us count how many times each $d_{\sigma(n)}$ appears in the sum $\sum_{n\in[N]}\sum_{\ell\in E_n}(d_{\sigma(\ell)} - d_{\sigma(n)})$. In the sum $\sum_{n\in[N]}\sum_{\ell\in E_n}d_{\sigma(\ell)}$, note that $d_{\sigma(n)}$ appears once for every $k\in [N]$ such that $n\in E_k$. Moreover, $d_{\sigma(n)}$ appears once for every $\ell\in E_n$ in the sum $\sum_{n\in[N]}\sum_{\ell\in E_n}d_{\sigma(n)}$. Hence, the aggregated number of appearance of $d_{\sigma(n)}$ in $\sum_{n\in[N]}\sum_{\ell\in E_n}(d_{\sigma(\ell)} - d_{\sigma(n)})$ is precisely
	\begin{equation*}
		\left|\left\{k\in[N]:\ k<n,\ \sigma(k)> \sigma(n)\right\}\right| - \left|\left\{\ell\in[N]:\ \ell>n,\ \sigma(\ell)< \sigma(n)\right\}\right|.
	\end{equation*}
	Note that
	$$ \left|\left\{k\in[N]:\ k<n,\ \sigma(k)> \sigma(n)\right\}\right|+\left|\left\{k\in[N]:\ k<n,\ \sigma(k)< \sigma(n)\right\}\right| =\left|\left\{k\in[N]:\ k<n\right\}\right|=n-1$$
	and that
	$$ \left|\left\{\ell\in[N]:\ell>n,\sigma(\ell)< \sigma(n)\right\}\right|+\left|\left\{k\in[N]:k<n, \sigma(k)< \sigma(n)\right\}\right| =\left|\left\{\ell\in[N]: \sigma(\ell)< \sigma(n)\right\}\right|=\sigma(n)-1.$$
	This implies that
	$$\left|\left\{k\in[N]:\ k<n,\ \sigma(k)> \sigma(n)\right\}\right| - \left|\left\{\ell\in[N]:\ \ell>n,\ \sigma(\ell)< \sigma(n)\right\}\right| = (n-1)-(\sigma(n)-1)=n-\sigma(n),$$
	and therefore, the aggregated count of $d_{\sigma(n)}$ in $\sum_{n\in[N]}\sum_{\ell\in E_n}(d_{\sigma(\ell)} - d_{\sigma(n)})$ is exactly $n-\sigma(n)$. This means that
	\begin{equation}\label{eq:sum2}
		\sum_{n\in[N]}\sum_{\ell\in E_n}(d_{\sigma(\ell)} - d_{\sigma(n)})\scale=\sum_{n\in[N]}d_{\sigma(n)}(n-\sigma(n))\scale.
	\end{equation}
	Hence, we deduce from~\eqref{eq:sum1} and~\eqref{eq:sum2} that $ \sum_{n\in[N]}\left(d_{\sigma(n)}-d_n\right)\cdot n\scale=\sum_{n\in[N]}\sum_{\ell\in E_n}(d_{\sigma(\ell)} - d_{\sigma(n)})\scale$, as required.
\end{proof}

\section{Proof of Theorem~\ref{thm:uniform-ub-1}}\label{sec:proof:uniform-ub-1}

In Section~\ref{sec:uniform-ub-step1}, we prove Lemma~\ref{lemma:uniform-ub-1} that gives the regret upper bound~\eqref{regret:uniform-ub-1'}. The bound~\eqref{regret:uniform-ub-1'} has terms involving the parameter $\tau$, which is the length of the preemption phase. In Section~\ref{sec:uniform-ub-step2}, setting $\tau$ as in~\eqref{eq:Ts-choice} gives rise to the regret upper bound~\eqref{regret:uniform-ub-1}, thereby proving Theorem~\ref{thm:uniform-ub-1}.

\subsection{Proof of Lemma~\ref{lemma:uniform-ub-1}}\label{sec:uniform-ub-step1}
In this section, we give a complete proof of Lemma~\ref{lemma:uniform-ub-1}, which states that the expected regret of Algorithm~\ref{preempt-then-nonpreemptive} is bounded above by
\begin{equation}
	O\left(N \Ts + \frac{N\scale(\log N\scale)^{1/2}}{N_{\min}^{1/2}(\Ts+1)^{1/2}}+\min\left\{IN,I^{1/2}N{(\log N)^{1/2}},\frac{N^{3/2}}{N_{\min}^{1/2}}\right\}\scale^{1/2}({\log N\scale})^{1/2}\right).\tag{\ref{regret:uniform-ub-1'}}
\end{equation}

We use Lemma~\ref{lemma:regret-1} to provide the regret upper bound~\eqref{regret:uniform-ub-1'}. In~\eqref{prelim-regret}, the first sum comes from jobs getting delayed. We will argue that under Algorithm~\ref{preempt-then-nonpreemptive}, the term can be bounded by the delay costs incurred during the preemption phase only. We further decompose the second sum in~\eqref{prelim-regret} to the terms for the first job completed and the other terms. The first job for nonpreemptive serving is chosen right after the preemption phase, and we can bound the corresponding terms by upper bounding the gaps between jobs' mean holding costs. The terms for the other jobs can be analyzed similarly by understanding how large the gaps between jobs' mean holding costs are, but the difficulty is that as jobs get finished and leave the system, we need to carefully keep track of the number of remaining jobs and the confidence interval for the mean holding cost of each class.

We have defined the notion of clean event in Appendix~\ref{sec:clean-event}. Let us consider the case where the clean event does not hold first. Algorithm~\ref{preempt-then-nonpreemptive} is a work conserving policy, under which all jobs must be completed by the end of $N\scale$th time slot. An obvious implication of this is that the completion time of each job is bounded above by $N\scale$. Another straightforward fact is that the expected regret of Algorithm~\ref{preempt-then-nonpreemptive} is upper bounded by it expected cumulative holding cost, which is 
$$
\sum_{i\in \I}\sum_{n\in \J_i}c_i\cdot(\text{the completion time of job $n$ of class $i$ under Algorithm~\ref{preempt-then-nonpreemptive}}).
$$
As the completion time of each job under Algorithm~\ref{preempt-then-nonpreemptive} is at most $N\scale$ and $c_i \in [0,1]$ for all $i\in \I$, the expected cumulative holding cost is at most $N^2\scale$, and so is the expected regret.

We next focus on the case where the clean event holds. In particular, inequality~\eqref{clean-event:gap} holds for every pair of two jobs from different classes. We first use Lemma~\ref{lemma:regret-1} to bound the regret at completion $R^{\pi}$. We claim that $W_n\leq \Ts$. Let $t$ be some time slot in which the server gives service to a job other than $\sigma(1),\ldots, \sigma(n)$ while $\sigma(n)$ still waits to be served. Note that Algorithm~\ref{preempt-then-nonpreemptive} serves jobs without preemption after the preemption phase, which means that the time slot $t$ must be within the preemption phase. Hence, $t\leq \Ts$, and thus $W_n\leq \Ts$. Then by Lemma~\ref{lemma:regret-1} and~\eqref{prelim-regret}, 
$$R^\pi\leq N\Ts+\sum_{n\in[N]}\sum_{\ell\in E_n}(d_{\sigma(\ell)} - d_{\sigma(n)})\scale.$$
We will bound the second term on the right hand side of this inequality. Take $n=1$ and consider $\sum_{\ell\in E_1}(d_{\sigma(\ell)} - d_{\sigma(1)})$.
Note that $\sigma(1)$ is the job selected right after the preemption phase of Algorithm~\ref{preempt-then-nonpreemptive}, implying in turn that $\hat d_{\sigma(1), \Ts+1}\geq \hat d_{\sigma(\ell), \Ts+1}$ for all $\ell\in[N]$. Since each job requires $\scale$ units of service to finish, all $N$ jobs remain in the system until the end of the $\scale$th time slot. This means that as $\Ts<\scale$, all $N$ jobs are present in the system at the beginning of the $(\Ts+1)$th time slot. Then we have $N_{i,s}=N_i\geq N_{\min}$ for all $s\leq \Ts+1$ and $i\in \I$. It follows from inequality~\eqref{clean-event:gap} that
$$d_{\sigma(\ell)}-d_{\sigma(1)}\leq \sqrt{\frac{12}{N_{\min}(\Ts+1)}\log N\scale}.$$
As $E_1\subseteq [N]$, the cardinality of $E_1$ is trivially at most $N$, and therefore, we obtain
$$R^\pi\leq N\Ts+N\scale\sqrt{\frac{12}{N_{\min}(\Ts+1)}\log N\scale}+\sum_{n\in[N]\setminus\{1\}}\sum_{\ell\in E_n}(d_{\sigma(\ell)} - d_{\sigma(n)})\scale.$$
Now it remains to bound the third term on the right hand side of this bound on $R^\pi$. For $n\geq 2$, let $t_n$ denote the time when job $\sigma(n)$ is selected by Algorithm~\ref{preempt-then-nonpreemptive} after the preemption phase. As $t_n$ is a moment after jobs $\sigma(1),\ldots, \sigma(n-1)$ are completed, $t_n\geq (n-1)\scale$. For $n\geq 2$, Algorithm~\ref{preempt-then-nonpreemptive} finishes $\sigma(n)$ before $\sigma(\ell)$ for any $\ell\in E_n$, meaning that $\hat d_{\sigma(n),t_n}\geq \hat d_{\sigma(\ell),t_n}$. Then~\eqref{clean-event:gap} implies that for $n\geq 2$ and $\ell\in E_n$, 
\begin{align}\label{gapeq-balanced}
	\begin{aligned}d_{\sigma(\ell)} - d_{\sigma(n)}&\leq \sqrt{\frac{3}{\sum_{s=1}^{t_n} N_{\text{class of $\sigma(\ell)$},s}}\log N\scale}+\sqrt{\frac{3}{\sum_{s=1}^{t_n} N_{\text{class of $\sigma(n)$},s}}\log N\scale}\\
		&\leq  \sqrt{\frac{3}{\sum_{s=1}^{(n-1)\scale} N_{\text{class of $\sigma(\ell)$},s}}\log N\scale}+\sqrt{\frac{3}{\sum_{s=1}^{(n-1)\scale} N_{\text{class of $\sigma(n)$},s}}\log N\scale} 
	\end{aligned}
\end{align}
where the second inequality is due to our observation that $t_n\geq (n-1)\scale$. Based on~\eqref{gapeq-balanced}, we obtain
\begin{align}\label{bound1-balanced}
	\begin{aligned}
		&\sum_{n\geq 2}\sum_{\ell\in E_n}\left(d_{\sigma(\ell)}-d_{\sigma(n)}\right)\\
		&\leq \sum_{n\geq 2}\sum_{\ell\in E_n}\left(\sqrt{\frac{3}{\sum_{s=1}^{(n-1)\scale} N_{\text{class of $\sigma(\ell)$},s}}\log N\scale}+\sqrt{\frac{3}{\sum_{s=1}^{(n-1)\scale} N_{\text{class of $\sigma(n)$},s}}\log N\scale} \right)\\
		&\leq \sum_{n\geq 2}\sum_{\ell\in [N]}\sqrt{\frac{3}{\sum_{s=1}^{(n-1)\scale} N_{\text{class of $\sigma(\ell)$},s}}\log N\scale}+\sum_{n\geq 2}N\sqrt{\frac{3}{\sum_{s=1}^{(n-1)\scale} N_{\text{class of $\sigma(n)$},s}}\log N\scale}
	\end{aligned}
\end{align}
where the first inequality is directly implied by~\eqref{gapeq-balanced} and the second inequality is because $E_n\subseteq [N]$. We look at the second sum at the last part of inequality~\eqref{bound1-balanced} first.
\begin{equation}\label{bound2-balanced}
	\sum_{n\geq 2}\sqrt{\frac{3}{\sum_{s=1}^{(n-1)\scale} N_{\text{class of $\sigma(n)$},s}}}=\sum_{i\in\I}\sum_{n\geq2:\sigma(n)\in\J_i}\sqrt{\frac{3}{\sum_{s=1}^{(n-1)\scale} N_{\text{class of $\sigma(n)$},s}}}.
\end{equation}
Let $i\in\I$ be a class that $\sigma(1)$ does not belong to. Then for some $2\leq n_{i_1}\leq\cdots \leq n_{i_{N_i}}$, jobs $\sigma(n_{i_1}),\ldots, \sigma(n_{i_{N_i}})$ are in class $i$. Then
\begin{align}\label{bound3-balanced}
	\begin{aligned}
		\sum_{n\geq2:\sigma(n)\in\J_i}\sqrt{\frac{3}{\sum_{s=1}^{(n-1)\scale} N_{\text{class of $\sigma(n)$},s}}}
		&=\sum_{k=1}^{N_i}\sqrt{\frac{3}{\sum_{s=1}^{(n_{i_k}-1)\scale} N_{i,s}}}\\
		&=\sum_{k=1}^{\lceil N_i/2\rceil}\sqrt{\frac{3}{\sum_{s=1}^{(n_{i_k}-1)\scale} N_{i,s}}}+\sum_{k=\lceil N_i/2\rceil+1}^{N_i}\sqrt{\frac{3}{\sum_{s=1}^{(n_{i_k}-1)\scale} N_{i,s}}}\\
		&\leq \sum_{k=1}^{\lceil N_i/2\rceil}\sqrt{\frac{3}{\sum_{s=1}^{(n_{i_k}-1)\scale} N_{i,s}}}+\lfloor\frac{N_i}{2}\rfloor\sqrt{\frac{3}{\sum_{s=1}^{n_{i_{\lceil N_i/2\rceil}}\scale} N_{i,s}}}\\
		&\leq 2\sum_{k=1}^{\lceil N_i/2\rceil}\sqrt{\frac{3}{\sum_{s=1}^{(n_{i_k}-1)\scale} N_{i,s}}}\\
		&\leq 2\sum_{k=1}^{\lceil N_i/2\rceil}\sqrt{\frac{6}{(n_{i_k}-1)\scale N_i}}
	\end{aligned}
\end{align}
where the first inequality is due to $\sum_{s=1}^{(n_{i_{k}}-1)\scale} N_{i,s}\geq \sum_{s=1}^{n_{i_{\lceil N_i/2\rceil}}\scale} N_{i,s}$ for any $k\geq \lceil N_i/2\rceil+1$, the second inequality comes from  $\sum_{s=1}^{(n_{i_{k}}-1)\scale} N_{i,s}\leq \sum_{s=1}^{n_{i_{\lceil N_i/2\rceil}}\scale} N_{i,s}$ for any $k\leq \lceil N_i/2\rceil$, and the last inequality is because at least $N_i/2$ jobs of class $i$ remain in the system until choosing the $\lceil N_i/2\rceil$th job of class $i$. 

If $\sigma(1)$ is of class $i\in \I$, then for some $2\leq n_{i_1}\leq\cdots \leq n_{i_{N_i-1}}$, jobs $\sigma(n_{i_1}),\ldots, \sigma(n_{i_{N_i-1}})$ are in class $i$. Here, if $N_i=1$, then
\begin{equation}\label{bound4-balanced}
	\sum_{n\geq2:\sigma(n)\in\J_i}\sqrt{\frac{3}{\sum_{s=1}^{(n-1)\scale} N_{\text{class of $\sigma(n)$},s}}}=0.
\end{equation}
Now assume that $N_i\geq 2$. Then
\begin{align}\label{bound5-balanced}
	\begin{aligned}
		\sum_{n\geq2:\sigma(n)\in\J_i}\sqrt{\frac{3}{\sum_{s=1}^{(n-1)\scale} N_{\text{class of $\sigma(n)$},s}}}
		&=\sum_{k=1}^{N_i-1}\sqrt{\frac{3}{\sum_{s=1}^{(n_{i_k}-1)\scale} N_{i,s}}}\\
		&=\sum_{k=1}^{\lfloor N_i/2\rfloor}\sqrt{\frac{3}{\sum_{s=1}^{(n_{i_k}-1)\scale} N_{i,s}}}+\sum_{k=\lceil N_i/2\rceil}^{N_i-1}\sqrt{\frac{3}{\sum_{s=1}^{(n_{i_k}-1)\scale} N_{i,s}}}\\
		&\leq \sum_{k=1}^{\lfloor N_i/2\rfloor}\sqrt{\frac{3}{\sum_{s=1}^{(n_{i_k}-1)\scale} N_{i,s}}}+\lfloor\frac{N_i}{2}\rfloor\sqrt{\frac{3}{\sum_{s=1}^{n_{i_{\lfloor N_i/2\rfloor}}\scale} N_{i,s}}}\\
		&\leq 2\sum_{k=1}^{\lfloor N_i/2\rfloor}\sqrt{\frac{3}{\sum_{s=1}^{(n_{i_k}-1)\scale} N_{i,s}}}\\
		&\leq 2\sum_{k=1}^{\lfloor N_i/2\rfloor}\sqrt{\frac{12}{(n_{i_k}-1)\scale N_i}}
	\end{aligned}
\end{align}
where the first inequality is due to $\sum_{s=1}^{(n_{i_{k}}-1)\scale} N_{i,s}\geq \sum_{s=1}^{n_{i_{\lfloor N_i/2\rfloor}}\scale} N_{i,s}$ for any $k\geq \lceil N_i/2\rceil$, the second inequality comes from  $\sum_{s=1}^{(n_{i_{k}}-1)\scale} N_{i,s}\leq \sum_{s=1}^{n_{i_{\lfloor N_i/2\rfloor}}\scale} N_{i,s}$ for any $k\leq \lfloor N_i/2\rfloor$, and the last inequality is because at least $\lfloor N_i/2\rfloor \geq N_i/4$ jobs of class $i$ remain in the system until choosing the $\lfloor N_i/2\rfloor$th job of class $i$. 

Then it follows from~\eqref{bound2-balanced}--\eqref{bound5-balanced} that 
\begin{align}\label{bound6-balanced}
	\begin{aligned}
		&\sum_{n\geq 2}N\sqrt{\frac{3}{\sum_{s=1}^{(n-1)\scale} N_{\text{class of $\sigma(n)$},s}}\log N\scale}\\
		&\leq \sum_{i\in\I:\text{for some $n\geq 2$, $\sigma(n)\in\J_i$}}4N\sqrt{\log N\scale}\sum_{k=1}^{\lceil N_i/2\rceil}\sqrt{\frac{3}{(n_{i_k}-1)\scale N_i}}.
	\end{aligned}
\end{align}
Next, we turn our attention to the first sum at the end of inequality~\eqref{bound1-balanced}. Note that
\begin{equation}\label{bound7-balanced}
	\sum_{n\geq 2}\sum_{\ell\in [N]}\sqrt{\frac{3}{\sum_{s=1}^{(n-1)\scale} N_{\text{class of $\sigma(\ell)$},s}}}=\sum_{i\in\I}N_i\sum_{n\geq2:\sigma(n)\in\J_i}\sqrt{\frac{3}{\sum_{s=1}^{(n-1)\scale} N_{i,s}}}.
\end{equation}
Let $i\in \I$. If $\sigma(1)$ is not in class $i$, then as before, for some $2\leq n_{i_1}\leq\cdots \leq n_{i_{N_i}}$, jobs $\sigma(n_{i_1}),\ldots, \sigma(n_{i_{N_i}})$ are in class $i$. Moreover, 
\begin{align}\label{bound8-balanced}
	\begin{aligned}
		N_i\sum_{n\geq2}\sqrt{\frac{3}{\sum_{s=1}^{(n-1)\scale} N_{i,s}}}
		&\leq N_i\sum_{n=2}^{n_{i_{\lceil N_i/2\rceil}}}\sqrt{\frac{3}{\sum_{s=1}^{(n-1)\scale} N_{i,s}}}+N_i\sum_{n\geq n_{i_{\lceil N_i/2\rceil}}+1}\sqrt{\frac{3}{\sum_{s=1}^{(n-1)\scale} N_{i,s}}}\\
		&\leq N_i\sum_{n=2}^{n_{i_{\lceil N_i/2\rceil}}}\sqrt{\frac{6}{(n-1)\scale N_i}}+N_i\sum_{n\geq n_{i_{\lceil N_i/2\rceil}}+1}\sqrt{\frac{3}{\sum_{s=1}^{(n-1)\scale} N_{i,s}}}\\
		&\leq N_i\sum_{n=2}^{n_{i_{\lceil N_i/2\rceil}}}\sqrt{\frac{6}{(n-1)\scale N_i}}+N_iN\sqrt{\frac{3}{\sum_{s=1}^{n_{i_{\lceil N_i/2\rceil}}\scale} N_{i,s}}}\\
		&\leq N_i\sum_{n=2}^{n_{i_{\lceil N_i/2\rceil}}}\sqrt{\frac{6}{(n-1)\scale N_i}}+2N\sum_{k=1}^{\lceil N_i/2\rceil}\sqrt{\frac{6}{(n_{i_k}-1)\scale N_i}}
	\end{aligned}
\end{align}
where the second inequality is because there are at least $N_i/2$ jobs waiting until the selection of the $\lceil N_i/2\rceil$th job of class $i$, the third inequality is because $\{n\in[N]:n\geq n_{i_{\lceil N_i/2\rceil}}+1\}$ contains at most $N$ elements, and the last inequality follows from
$$\frac{N_i}{2}\cdot\sqrt{\frac{3}{\sum_{s=1}^{n_{i_{\lceil N_i/2\rceil}}\scale} N_{i,s}}}\leq \sum_{k=1}^{\lceil N_i/2\rceil}\sqrt{\frac{6}{\sum_{s=1}^{(n_{i_{k}}-1)\scale} N_{i,s}}}\leq\sum_{k=1}^{\lceil N_i/2\rceil}\sqrt{\frac{6}{(n_{i_k}-1)\scale N_i}}$$ 
which holds true because $n_{i_k}\leq n_{i_{\lceil N_i/2\rceil}}$ for $k\leq \lceil N_i/2\rceil$ and there are at least $N_i/2$ jobs remaining until choosing the $\lceil N_i/2\rceil$th job is chosen.

Now let $i$ be the class of $\sigma(1)$. If $N_i=1$, then 
\begin{equation}\label{bound9-balanced}
	\sum_{n\geq2:\sigma(n)\in\J_i}\sqrt{\frac{3}{\sum_{s=1}^{(n-1)\scale} N_{i,s}}}=0.
\end{equation}
If $N_i\geq 2$, as before, for some $2\leq n_{i_1}\leq\cdots \leq n_{i_{N_i}-1}$, jobs $\sigma(n_{i_1}),\ldots, \sigma(n_{i_{N_i}-1})$ are in class $i$. Then we can similarly argue that
\begin{align}\label{bound10-balanced}
	N_i\sum_{n\geq2}\sqrt{\frac{3}{\sum_{s=1}^{(n-1)\scale} N_{i,s}}}
	\leq N_i\sum_{n=2}^{n_{i_{\lceil N_i/2\rceil}}}\sqrt{\frac{12}{(n-1)\scale N_i}}+2N\sum_{k=1}^{\lceil N_i/2\rceil}\sqrt{\frac{12}{(n_{i_k}-1)\scale N_i}}.
\end{align}
Then~\eqref{bound7-balanced}--\eqref{bound10-balanced} imply that
\begin{align}\label{bound11-balanced}
	\begin{aligned}
		&\sum_{n\geq 2}\sum_{\ell\in [N]}\sqrt{\frac{3}{\sum_{s=1}^{(n-1)\scale} N_{\text{class of $\sigma(\ell)$},s}}\log N\scale}\\
		&\leq \sum_{i\in\I:\text{for some $n\geq 2$, $\sigma(n)\in\J_i$}}2N_i\sqrt{\log N\scale}\sum_{n=2}^{n_{i_{\lceil N_i/2\rceil}}}\sqrt{\frac{3}{(n-1)\scale N_i}}\\
		&\qquad + \sum_{i\in\I:\text{for some $n\geq 2$, $\sigma(n)\in\J_i$}}4N\sqrt{\log N\scale}\sum_{k=1}^{\lceil N_i/2\rceil}\sqrt{\frac{3}{(n_{i_k}-1)\scale N_i}}.
	\end{aligned}
\end{align}
Since~\eqref{bound6-balanced} and~\eqref{bound11-balanced} provide upper bounds on the first and second terms at the rightmost side of~\eqref{bound1-balanced}, we obtain
\begin{align}\label{bound12-balanced}
	\begin{aligned}
		&\sum_{n\geq 2}\sum_{\ell\in E_n}\left(d_{\sigma(\ell)}-d_{\sigma(n)}\right)\scale\\
		&\leq \sum_{i\in\I:\text{for some $n\geq 2$, $\sigma(n)\in\J_i$}}2\scale N_i\sqrt{\log N\scale}\sum_{n=2}^{n_{i_{\lceil N_i/2\rceil}}}\sqrt{\frac{3}{(n-1)\scale N_i}}\\
		&\qquad + \sum_{i\in\I:\text{for some $n\geq 2$, $\sigma(n)\in\J_i$}}8N\scale\sqrt{\log N\scale}\sum_{k=1}^{\lceil N_i/2\rceil}\sqrt{\frac{3}{(n_{i_k}-1)\scale N_i}}.
	\end{aligned}
\end{align}
Consequently, it remains to bound the two terms on the right hand side of inequality~\eqref{bound12-balanced}. We will show that both terms are at most
$$\kappa\cdot \min\left\{IN,\sqrt{I}N\sqrt{\log N},\frac{N^{3/2}}{N_{\min}^{1/2}}\right\}\sqrt{\scale\log N\scale}$$
for some constant $\kappa>0$, completing the proof of Lemma~\ref{lemma:uniform-ub-1}. 
Let us first consider the second sum for which we provide three different bounds. First, the following holds for some constant $\kappa_1>0$:
\begin{align}\label{final-bound1-balanced}
	\begin{aligned}
		&\sum_{i\in\I:\text{for some $n\geq 2$, $\sigma(n)\in\J_i$}}8N\scale\sqrt{\log N\scale}\sum_{k=1}^{\lceil N_i/2\rceil}\sqrt{\frac{3}{(n_{i_k}-1)\scale N_i}}\\
		&\leq \frac{8N\sqrt{\scale\log N\scale}}{\sqrt{N_{\min}}}\sum_{i\in\I:\text{for some $n\geq 2$, $\sigma(n)\in\J_i$}}\sum_{k=1}^{\lceil N_i/2\rceil}\sqrt{\frac{3}{(n_{i_k}-1)}}\\
		&\leq\frac{8N\sqrt{\scale\log N\scale}}{\sqrt{N_{\min}}}\sum_{n\geq 2}\sqrt{\frac{3}{(n-1)}}\\
		&\leq\kappa_1\cdot \frac{N\sqrt{\scale\log N\scale}}{\sqrt{N_{\min}}}\sqrt{N}.\\
		&=\kappa_1\cdot \frac{N^{3/2}}{N_{\min}^{1/2}}\sqrt{\scale\log N\scale}
	\end{aligned}
\end{align}
where the first inequality is by $N_i\geq N_{\min}$ and the second inequality is because each $n_{i_k}$ belongs to $[N]\setminus \{1\}$. Second, for some constant $\kappa_2>0$, the following holds:
\begin{align}\label{final-bound2-balanced}
	\begin{aligned}
		&\sum_{i\in\I:\text{for some $n\geq 2$, $\sigma(n)\in\J_i$}}8N\scale\sqrt{\log N\scale}\sum_{k=1}^{\lceil N_i/2\rceil}\sqrt{\frac{3}{(n_{i_k}-1)\scale N_i}}\\
		&\leq 8N\sqrt{\scale\log N\scale}\sum_{i\in\I:\text{for some $n\geq 2$, $\sigma(n)\in\J_i$}}\sum_{k=1}^{\lceil N_i/2\rceil}\sqrt{\frac{3}{kN_i}}\\
		&\leq 8N\sqrt{\scale\log N\scale}\sum_{i\in\I}\frac{1}{\sqrt{N_i}}\sum_{k=1}^{\lceil N_i/2\rceil}\sqrt{\frac{3}{k}}\\
		&\leq \kappa_2 \cdot N\sqrt{\scale\log N\scale}\sum_{i\in\I}\frac{1}{\sqrt{N_i}}\sqrt{N_i}\\
		&=\kappa_2 \cdot IN\sqrt{\scale\log N\scale}
	\end{aligned}
\end{align}
where the first inequality is because $n_{i_k}\geq k+1$.
Lastly, 
\begin{align}\label{final-bound3-balanced}
	\begin{aligned}
		&\sum_{i\in\I:\text{for some $n\geq 2$, $\sigma(n)\in\J_i$}}8N\scale\sqrt{\log N\scale}\sum_{k=1}^{\lceil N_i/2\rceil}\sqrt{\frac{3}{(n_{i_k}-1)\scale N_i}}\\
		&=8N\sqrt{\scale\log N\scale}\sum_{i\in\I:\text{for some $n\geq 2$, $\sigma(n)\in\J_i$}}\sum_{k=1}^{\lceil N_i/2\rceil}\sqrt{\frac{3}{(n_{i_k}-1)N_i}}\\
		&\leq 8N\sqrt{\scale\log N\scale}\sqrt{\sum_{i\in\I:\text{for some $n\geq 2$, $\sigma(n)\in\J_i$}}\sum_{k=1}^{\lceil N_i/2\rceil}\frac{1}{n_{i_k}-1}}\sqrt{\sum_{i\in \I}\sum_{k=1}^{\lceil N_i/2\rceil}\frac{3}{N_i}}\\
		&\leq 8N\sqrt{\scale\log N\scale}\sqrt{\sum_{n\geq 2}\frac{1}{n}}\sqrt{\sum_{i\in\I}3}\\
		&\leq \kappa_3\cdot \sqrt{I}N\sqrt{\log N}\sqrt{\scale\log N\scale}
	\end{aligned}
\end{align}
for some constant $\kappa_3>0$.
where the first inequality is given by the Cauchy-Schwarz inequality and the last inequality is because ${\sum_{n\geq 2}1/n}=O(\log N)$. 
Hence,~\eqref{final-bound1-balanced}--\eqref{final-bound3-balanced} imply the desired bound on the second sum:
\begin{align}\label{final-bound4-balanced}
	\begin{aligned}
		&\sum_{i\in\I:\text{for some $n\geq 2$, $\sigma(n)\in\J_i$}}8N\scale\sqrt{\log N\scale}\sum_{k=1}^{\lceil N_i/2\rceil}\sqrt{\frac{3}{(n_{i_k}-1)\scale N_i}}\\
		&\leq \max\{\kappa_1,\kappa_2,\kappa_3\}\cdot\min\left\{IN, \sqrt{I}N\sqrt{\log N}, \frac{N^{3/2}}{N_{\min}^{1/2}}\right\}\sqrt{\scale\log N\scale}
	\end{aligned}
\end{align}
Next we consider the first sum. We show that
\begin{align}\label{final-bound5-balanced}
	\begin{aligned}
		&\sum_{i\in\I:\text{for some $n\geq 2$, $\sigma(n)\in\J_i$}}2\scale N_i\sqrt{\log N\scale}\sum_{n=2}^{n_{i_{\lceil N_i/2\rceil}}}\sqrt{\frac{3}{(n-1)\scale N_i}}\\
		&=\sum_{i\in\I:\text{for some $n\geq 2$, $\sigma(n)\in\J_i$}}2\sqrt{N_i\scale\log N\scale}\sum_{n=2}^{n_{i_{\lceil N_i/2\rceil}}}\sqrt{\frac{3}{(n-1)}}\\
		&\leq \sum_{i\in\I:\text{for some $n\geq 2$, $\sigma(n)\in\J_i$}}2\sqrt{N_i\scale\log N \scale}\sum_{n=2}^{N}\sqrt{\frac{3}{(n-1)}}\\
		&\leq \kappa_1\cdot \sqrt{N\scale\log N\scale}\sum_{i\in \I} \sqrt{N_i}\\
		&\leq \kappa_4\cdot \sqrt{N\scale\log N\scale}\cdot \sqrt{IN}\\
		&=\kappa_4\cdot \sqrt{I}\cdot N\sqrt{\scale \log N\scale}
	\end{aligned}
\end{align}
holds for some constant $\kappa_4>0$
where the first inequality is due to $n_{i_{\lceil N_i/2\rceil}}\leq N$, the second inequality is because $\sum_{n=2}^N\sqrt{1/(n-1)} =O(\sqrt{N})$, and the last inequality is by the Cauchy–Schwarz inequality. Lastly, $I N_{\min}\leq N$ implies that $$\sqrt{I}\leq \frac{N^{1/2}}{N_{\min}^{1/2}}.$$
Then it follows from~\eqref{final-bound5-balanced} that
\begin{align}\label{final-bound6-balanced}
	\begin{aligned}
		&\sum_{i\in\I:\text{for some $n\geq 2$, $\sigma(n)\in\J_i$}}2\scale N_i\sqrt{\log N\scale}\sum_{n=2}^{n_{i_{\lceil N_i/2\rceil}}}\sqrt{\frac{3}{(n-1)\scale N_i}}\\
		&\leq \kappa_4\cdot \min\left\{IN, \sqrt{I}N\sqrt{\log N}, \frac{N^{1/3}}{N_{\min}^{1/2}}\right\} \sqrt{\scale \log N\scale}.
	\end{aligned}
\end{align}
Finally, combining~\eqref{bound12-balanced},~\eqref{final-bound4-balanced}, and~\eqref{final-bound6-balanced}, we show that
$$\sum_{j\geq 2}\sum_{\ell\in E_j}\left(d_{\sigma(\ell)}-d_{\sigma(j)}\right)\scale\leq \min\left\{IN, \sqrt{I}N\sqrt{\log N}, \frac{N^{1/3}}{N_{\min}^{1/2}}\right\} \sqrt{\scale \log N\scale},$$
as required.

Therefore, note that
\begin{align*}
	&\mathbb{E}[R^\pi]\\&= \mathbb{E}[R^\pi\mid \neg~\text{clean event}]\cdot \mathbb{P}[\neg~\text{clean event}] +\mathbb{E}[R^\pi \mid \text{clean event}]\cdot \mathbb{P}[\text{clean event}]\\
	&=O\left(\frac{2}{N^4\scale}N^2\scale+\right.\\
	&\qquad\left.\left(1-\frac{2}{N^4\scale}\right)\left(N \Ts + \frac{N\scale(\log N\scale)^{1/2}}{N_{\min}^{1/2}(\Ts+1)^{1/2}}+\min\left\{IN,I^{1/2}N{(\log N)^{1/2}},\frac{N^{3/2}}{N_{\min}^{1/2}}\right\}({\scale\log N\scale})^{1/2}\right)\right)\\
	&=O\left(N \Ts + \frac{N\scale(\log N\scale)^{1/2}}{N_{\min}^{1/2}(\Ts+1)^{1/2}}+\min\left\{IN,I^{1/2}N{(\log N)^{1/2}},\frac{N^{3/2}}{N_{\min}^{1/2}}\right\}({\scale\log N\scale})^{1/2}\right),
\end{align*}
which completes the proof of Lemma~\ref{lemma:uniform-ub-1}.

\subsection{Final step: plugging in the length of the preemption phase}\label{sec:uniform-ub-step2}

Recall that the first two terms in~\eqref{regret:uniform-ub-1'} has dependence on $\Ts$. To decide a value for $\Ts$ asymptotically minimizing the sum of the two terms, we consider function $f$ defined as follows:
$$f(x):= x + \frac{\scale(\log N\scale)^{1/2}}{N_{\min}^{1/2}(x+1)^{1/2}},\quad x\geq -1.$$
Note that the derivative of $f$ is given by
$$f^\prime(x)=1-\frac{\scale(\log N\scale)^{1/2}}{2N_{\min}^{1/2}(x+1)^{3/2}}.$$
Then we have
$$
\begin{cases}
	f^\prime(x)>0,&\text{if $x>2^{-2/3}N_{\min}^{-1/3}\scale^{2/3}(\log N\scale)^{1/3}-1$}\\
	f^\prime(x)=0,&\text{if $x=2^{-2/3}N_{\min}^{-1/3}\scale^{2/3}(\log N\scale)^{1/3}-1$}\\
	f^\prime(x)<0,&\text{if $-1<x<2^{-2/3}N_{\min}^{-1/3}\scale^{2/3}(\log N\scale)^{1/3}-1$}.
\end{cases}
$$
Therefore, it follows that
$$f(x)\geq (2^{-2/3}+2^{1/3})N_{\min}^{-1/3}\scale^{2/3}(\log N\scale)^{1/3}-1.$$
As in Section~\ref{sec:uniform}, we use notation $$\bar{\Ts}=N_{\min}^{-1/3} \scale^{2/3}\left(\log N\scale\right)^{1/3}.$$ This provides an intuition for our choice of $\Ts$ given in~\eqref{eq:Ts-choice}. We next formalize the intuition by proving the following lemma.
\begin{lemma}\label{lemma:Ts-choice}
	If $\Ts$ is given as in~\eqref{eq:Ts-choice}, then the following holds
	\begin{equation}
		\label{eq:T_s-choice-regret}
		\Ts + \frac{\scale(\log N\scale)^{1/2}}{N_{\min}^{1/2}(\Ts+1)^{1/2}}=O\left(\max\left\{\frac{\scale^{2/3}(\log N\scale)^{1/3}}{N_{\min}^{1/3}},\ \frac{\scale^{1/2}({\log N\scale})^{1/2}}{N_{\min}^{1/2}},\ ({\log N\scale})^{1/2}\right\}\right)
	\end{equation}
\end{lemma}
\begin{proof}
	If $\bar{\Ts}\geq 1$, then it follows that $(2^{-2/3}+2^{1/3})\bar{\Ts}-1\geq (2^{-2/3}+2^{1/3}-1)\bar{\Ts}$. This implies that
	$$\min_{x\geq 0}f(x)=\Omega(\bar{\Ts}),\quad\text{if $\bar{\Ts}\geq 1$}.$$
	On the other hand, we have for any $\bar{\Ts}\geq 0$,
	$$
	f(\lfloor \bar{\Ts}\rfloor )=\lfloor \bar{\Ts}\rfloor + \frac{\scale(\log N\scale)^{1/2}}{N_{\min}^{1/2}(\lfloor \bar{\Ts}\rfloor +1)^{1/2}}\leq \bar{\Ts}+\frac{\scale(\log N\scale)^{1/2}}{N_{\min}^{1/2}\bar{\Ts}^{1/2}}=2\bar{\Ts}.
	$$
	Consequently, $\lfloor \bar{\Ts}\rfloor$ asymptotically minimizes $f$ if $\bar{\Ts}\geq 1$. Moreover, if $\bar{\Ts}<\scale$, then we may set $\Ts=\lfloor \bar{\Ts}\rfloor$. In this case, 
	$$\Ts + \frac{\scale(\log N\scale)^{1/2}}{N_{\min}^{1/2}(\Ts+1)^{1/2}}=O\left(\frac{\scale^{2/3}(\log N\scale)^{1/3}}{N_{\min}^{1/3}}\right),$$
	which gives rise to the bound~\eqref{eq:T_s-choice-regret}. Therefore, when $1\leq \bar{\Ts}<\scale$, Algorithm~\ref{preempt-then-nonpreemptive} with $\Ts=\lfloor \bar{\Ts}\rfloor$ achieves~\eqref{eq:T_s-choice-regret}.
	
	Next, let us consider the case $\bar{\Ts}\geq \scale\geq 1$. In this case, we set $\Ts=S-1$, and as a result, the regret upper bound~\eqref{regret:uniform-ub-1'} becomes
	\begin{equation}\label{eq:regret-ub-1-final}
		\Ts + \frac{\scale(\log N\scale)^{1/2}}{N_{\min}^{1/2}(\Ts+1)^{1/2}}=(S-1)+\frac{\scale^{1/2}(\log N\scale)^{1/2}}{N_{\min}^{1/2}}.
	\end{equation}
	It is straightforward that the second sum on the right-hand side of~\eqref{eq:regret-ub-1-final} is subsumed by the second term on the right-hand side of~\eqref{eq:T_s-choice-regret}. Moreover, since $\bar{\Ts}>S$, we know that $\log N\scale > N_{\min}\scale$, and thus $S\leq \log N\scale/N_{\min}$. In particular, $(S-1)\leq S^{2/3}(\log N\scale)^{1/3}/N_{\min}^{1/3}$. Therefore,~\eqref{eq:T_s-choice-regret} also holds when $\bar{\Ts}\geq \scale\geq 1$.
	
	Lastly, we consider the case where $\bar{\Ts}<1$. In this case, we set $\Ts=\lfloor \bar{\Ts}\rfloor =0$. As a result, the upper bound~\eqref{regret:uniform-ub-1'} reduces to 
	\begin{equation}\label{eq:regret-ub-1-final'}
		\Ts + \frac{\scale(\log N\scale)^{1/2}}{N_{\min}^{1/2}(\Ts+1)^{1/2}}=\frac{\scale(\log N\scale)^{1/2}}{N_{\min}^{1/2}}.
	\end{equation}
	Here, as $\bar{\Ts}<1$, it follows that $N_{\min}^{-1/3}\scale^{2/3} <1$ and thus $\scale< N_{\min}^{1/2}$. This means that the right-hand side of~\eqref{eq:regret-ub-1-final'} is less than $(\log N\scale)^{1/2}$, implying in turn that it is less than or equal to the third term on the right-hand side of~\eqref{eq:T_s-choice-regret}. Hence,~\eqref{eq:T_s-choice-regret} holds true when $\bar{\Ts}<1$.
\end{proof}

Lemma~\ref{lemma:Ts-choice} shows that~\eqref{regret:uniform-ub-1'} with $\Ts$ given in~\eqref{eq:Ts-choice} is bounded above by~\eqref{regret:uniform-ub-1''}. Therefore, by Lemma~\ref{lemma:uniform-ub-1},~\eqref{regret:uniform-ub-1''} is indeed an upper bound on the expected regret of Algorithm~\ref{preempt-then-nonpreemptive}. Lastly, we obtain the upper bound~\eqref{regret:uniform-ub-1''} because $N_{\min}\geq 1$, as required.

\section{Proof of Theorem~\ref{thm:uniform-lb}}\label{sec:proof:uniform-lb}

In this section, we prove Theorem~\ref{thm:uniform-lb}. To prove that the expected regret of any (randomized) scheduling algorithm has a lower bound of $$\Omega\left(\max\left\{\barn^{2/3}\scale^{2/3},\ N^{1/2}\barn^{1/2}\scale^{1/2} \right\}\right),$$
we show that $\Omega\left(\barn^{2/3}\scale^{2/3}\right)$ and $\Omega\left(N^{1/2}\barn^{1/2}\scale^{1/2}\right)$ are two lower bounds on the expected regret. To explain our proof strategy, let us take some nonempty sets $\I_1$ and $\I_2$ partitioning $\I$, the set of all classes.
Let 
\begin{equation}\label{eq:Ms}
	M_1:=\sum_{i\in \I_1}N_i,\quad M_2:=\sum_{i\in \I_2}N_i,
\end{equation}
and assume that $M_1\geq M_2$.
In Section~\ref{sec:uniform-lb-first}, we show that the expected regret of any (randomized) scheduling algorithm is bounded below by $\Omega\left(M_2^{2/3}\scale^{2/3}\right)$ (under some mild condition), and in Section~\ref{sec:uniform-lb-second}, we prove that the expected regret is bounded below by $\Omega\left(M_1^{1/2}M_2^{1/2}\scale^{1/2}\right)$. These two lower bounds hold true for any partition $(\I_1, \I_2)$ of $\I$. In particular, in Section~\ref{sec:uniform-lb-final}, we show that there always exist a partition $(\I_1, \I_2)$ such that $M_1=\Omega(\barn)$ and a partition $(\I_1, \I_2)$ such that $M_1M_2=\Omega(N\barn)$. This in turn gives us the desired lower bound on the expected regret.

Let $\I_1$ and $\I_2$ be some nonempty sets partitioning $\I$, the set of all classes. Let $M_1$ and $M_2$ be defined as in~\eqref{eq:Ms}, and assume that $M_1\geq M_2$.
Let us consider the following family of two problem instances, with parameter $\epsilon>0$ to be decided later:
\begin{equation}\label{lb-instance-1}
	\mathcal{P}_1 =\begin{cases}
		c_i=1/2&\text{for each class}\ i\in \I_1\\
		c_i = (1+\epsilon)/2& \text{for each class}\ i\in \I_2
	\end{cases}
\end{equation}
and
\begin{equation}\label{lb-instance-2}
	\mathcal{P}_2 =\begin{cases}
		c_i=1/2&\text{for each class}\ i\in \I_1\\
		c_i = (1-\epsilon)/2& \text{for each class}\ i\in \I_2.
	\end{cases}
\end{equation} 
Moreover, we consider an additional problem instance 
$$\mathcal{P}_0 = \left\{c_i = {1}/{2}\quad \text{for every class}\ i\in\I\right.$$
which we refer to as the "base instance". We fix a scheduling algorithm, and we will analyze the expected regret of the algorithm under the problem instances.

For each job $n\in [N]$, define the $t$-round sample space $\Omega_n^t=\{0,1\}^t$, where each outcome corresponds to a particular realization of the random cost values $X_{n,1},\ldots, X_{n,t}$ of job $n$ for the first $t$ time steps. We focus on
$$\Omega = \prod_{n\in[N]}\Omega_n^t$$
so that the random costs of the $N$ jobs for the first $t$ time steps can be considered. Note that the ``actual" sample space can be strictly smaller than $\Omega$, because a job leaves the system after being chosen in $\scale$ time slots. Nevertheless, we consider $\Omega$ in our analysis.

We define distribution $\mathbb{P}_0$ on $\Omega$ as
$$
\mathbb{P}_0[A]= \mathbb{P}[ A \mid \mathcal{P}_0]\quad\text{for each}\ A\subseteq \Omega.
$$
Similarly, for each $k\in\{1,2\}$, let distribution $\mathbb{P}_k$ on $\Omega$ be defined as
$$
\mathbb{P}_k[A]= \mathbb{P}[A \mid \mathcal{P}_k] \quad\text{for each}\ A\subseteq \Omega.
$$
Note that, for $k\in\{0,1,2\}$, $\mathbb{P}_k$ can be expressed as
$$
\mathbb{P}_k=\prod_{i\in[N],s\in[t]}\mathbb{P}_k^{n,s}
$$
where $\mathbb{P}_k^{n,s}$ is the distribution of the random cost of job $n$ at time step $t$. Based on the notion of Kullback–Leibler(KL)-divergence, we obtain the following for each event $A\subseteq \Omega$: 
\begin{equation}\label{eq:tv-2}
	2\left(\mathbb{P}_0[A]-\mathbb{P}_k[A]\right)^2\leq \mathrm{KL}(\mathbb{P}_0,\mathbb{P}_k)=\sum_{n\in[N]}\sum_{s\in[t]} \mathrm{KL}(\mathbb{P}_0^{n,s},\mathbb{P}_k^{n,s})
\end{equation}
where $\mathrm{KL}(\mathbb{P},\mathbb{Q})$ denotes the KL-divergence between distributions $\mathbb{P}$ and $\mathbb{Q}$.
Here, the first inequality directly follows from Pinsker's inequality, which bounds the total variation distance between two distributions. The equality follows from a property of KL-divergence on product distributions. For more details on KL-divergence, we refer the reader to~\cite[Section 2.1]{MAL-068}.

\begin{lemma}\label{KL:variant}
	Let $\mu\geq1$ and $0<c\leq \mu$, and let $p_0$ be a probability distribution on $\{0,\mu\}$ with $p_0(\mu)=c/2\mu$. Let $\epsilon\in(-1/\sqrt{2},1/\sqrt{2})$ and $p_\epsilon$ be a probability distribution on the same sample space $\{0,\mu\}$ with $p_\epsilon(\mu)=(1+\epsilon)c/2\mu$. Then
	$$
	\mathrm{KL}(p_0,p_\epsilon)\leq \frac{c}{\mu}\epsilon^2.
	$$
\end{lemma}
\begin{proof}
	By definition, we have
	\begin{align*}
		\mathrm{KL}(p_0,p_\epsilon)&= \frac{c}{2\mu}\log\frac{1}{1+\epsilon}+\left(1-\frac{c}{2\mu}\right)\log\frac{1-{c}/{2\mu}}{1-{(1+\epsilon)c}/{2\mu}}\\
		&=\frac{c}{2\mu}\log\frac{1}{1+\epsilon}+\frac{2\mu-c}{2\mu}\log\frac{2\mu-c}{2\mu-(1+\epsilon)c
		}\\
		&=-\frac{c}{2\mu}\log(1+\epsilon)-\frac{1}{2\mu}\log\left(1- \frac{\epsilon c}{2\mu-c}\right)^{2\mu-c}\\
		&=-\frac{c}{2\mu}\log(1+\epsilon)\left(1- \frac{\epsilon}{2\mu/c-1}\right)^{2\mu/c-1}.
	\end{align*}
	
	Since $2\mu/c-1\geq 1$, it follows from basic calculus that for $0\leq \epsilon\leq 1$, 
	$$\left(1- \frac{\epsilon}{2\mu/c-1}\right)^{2\mu/c-1}\geq 1-\epsilon.$$
	Moreover, if $-1\leq \epsilon\leq 0$, then 
	$$\left(1- \frac{\epsilon}{2\mu/c-1}\right)^{2\mu/c-1}\geq 1-\epsilon.$$
	Since $\log$ is an increasing function, we have that
	$$-\frac{c}{2\mu}\log(1+\epsilon)\left(1- \frac{\epsilon}{2\mu/c-1}\right)^{2\mu/c-1}\leq -\frac{c}{2\mu}\log (1+\epsilon)(1-\epsilon) = -\frac{c}{2\mu}\log \left(1-\epsilon^2\right).$$
	Moreover, as long as $0\leq \epsilon^2\leq 1/2$, $\log \left(1-\epsilon^2\right)\geq -2\epsilon^2$. Therefore, 
	$$
	\mathrm{KL}(p_0,p_\epsilon)\leq -\frac{c}{2\mu}\log \left(1-\epsilon^2\right)\leq \frac{c}{\mu}\epsilon^2,
	$$
	as required.
\end{proof}
By Lemma~\ref{KL:variant} and~\eqref{eq:tv-2},
\begin{equation}\label{eq:tv-3}
	2\left(\mathbb{P}_0[A]-
	\mathbb{P}_k[A]\right)^2\leq \epsilon^2t\sum_{i\in \I_2}N_i.
\end{equation}

\subsection{First lower bound}\label{sec:uniform-lb-first}

Let $(\I_1,\I_2)$ be a partition of $\I$, the set of of all classes, such that $M_1=\sum_{i\in \I_1}N_i\geq M_2=\sum_{i\in \I_2}N_i$. Define $\mathcal{P}_1$ and $\mathcal{P}_2$ as in~\eqref{lb-instance-1} and~\eqref{lb-instance-2}, respectively.

\begin{theorem}\label{thm:first-lb}
	Fix any (randomized) scheduling algorithm $\pi$. Choose $k$ from $\{1,2\}$ uniformly at random, and run the algorithm on instance $\mathcal{P}_k$. Assume that $M_2^{-1/3}\scale^{2/3}\geq1$ where $M_2=\sum_{i\in \I_2}N_i$.
	Then
	$$
	\mathbb{E}[R^\pi]=\Omega\left(M_2^{2/3}\scale^{2/3}\right)$$
	where the expectation is taken over the choice of $k$ and the randomness in holding costs and the algorithm.
\end{theorem}
\begin{proof}
	We set
	$$T_0=\lfloor M_2^{-1/3}\scale^{2/3}\rfloor\quad\text{and}\quad\epsilon=
	\frac{M_2^{-1/3}\scale^{-1/3}}{3}.
	$$
	Since $M_2^{-1/3}\scale^{2/3}\geq 1$, we have
	\begin{equation}\label{eq:first-T0}
		\frac{M_2^{-1/3}\scale^{2/3}}{2}\leq T_0   \leq M_2^{-1/3}\scale^{2/3}.
	\end{equation}
	Then we consider the $T_0$-round sample space $\Omega_n^{T_0}=\{0,1\}^{T_0}$ of each job $n\in[N]$, and we define $\Omega$ as before. Then it follows from~\eqref{eq:tv-3} that for any event $A\subseteq \Omega$, 
	\begin{equation}\label{eq:KLgap-1}
		\left|\mathbb{P}[ A \mid \mathcal{P}_0]-\mathbb{P}[ A \mid \mathcal{P}_k]\right|\leq \frac{1}{3}\quad\text{for}~k\in\{1,2\}.
	\end{equation}
	Let $B\subseteq\Omega$ be the event that algorithm $\pi$ chooses a job from some class in $\I_2$ in at least $T_0/2$ time slots until the end of the $T_0$th time slot. Then under $\neg B\subseteq \Omega$, algorithm $\pi$ chooses a job from some class in $\I_1$ in at least $T_0/2$ time slots until the end of the $T_0$th time slot. Furthermore, we have that
	$$
	\mathbb{P}[ B \mid \mathcal{P}_0]+\mathbb{P}[ \neg B \mid \mathcal{P}_0]=1.$$
	Notice that
	\begin{align}\label{eq:first-lb-1}
		\begin{aligned}
			\mathbb{E}[R^\pi]&=\frac{1}{2}\mathbb{E}[ R^\pi \mid \mathcal{P}_1]+\frac{1}{2}\mathbb{E}[ R^\pi \mid \mathcal{P}_2]\\
			&=\frac{1}{2}\sum_{k\in\{1,2\}}\mathbb{P}[B \mid \mathcal{P}_{k}]\cdot\mathbb{E}[ R^\pi \mid B, \mathcal{P}_k]+\frac{1}{2}\sum_{k\in\{1,2\}}\mathbb{P}[\neg B \mid \mathcal{P}_{k}]\cdot\mathbb{E}[ R^\pi \mid \neg B, \mathcal{P}_k].
		\end{aligned}
	\end{align}
	If $\mathbb{P}[ B \mid \mathcal{P}_0]\geq {1}/{2}$, then by~\eqref{eq:KLgap-1}, we have $\mathbb{P}[ B \mid \mathcal{P}_k]\geq {1}/{6}$ for $k\in\{1,2\}$. In this case, we deduce from~\eqref{eq:first-lb-1} that
	\begin{equation}\label{eq:first-lb-case1}
		\mathbb{E}[R^\pi]\geq\frac{1}{12}\mathbb{E}[ R^\pi \mid B, \mathcal{P}_2].
	\end{equation}
	If not, we have $\mathbb{P}[ \neg B \mid \mathcal{P}_0]\geq {1}/{2}$, and therefore, $\mathbb{P}[ \neg B \mid \mathcal{P}_k]\geq {1}/{6}$ for $k\in\{1,2\}$ by~\eqref{eq:KLgap-1}. Then it follows from~\eqref{eq:first-lb-1} that
	\begin{equation}\label{eq:first-lb-case2}
		\mathbb{E}[R^\pi]\geq\frac{1}{12}\mathbb{E}[ R^\pi \mid \neg B, \mathcal{P}_1].
	\end{equation}
	Basically, thanks to~\eqref{eq:first-lb-case1} and~\eqref{eq:first-lb-case2}, it is sufficient to bound the terms $\mathbb{E}[ R^\pi \mid B, \mathcal{P}_2]$ and $\mathbb{E}[ R^\pi \mid \neg B, \mathcal{P}_1]$.

	As in Section~\ref{sec:regret-basic}, we assume that $c_1\geq c_2\geq \cdots\geq c_I$. Then we number the $N$ jobs from $1$ to $N$ so that jobs $1+\sum_{n\in[i-1]}N_n,\ldots, \sum_{n\in[i]}N_n$ belong to class $i$. Let $d_n$ denote the mean per-time holding cost of job $n\in[N]$. Then, if job $n$ is of class $i$, then we have $d_n=c_i$. Let $\sigma:[N]\to[N]$ be the permutation of $[N]$ that gives the sequence of jobs completed by the algorithm. 
	
	Next, let $T_n$ denote the number of time steps where job $n$ is processed by the scheduling algorithm during the period of the first $T_0$ time steps. Notice that $T_0\leq \scale$, so no job finishes until the $T_0$th time slot. This means that $T_0-\sum_{\ell=1}^n T_{\sigma(\ell)}$ time slots are used to serve jobs other than $\sigma(1),\ldots,\sigma(n)$, and therefore, $W_n\geq T_0-\sum_{\ell=1}^n T_{\sigma(\ell)}$. Then it follows from  Lemma~\ref{lemma:regret-1} that
	\begin{equation}\label{eq:lb-first-bound1}
		R^\pi\geq\sum_{n\in[N]}d_{\sigma(n)}\left(T_0 -\sum_{\ell=1}^{n}T_{\sigma(\ell)} \right)+  \sum_{n\in[N]}\sum_{\ell\in E_n}\left(d_{\sigma(\ell)}-d_{\sigma(n)}\right)\scale.
	\end{equation}
	
	Consider the case where we are under the instance $\mathcal{P}_2$ and the event $B$. Let $n_2$ be the smallest number such that $\sigma(n_2)$ belongs to a class in $\I_2$. Then $d_{\sigma(n_2)}=(1-\epsilon)/2$ and $n_2\leq \sum_{i\in \I_1} N_i+1=N-M_2 +1$. Notice that
	$$|E_{n_2}|= \sum_{i\in \I_1} N_i - (n_2-1)=(N-M_2) - (n_2-1).$$
	If $n_2\leq (N-M_2+1)/2$, then we have $|E_{n_2}|\geq (N-M_2)/2$. Moreover, if $\ell\in E_{n_2}$, then $\sigma(\ell)$ belongs to a class in $\I_1$, meaning that $d_{\sigma(\ell)}=1/2$. Then it follows from~\eqref{eq:lb-first-bound1} that
	\begin{equation}\label{eq:lb-first-bound2}
		R^\pi\geq\sum_{\ell\in E_{n_2}}\left(d_{\sigma(\ell)}-d_{\sigma({n_2})}\right)\scale=\frac{\epsilon}{2}\cdot |E_{n_2}|\scale\geq \frac{1}{12}\cdot  (N-M_2)M_2^{-1/3} \scale^{2/3}.
	\end{equation}
	If ${n_2}> (N-M_2+1)/2$, then jobs $\sigma(1),\ldots,\sigma(\lfloor (N-M_2+1)/2\rfloor)$ belong to some classes in $\I_1$. Since we are under the event $B$, 
	$$\sum_{i\in \I_1}\sum_{n\in\J_i}T_{n}\leq\frac{T_0}{2}.$$
	This implies that for any $n\leq \lfloor (N-M_2+1)/2\rfloor$,
	$$T_0-\sum_{\ell=1}^nT_{\sigma(\ell)}\geq T_0-\sum_{i\in \I_1}\sum_{n\in\J_i}T_{n}\geq \frac{T_0}{2}.$$
	Hence, from~\eqref{eq:lb-first-bound1}, we obtain
	\begin{equation}\label{eq:lb-first-bound3}
		R^\pi\geq \sum_{n=1}^{\lfloor (N-M_2+1)/2\rfloor}d_{\sigma(n)}\left(T_0-\sum_{\ell=1}^nT_{\sigma(\ell)}\right)=\lfloor \frac{N-M_2+1}{2}\rfloor\cdot \frac{1}{2}\cdot \frac{T_0}{2}\geq \frac{1}{16}\cdot (N-M_2)M_2^{-1/3}\scale^{2/3}
	\end{equation}
	where the last inequality is from~\eqref{eq:first-T0} which says that $T_0\geq M_2^{-1/3}\scale^{2/3}/2$.
	Based on~\eqref{eq:lb-first-bound2} and~\eqref{eq:lb-first-bound3}, we obtain
	\begin{equation}\label{eq:lb-first-bound4}
		\mathbb{E}\left[R^\pi \mid B,\mathcal{P}_2\right]\geq \frac{1}{16}\cdot (N-M_2)M_2^{-1/3} \scale^{2/3}\geq \frac{1}{16}\cdot M_2^{2/3} \scale^{2/3}
	\end{equation}
	where the second inequality comes from
	$N-M_2=N_1=\sum_{i\in \I_1}N_i\geq\sum_{i\in \I_2}N_i=M_2$.
	
	Next assume that we are under the instance $\mathcal{P}_1$ and the event $\neg B$. Let ${n_1}$ be the smallest number such that $\sigma({n_1})$ belongs to a class in $\I_1$. Then $d_{\sigma({n_1})}=1/2$ and ${n_1}\leq \sum_{i\in \I_2} N_i+1=M_2 +1$. Note that  
	$$|E_{n_1}|= \sum_{i\in \I_2} N_i - ({n_1}-1)=M_2 - ({n_1}-1).$$
	If ${n_1}\leq (M_2+1)/2$, then we have $|E_{n_1}|\geq M_2/2$. Note also that if $\ell\in E_{n_1}$, then $\sigma(\ell)$ belongs to a class in $\I_2$, which implies that $d_{\sigma(\ell)}=(1+\epsilon)/2$. Then, by~\eqref{eq:lb-first-bound1}, we obtain
	\begin{equation}\label{eq:lb-first-bound5}
		R^\pi\geq\sum_{\ell\in E_{n_1}}\left(d_{\sigma(\ell)}-d_{\sigma({n_1})}\right)\scale=\frac{\epsilon}{2}\cdot |E_{n_1}|\scale\geq\frac{1}{12}\cdot M_2^{2/3} \scale^{2/3}.
	\end{equation}
	If ${n_1}> (M_2+1)/2$, then jobs $\sigma(1),\ldots,\sigma(\lfloor (M_2+1)/2\rfloor)$ belong to some class in $\I_2$. As we are under the event $\neg B$, 
	$$\sum_{i\in \I_2}\sum_{n\J_i}T_{n}\leq\frac{T_0}{2}.$$
	Then it follows that for any $n\leq \lfloor (M_2+1)/2\rfloor$,
	$$T_0-\sum_{\ell=1}^nT_{\sigma(\ell)}\geq T_0-\sum_{i\in \I_1}\sum_{n\J_i}T_{n}\geq \frac{T_0}{2}.$$
	Therefore, we obtain from~\eqref{eq:lb-first-bound1} that
	\begin{equation}\label{eq:lb-first-bound6}
		R^\pi\geq \sum_{n=1}^{\lfloor (M_2+1)/2\rfloor}d_{\sigma(n)}\left(T_0-\sum_{\ell=1}^nT_{\sigma(\ell)}\right)=\lfloor \frac{M_2+1}{2}\rfloor\cdot \frac{1}{2}\cdot \frac{T_0}{2}\geq \frac{1}{16}\cdot M_2^{2/3} \scale^{2/3}
	\end{equation}
	where the last inequality is from~\eqref{eq:first-T0} which says that $T_0\geq M_2^{-1/3}\scale^{2/3}/2$.
	Based on~\eqref{eq:lb-first-bound5} and~\eqref{eq:lb-first-bound6},
	\begin{equation}\label{eq:lb-first-bound7}
		\mathbb{E}\left[R^\pi \mid \neg B,\mathcal{P}_1\right]\geq \frac{1}{16}\cdot M_2^{2/3} \scale^{2/3}.
	\end{equation}
	By~\eqref{eq:first-lb-case1},~\eqref{eq:first-lb-case2},~\eqref{eq:lb-first-bound4}, and~\eqref{eq:lb-first-bound7}, we have finally proved that
	$\mathbb{E}[R^\pi]=\Omega\left(M_2^{2/3}\scale^{2/3}\right)$,
	as required.
\end{proof}

\subsection{Second lower bound}\label{sec:uniform-lb-second}

We next provide the second lower bound. As in Section~\ref{sec:uniform-lb-first}, we consider some nonempty sets $\I_1$ and $\I_2$ partitioning $\I$, the set of all classes and prove a lower bound that is a function of $M_1=\sum_{i\in \I_1}N_i$ and $M_2=\sum_{i\in \I_2}N_i$. Let $(\I_1,\I_2)$ be a partition of $\I$, the set of of all classes. Define $\mathcal{P}_1$ and $\mathcal{P}_2$ as in~\eqref{lb-instance-1} and~\eqref{lb-instance-2}, respectively.

\begin{theorem}\label{thm:second-lb}
	Fix any (randomized) scheduling algorithm $\pi$. Choose $k$ from $\{1,2\}$ uniformly at random, and run the algorithm on instance $\mathcal{P}_k$. Let $M_1=\sum_{i\in \I_1}N_i$ and $M_2=\sum_{i\in \I_2}N_i$.
	Then
	$$
	\mathbb{E}[R^\pi]=\Omega\left(M_1^{1/2}M_2^{1/2}\scale^{1/2}\right)$$
	where the expectation is taken over the choice of $k$ and the randomness in holding costs and the algorithm.
\end{theorem}
\begin{proof}
	Without loss of generality, assume that $M_1\geq M_2$. We set 
	$$\epsilon=\frac{M_1^{-1/2}M_2^{-1/2}\scale^{-1/2}}{2}.$$
	We consider
	$$T_0=\lfloor\frac{M_1\scale}{2}\rfloor$$
	and the $T_0$-round sample space $\Omega_n^{T_0}=\{0,1\}^{T_0}$ of each job $n\in[N]$, and we define $\Omega$ as in Section~\ref{sec:uniform-lb-first}. With our choice of $\epsilon$ and $T_0$, it follows from~\eqref{eq:tv-3} that for any event $A\subseteq \Omega$, 
	\begin{equation}\label{eq:KLgap}
		\left|\mathbb{P}[ A \mid \mathcal{P}_0]-\mathbb{P}[ A \mid \mathcal{P}_k]\right|\leq \frac{1}{4}\quad\text{for}~k\in\{1,2\}.
	\end{equation}
	Let $B\subseteq\Omega$ be the event that algorithm $\pi$ chooses a job from classes in $\I_2$ in at least $M_2\scale/4$ time slots until the end of the $T_0$th time slot. Then under $\neg B\subseteq \Omega$, algorithm $\pi$ chooses a job from classes in $\I_2$ in at most $M_2\scale/4$ time slots until the end of the $T_0$th time slot. Then we have $\mathbb{P}[ B \mid \mathcal{P}_0]+\mathbb{P}[ \neg B \mid \mathcal{P}_0]=1$. Following~\eqref{eq:first-lb-1}, we obtain \begin{equation}\label{eq:second-lb-1}
		\mathbb{E}[R^\pi]=\frac{1}{2}\sum_{k\in\{1,2\}}\mathbb{P}[B \mid \mathcal{P}_{k}]\cdot\mathbb{E}[ R^\pi \mid B, \mathcal{P}_k]+\frac{1}{2}\sum_{k\in\{1,2\}}\mathbb{P}[\neg B \mid \mathcal{P}_{k}]\cdot\mathbb{E}[ R^\pi \mid \neg B, \mathcal{P}_k].
	\end{equation}
	If $\mathbb{P}[ B \mid \mathcal{P}_0]\geq {1}/{2}$, then by~\eqref{eq:KLgap}, we have $\mathbb{P}[ B \mid \mathcal{P}_k]\geq {1}/{4}$ for $k\in\{1,2\}$. If not, we have $\mathbb{P}[ \neg B \mid \mathcal{P}_0]\geq {1}/{2}$, and therefore, $\mathbb{P}[ \neg B \mid \mathcal{P}_k]\geq {1}/{4}$ for $k\in\{1,2\}$ by~\eqref{eq:KLgap}. Therefore, we know that one of $\mathbb{P}[ B \mid \mathcal{P}_k]\geq {1}/{4}$ and $\mathbb{P}[ \neg B \mid \mathcal{P}_k]\geq {1}/{4}$ must hold, implying in turn that
	\begin{equation}\label{eq:second-lb-cases}
		\mathbb{E}[R^\pi]\geq\frac{1}{8}\mathbb{E}[ R^\pi \mid B, \mathcal{P}_2]~~\text{or}~~\mathbb{E}[R^\pi]\geq\frac{1}{8}\mathbb{E}[ R^\pi \mid \neg B, \mathcal{P}_1].
	\end{equation}
	Hence, based on~\eqref{eq:second-lb-cases}, it is sufficient to show that
	\begin{equation}\label{eq:second-lb-conc}
		\mathbb{E}[ R^\pi \mid B, \mathcal{P}_2]=\Omega\left(M_1^{1/2}M_2^{1/2} \scale^{1/2}\right)~~\text{and}~~\mathbb{E}[ R^\pi \mid \neg B, \mathcal{P}_1]=\Omega\left(M_1^{1/2}M_2^{1/2} \scale^{1/2}\right).
	\end{equation}
	We first consider the case where $M_1=1$ and $\scale =1$. Since $M_1\geq M_2$, we also have $M_2=1$. In this case, there are precisely 2 jobs in the system, and the service time of each job is just 1. Under the event $B$ and instance $\mathcal{P}_2$, the algorithm serves the job of mean holding cost $(1-\epsilon)/2$ and then the job of mean holding cost $1/2$ next, but the optimal sequence is the opposite. Hence, we obtain
	$$\mathbb{E}[ R^\pi \mid B, \mathcal{P}_2]=\left(\frac{1-\epsilon}{2} + \frac{1}{2}\cdot 2\right)-\left(\frac{1}{2} + \frac{1-\epsilon}{2}\cdot 2\right)=\frac{\epsilon}{2}=\frac{1}{4}.$$
	Similarly, under the event $\neg B$ and instance $\mathcal{P}_1$, the algorithm serves the job of mean holding cost $1/2$ and then the job of mean holding cost $(1+\epsilon)/2$ next. Therefore,
	$$\mathbb{E}[ R^\pi \mid \neg B, \mathcal{P}_1]=\left(\frac{1}{2} + \frac{1+\epsilon}{2}\cdot 2\right)-\left(\frac{1+\epsilon}{2} + \frac{1}{2}\cdot 2\right)=\frac{\epsilon}{2}=\frac{1}{4}.$$
	Since $M_1=M_2=\scale =1$, we have $M_1^{1/2}M_2^{1/2} \scale^{1/2}=1$, in which case~\eqref{eq:second-lb-conc} holds, as required.
	
	Henceforth, we assume that $M_1\geq 2$ or $\scale\geq 2$, so $M_1 \scale\geq 2$. This means that
	\begin{equation}\label{eq:second-lb-T0}
		\frac{1}{3}M_1\scale \leq T_0 \leq \frac{1}{2} M_1\scale.
	\end{equation}
	As in Section~\ref{sec:regret-basic}, we assume that $c_1\geq c_2\geq \cdots\geq c_I$. Then we number the $N$ jobs from $1$ to $N$ so that jobs $1+\sum_{n\in[i-1]}N_n,\ldots, \sum_{n\in[i]}N_n$ belong to class $i$. Let $d_n$ denote the mean per-time holding cost of job $n\in[N]$. Then, if job $n$ is of class $i$, then we have $d_n=c_i$. Let $\sigma:[N]\to[N]$ be the permutation of $[N]$ that gives the sequence of jobs completed by the algorithm. 
	
	Consider the case where we are under the instance $\mathcal{P}_2$ and the event $B$. Let $n_1$ be the number such that $\sigma(n_1)$ is the $\lfloor(M_1+1)/2\rfloor$th job completed among the ones in $\I_1$. Then $n_1\geq \lfloor(M_1+1)/2\rfloor\geq 1$ and right before job $\sigma(n_1)$ finishes, there remain at least $\lfloor(M_1+1)/2\rfloor\geq M_1/2$ jobs from $\I_1$, including $\sigma(n_1)$. Note that $n_1 - \lfloor(M_1+1)/2\rfloor$ is the number of jobs from $\I_2$ that are completed before job $\sigma(n_1)$. 
	
	If $n_1 - \lfloor(M_1+1)/2\rfloor\geq M_2 / 8$, then at least $M_2/8$ jobs from $\I_2$ are served before job $\sigma(n_1)$. In this case, from Lemma~\ref{lemma:regret-1} and~\eqref{prelim-regret}, we have
	\begin{equation}\label{eq:lb-second-bound2}
		R^\pi\geq\sum_{n\in[n_1]}\sum_{\ell\in E_n}\left(d_{\sigma(\ell)}-d_{\sigma(n)}\right)\scale \geq \frac{M_2}{8}\cdot \frac{M_1}{2}\cdot \frac{\epsilon}{2}\scale = \frac{1}{64}\cdot M_1^{1/2}M_2^{1/2}\scale^{1/2}
	\end{equation}
	because for any $n$ such that $\sigma(n)$ is from $\I_2$, we have $|E_n|\geq M_1/2$ and $d_{\sigma(\ell)}-d_{\sigma(n)}=\epsilon/2$ for any $\ell\in E_n$. 
	
	If $n_1 - \lfloor(M_1+1)/2\rfloor< M_2 / 8$, less than $M_2 / 8$ jobs from $\I_2$ finish until the completion of job $\sigma(n_1)$. In this case, it requires less than $M_2\scale / 8$ time steps to complete the jobs from $\I_2$ that are sequenced before job $\sigma(n_1)$. However, we are under the event $B$, and therefore, at least $M_2\scale /4-M_2\scale / 8=M_2\scale/8$ time slots are used to serve jobs other than the ones before $\sigma(n_1)$. In particular, this implies that $W_{n_1}\geq M_2\scale/8$, so we obtain the following:
	\begin{equation}\label{eq:lb-second-bound3}
		R^\pi\geq\sum_{n\in[n_1]}d_{\sigma(n)}W_n\geq \frac{M_1}{2}\cdot\frac{1}{2}\cdot \frac{M_2\scale}{8}=\frac{1}{32}\cdot M_1 M_2 \scale.
	\end{equation}
	Based on~\eqref{eq:lb-second-bound2} and~\eqref{eq:lb-second-bound3}, we obtain
	\begin{equation}\label{eq:lb-second-bound4}
		\mathbb{E}[ R^\pi \mid B, \mathcal{P}_2]\geq \frac{1}{64}\cdot M_1^{1/2}M_2^{1/2}\scale^{1/2}.
	\end{equation}
	
	Next, assume that we are under the instance $\mathcal{P}_1$ and the event $\neg B$. Let $n_2$ be the number such that $\sigma(n_2)$ is the $\lfloor(M_2+1)/2\rfloor$th job completed among the ones in $\I_2$. Then $n_2\geq \lfloor(M_2+1)/2\rfloor\geq 1$ and right before job $\sigma(n_2)$ finishes, at least $\lfloor(M_2+1)/2\rfloor\geq M_2/2$ jobs from $\I_2$ are in the system, including job $\sigma(n_2)$. Moreover, $n_2 - \lfloor(M_2+1)/2\rfloor$ is the number of jobs from $\I_1$ that are completed before job $\sigma(n_2)$. 
	
	If $n_2 - \lfloor(M_2+1)/2\rfloor\geq M_1/24$, then at least $M_1/24$ jobs from $\I_1$ are processed before job $\sigma(n_2)$. Then it follows from Lemma~\ref{lemma:regret-1} and~\eqref{prelim-regret} that
	\begin{equation}\label{eq:lb-second-bound5}
		R^\pi\geq\sum_{n\in[n_2]}\sum_{\ell\in E_n}\left(d_{\sigma(\ell)}-d_{\sigma(n)}\right)\scale \geq \frac{M_1}{24}\cdot \frac{M_2}{2}\cdot \frac{\epsilon}{2}\scale = \frac{1}{192}\cdot M_1^{1/2}M_2^{1/2}\scale^{1/2}
	\end{equation}
	since for any $n$ such that $\sigma(n)$ is from $\I_1$, we have $|E_n|\geq M_2 /2$ and $d_{\sigma(\ell)} - d_{\sigma(n)}=\epsilon/2$ for any $\ell \in E_n$.
	
	If $n_2 - \lfloor(M_2+1)/2\rfloor< M_1 / 24$, then the number of jobs from $\I_1$ that are completed before the completion of job $\sigma(n_2)$ is less than $M_1 /24$. Then less than $M_1\scale / 24$ time slots are used to complete the jobs from $\I_1$ that are sequenced before job $\sigma(n_2)$. However, we are under the event $\neg B$, so at least $T_0-M_2\scale /4$ time slots are allocated for serving jobs from $\I_1$. Here, we know that
	$$T_0-\frac{M_2\scale}{4}\geq \frac{M_1\scale}{3} - \frac{M_2\scale}{4}\geq \frac{M_1\scale}{12}$$
	where the first inequality is from~\eqref{eq:second-lb-T0}.
	This in turn implies that at least $M_1\scale /12-M_1\scale / 24=M_1\scale/24$ time slots are used to serve jobs other than the ones before $\sigma(n_2)$. Therefore, we obtain $W_{n_2}\geq M_1\scale/24$, so the following holds:
	\begin{equation}\label{eq:lb-second-bound6}
		R^\pi\geq\sum_{n\in[n_2]}d_{\sigma(n)}W_n\geq \frac{M_2}{2}\cdot\frac{1}{2}\cdot \frac{M_1\scale}{24}=\frac{1}{96}\cdot M_1 M_2 \scale.
	\end{equation}
	Based on~\eqref{eq:lb-second-bound5} and~\eqref{eq:lb-second-bound6}, we get
	\begin{equation}\label{eq:lb-second-bound7}
		\mathbb{E}[ R^\pi \mid \neg B, \mathcal{P}_1]\geq \frac{1}{192}\cdot M_1^{1/2}M_2^{1/2}\scale^{1/2}.
	\end{equation}
	Combining~\eqref{eq:second-lb-cases}, ~\eqref{eq:lb-second-bound4}, and~\eqref{eq:lb-second-bound7}, it follows that $\mathbb{E}[ R^\pi]=\Omega\left(M_1^{1/2}M_2^{1/2}\scale^{1/2}\right)$, as required.
\end{proof}

\subsection{Completing the proof by characterizing an optimal partition}\label{sec:uniform-lb-final}

As in Sections~\ref{sec:uniform-lb-first} and~\ref{sec:uniform-lb-second}, we use notations $M_1= \sum_{i\in \I_1}N_i$ and $M_2= \sum_{i\in \I_2}N_i$
for a partition $(\I_1,\I_2)$ of $\I$. We have proved that
$\Omega\left(M_2^{2/3}\scale^{2/3}\right)$
is a lower bound on the expected regret of any (randomized) scheduling algorithm, under the condition that $M_2^{-1/3}\scale^{2/3}\geq 1$. We have also shown that $\Omega\left(M_1^{1/2}M_2^{1/2}\scale^{1/2}\right)$
is a lower bound on the expected regret. Hence, the second lower bound holds true in any case. In fact, the second lower bound is stronger than the first one if $M_2^{-1/3}\scale^{2/3}<1$ as we can check from
$$M_1^{1/2}M_2^{1/2}\scale^{1/2} \geq M_1^{1/3}M_2^{2/3}\scale^{1/2}\geq M_2^{2/3}\scale^{7/6}$$ 
where the first inequality is because $M_1\geq M_2$ and the second inequality follows from $M_1\geq M_2> \scale^2$. Therefore, both $\Omega\left(M_2^{2/3}\scale^{2/3}\right)$ and $\Omega\left(M_1^{1/2}M_2^{1/2}\scale^{1/2}\right)$ are correct lower bounds on the expected regret.

To finish the proof of Theorem~\ref{thm:uniform-lb}, we find a partition $(\I_1,\I_2)$ maximizing $M_2=\sum_{i\in \I_2}N_i$ and a partition $(\I_1,\I_2)$ (not necessarily the same as the first one) maximizing $M_1M_2 = \left(\sum_{i\in \I_1}N_i\right)\left(\sum_{i\in \I_2}N_i\right)$. In fact, both 
$\sum_{i\in \I_2}N_i$ and $\left(\sum_{i\in \I_1}N_i\right)\left(\sum_{i\in \I_2}N_i\right)$ are maximized by a partition $(\I_1,\I_2)$ such that the gap between $\sum_{i\in \I_1}N_i$ and $\sum_{i\in \I_2}N_i$ is minimized.

\begin{lemma}\label{lemma:lb-best-partition}
	Let $(\I_1,\I_2)$ be a partition of $\I$ minimizing
	$$\left|\sum_{i\in \I_1}N_i - \sum_{i\in \I_2}N_i \right|$$
	over the partitions of $\I$. Assume that $\sum_{i\in \I_1}N_i\geq \sum_{i\in \I_2}N_i$.
	\begin{enumerate}[$(a)$]
		\item\label{lb-best-partition-1} If $\barn \leq N/3$, then 
		$$\I_1=\left\{i_{\max}\right\}~~\text{and}~~\I_2=\I\setminus\left\{i_{\max}\right\}.$$
		\item\label{lb-best-partition-2} If $\barn> N/3$, then 
		$$\sum_{i\in \I_1}N_i\leq \frac{2N}{3}~~\text{and}~~\sum_{i\in \I_2}N_i \geq\frac{N}{3}.$$
	\end{enumerate}
\end{lemma}
\begin{proof}
	\eqref{lb-best-partition-1} If $\barn\geq N/3$, then $N_{i_{\max}}\geq 2N/3$. Since we assumed that $\sum_{i\in \I_1}N_i\geq \sum_{i\in \I_2}N_i$, it follows that $i_{\max}$ belongs to $\I_1$. If $L$ is a strict superset of $\left\{i_{\max}\right\}$, then we have $\sum_{i\in L}N_i - \sum_{i\in \I\setminus L}N_i > N_{i_{\max}} - \barn$.
	As $\{\I_1,\I_2\}$ minimizes $\left|\sum_{i\in \I_1}N_i - \sum_{i\in \I_2}N_i \right|$, it follows that $\I_1=\left\{i_{\max}\right\}$ and $\I_2=\I\setminus\left\{i_{\max}\right\}$.
	
	\eqref{lb-best-partition-2} It is sufficient to find a partition $(P,Q)$ of $\I$ such that
	$N/3\leq \sum_{i\in P}N_i, \sum_{i\in Q}N_i\leq 2N/3$.
	If $\barn> N/3$, then $N_{i_{\max}}< 2N/3$. If $N_{i_{\max}}\geq \barn$, then $\I_1=\left\{i_{\max}\right\}$ and $\I_2=\I\setminus \left\{i_{\max}\right\}$ form a desired partition. Otherwise, we obtain $N_{i_{\max}}<N/2<\barn$. Here, if $\barn\leq 2N/3$, then $\I_1=\I\setminus \left\{i_{\max}\right\}$ and $\I_2=\left\{i_{\max}\right\}$ give us a desired partition. Thus we may assume that $N_{i_{\max}}<N/3$ and $2N/3<\barn$. 
	
	Let $i_1,\ldots,i_{I-1}$ be an arbitrary sequence of the classes in $\I\setminus \{i_{\max}\}$. We consider $\I\setminus\{i_{\max},i_1\},\I\setminus\{i_{\max},i_1,i_2\},\ldots$ sequentially. Let $\ell$ be the first index such that
	$$\sum_{i\in \I\setminus\{i_{\max},i_1,\ldots,i_\ell\}}N_i\leq \frac{2N}{3}.$$
	If $\ell=1$, then as $N_{i_1}\leq N_{i_{\max}}\leq N/3$, we have
	$\sum_{i\in \{i_{\max},i_1\}}N_i\leq 2N/3$. Moreover, since $\sum_{i\in \I\setminus\{i_{\max},i_1\}}N_i\leq 2N/3$, we have
	$N/3\leq\sum_{i\in \{i_{\max},i_1\}}N_i,\ \sum_{i\in \I\setminus\{i_{\max},i_1\}}N_i\leq 2N/3$.
	Therefore, $\{i_{\max},i_1\}\}$ and $\I\setminus\{i_{\max},i_1\}$ give rise to a required partition. If $\ell\geq 2$, then
	$$\sum_{i\in \I\setminus\{i_{\max},i_1,\ldots,i_{\ell-1}\}}N_i> \frac{2N}{3}~~\text{and}~~ \sum_{i\in \{i_{\max},i_1,\ldots,i_{\ell-1}\}}N_i< \frac{N}{3}.$$
	Since $N_{i_\ell}\leq N_{i_{\max}}\leq N/3$, it follows that
	$$\frac{N}{3}\leq\sum_{i\in \{i_{\max},i_1,\ldots,i_{\ell}\}}N_i,\ \sum_{i\in \I\setminus\{i_{\max},i_1,\ldots,i_{\ell}\}}N_i\leq \frac{2N}{3},$$
	as required.
\end{proof}

Lemma~\ref{lemma:lb-best-partition} implies that there always exists a partition $(\I_1,\I_2)$ such that $\sum_{i\in \I_2}N_i=\Omega(\barn)$. Moreover, it also implies that there is a partition $(\I_1,\I_2)$ such that $\left(\sum_{i\in \I_1}N_i\right)\left(\sum_{i\in \I_2}N_i\right)=\Omega(N\barn)$. As a result, it follows that 
$$\Omega\left(\max\left\{\barn^{2/3}\scale^{2/3},\ N^{1/2}\barn^{1/2}\scale^{1/2}\right\}\right)$$
is a correct lower bound on the expected regret, as required.

\section{Lower bounds on the expected regret of Algorithm~\ref{preempt-then-nonpreemptive}}

\subsection{Proof of Proposition~\ref{prop:ex1}}\label{sec:prop1}

We first state a result that is concerned with lower bounding the tail probability of a binomial random variable. The result is based on Slud's inequality~\cite{slud}, which says that for i.i.d. Bernoulli random variables $X_1,\ldots, X_m$ with $\mathbb{E}[X_1]=p\leq 1/4$, 
\begin{equation}\label{eq:slud}
	\mathbb{P}\left[\sum_{i=1}^m X_i\geq k\right]\geq \mathbb{P}\left[Z \geq \frac{k-mp}{mp(1-p)}\right]
\end{equation}
where $Z$ is a standard normal random variable. 
\begin{lemma}\label{lemma:reverse-chernoff}
	Let $X_1,\ldots, X_m$ be i.i.d. Bernoulli random variables with $\mathbb{E}[X_1]=p\leq 1/4$. Then for any $\epsilon>0$,
	$$
	\mathbb{P}\left[\frac{1}{m}\sum_{i=1}^m X_i> p +\epsilon\right]\geq \frac{1}{4}\exp\left(-\frac{2m\epsilon^2}{p}\right).\footnote{Its proof is given in~\url{https://ece.uwaterloo.ca/~nmousavi/Papers/Chernoff-Tightness.pdf}.}
	$$
\end{lemma}

Using Lemma~\ref{lemma:reverse-chernoff}, we can prove that the following lemma, which gives a lower bound on the probability that Algorithm~\ref{preempt-then-nonpreemptive} selects the job of class 2.

\begin{lemma}\label{lemma:ex1-error}
	Consider the instance $\mathcal{P}_1$ in Example 1 with $N\geq 8$ and $\sqrt{4\ln 2/N\scale}\leq \epsilon\leq 9/10$. For any $\ln 2/2\epsilon^2\leq t\leq \ln 2/\epsilon^2$, if the job of class $2$ still waits to be served at time $t$, then 
	$\mathbb{P}\left[\hat c_{1,t}<\hat c_{2,t}\right]\geq 1/2^{13}$.
\end{lemma}
\begin{proof}
	Note first that
	\begin{equation}\label{ex:1}
		\mathbb{P}\left[\hat c_{1,t}<\hat c_{2,t}\right]\geq \mathbb{P}\left[\hat c_{1,t}\leq \frac{1+\epsilon}{8}<\hat c_{2,t}\right]=\mathbb{P}\left[\hat c_{1,t}\leq\frac{1+\epsilon}{8}\right]\cdot \mathbb{P}\left[\hat c_{2,t}> \frac{1+\epsilon}{8}\right]
	\end{equation}
	where the equality is because $\hat c_{1,t}$ and $\hat c_{2,t}$ are independent. By Hoeffding's inequality, we have
	\begin{equation}\label{ex:2}
		\mathbb{P}\left[\hat c_{1,t}\leq \frac{1+\epsilon}{8}\right]=1-\mathbb{P}\left[\hat c_{1,t}>\frac{1+\epsilon}{8}\right]\geq 1 - \exp\left(-2\epsilon^2\sum_{s=1}^t N_{1,s}\right).
	\end{equation}
	As $\epsilon \geq \sqrt{4\ln 2/N\scale}$, we have $t\leq N\scale /4$, implying in turn that there are at least $N_1- N/4$ jobs of class 1 remain in the system at time $t$. Due to our assumption that $N\geq 8$, we know that $N_1-N/4=3N/4-1\geq N/2$, and therefore, $\sum_{s=1}^t N_{1,s}\geq Nt/2$. Then it follows from~\eqref{ex:2} that
	\begin{equation}\label{ex:3}
		\mathbb{P}\left[\hat c_{1,t}\leq \frac{1+\epsilon}{8}\right]\geq 1- \exp\left(-\epsilon^2Nt\right)\geq 1-2^{-N/2}\geq \frac{1}{2}
	\end{equation}
	where the second inequality comes from $t\geq \ln 2/2\epsilon^2$ and the last inequality is due to the assumption that $N\geq 8$. Moreover, by Lemma~\ref{lemma:reverse-chernoff}, we obtain 
	\begin{equation}\label{ex:4}
		\mathbb{P}\left[\hat c_{2,t}> \frac{1+\epsilon}{8}\right]\geq \frac{1}{4}\exp\left(-\frac{\epsilon^2}{1-\epsilon}t\right)\geq \frac{1}{4}\exp\left(-10\epsilon^2t\right)\geq \frac{1}{4}2^{-10}=\frac{1}{2^{12}}
	\end{equation}
	where the second inequality is because $1/(1-\epsilon)\leq 4$ and the third inequality comes from $t\leq \ln 2/\epsilon^2$. Combining~\eqref{ex:1},~\eqref{ex:3}, and~\eqref{ex:4}, we obtain $\mathbb{P}\left[\hat c_{1,t}<\hat c_{2,t}\right]\geq 1/2^{13}$, as required.
\end{proof}

Recall that $\Ts$ is the length of the preemption phase in Algorithm~\ref{preempt-then-nonpreemptive} and its value is set according to~\eqref{eq:Ts-choice}. Since $N_{\min}=1$, we have $N_{\min}^{-1}\scale^{2/3}(\log N\scale)^{1/3}\geq 1.$ Hence, $$\Ts=\begin{cases}
	\lfloor \scale^{2/3}(\log N\scale)^{1/3}\rfloor,&\text{if $\scale>\log N\scale$}\\
	\scale-1,&\text{if $\scale \leq \log N\scale$}.\end{cases}$$
Here, if $\scale\leq 3$, then we have $\Ts=\scale -1$ because $\log N\scale \geq \log 16=4\geq 3=\scale$. On the other hand, if $\scale\geq 4$, then $\scale-1\geq \scale^{2/3}$, and therefore, $\Ts\geq \lfloor \scale^{2/3}\rfloor$. Based on this observation,  we separately consider the case $\scale \leq 3$ and the case $\scale\geq 4$.

\paragraph{The case when $\scale\leq 3$.} Let $\epsilon=\sqrt{\ln 2/\scale}$. Since $N\geq 8$, we have $\sqrt{4\ln 2/N\scale}\leq \epsilon$. Moreover, as $\sqrt{\ln 2}\leq 9/10$, we also have $\epsilon\leq 9/10$. Then, by Lemma~\ref{lemma:ex1-error}, we have 
$$\mathbb{P}\left[\hat c_{1,\scale}<\hat c_{2,\scale}\right]\geq 1/2^{13}.$$
We observed that $\Ts=\scale -1$ when $\scale\leq 3$, which means that Algorithm~\ref{preempt-then-nonpreemptive} chooses the first job for non-preemptive serving at the beginning of the $\scale$th time slot. If the job of class $2$ is chosen first, then the expected cumulative holding cost $C^\pi$ is at least
$$C^\pi\geq c_2\scale + c_1\sum_{n=2}^N n\scale$$
because the $n$th job completed by the algorithm stays in the sytem for at least $n\scale$ time steps for every $n\in[N]$.
Since the optimal expected cumulative holding cost is $c_2N\scale + c_1\sum_{n=1}^{N-1} n\scale$, it follows that
$$R^\pi\geq(c_1-c_2)(N-1)\scale = \frac{N-1}{8}\epsilon\scale=\frac{\epsilon}{8}\cdot(N-1)\scale\geq \frac{\sqrt{\ln 2}}{16}N\sqrt{\scale}\geq \frac{\sqrt{\ln 2}}{32}N\sqrt\scale^{2/3}$$
where the first inequality is because $N-1\geq N/2$ and the second inequality follows from $\scale\leq 3$. Then
\begin{equation}\label{eq:prop1-case1}
	\mathbb{E}\left[R^\pi\right] \geq \mathbb{P}\left[\hat c_{1,1}<\hat c_{2,1}\right]\cdot \mathbb{E}\left[R^\pi\mid \hat c_{1,1}<\hat c_{2,1}\right]\geq \frac{\sqrt{\ln2}}{2^{18}} N\scale^{2/3}.
\end{equation}

\paragraph{The case when $\scale\geq 4$.}

We consider
$$T_0=\lfloor \scale^{2/3}\rfloor\quad\text{and}\quad\epsilon = \scale^{-1/3}\sqrt{\ln 2}.$$
Since $(\ln 2)^{1/2}\leq 9/10$ and $N\geq 8$, it is clear that $\sqrt{4\ln 2/N\scale}\leq \epsilon\leq 9/10$. Thus, by Lemma~\ref{lemma:ex1-error}, $\mathbb{P}\left[\hat c_{1,t}<\hat c_{2,t}\right]\geq 1/2^{13}$ holds for any $\scale^{2/3}/2\leq t\leq \scale^{2/3}$. Note that $\scale^{2/3}/2\leq (T_0+1)/2$ and $T_0\leq \scale^{2/3}$. Therefore, it follows that $\mathbb{P}\left[\hat c_{1,t}<\hat c_{2,t}\right]\geq 1/2^{13}$ holds for any $(T_0+1)/2\leq t\leq T_0$.

For $\scale \geq 4$, we know that $T_0\leq \scale -1$. Moreover, $T_0\leq \lfloor \scale^{2/3}(\log N\scale)^{1/3}\rfloor$. Hence, it follows that $T_0\leq \Ts$, implying in turn that no job finishes until the end of the $T_0$th time slot. Let $Q(T_0)$ count the number of time slots in the first $T_0$ time steps where Algorithm~\ref{preempt-then-nonpreemptive} processes the job of class 2. Now we will show that
\begin{equation}\label{eq:prop1-Ts}
	\mathbb{E}\left[ Q(T_0)\right]\geq\sum_{t=1}^{T_0}\mathbb{P}\left[\hat c_{1,t}<\hat c_{2,t}\right]\geq \frac{1}{2^{14}}T_0.
\end{equation}
Recall that Algorithm~\ref{preempt-then-nonpreemptive} selects the job of class 2 at time $t$ if $\hat c_{1,t}<\hat c_{2,t}$, and therefore, the first inequality holds. Since $\mathbb{P}\left[\hat c_{1,t}<\hat c_{2,t}\right]\geq 1/2^{13}$ holds for any $(T_0+1)/2\leq t\leq T_0$, 
$$\mathbb{E}\left[ Q(T_0)\right]\geq \sum_{t=\lceil(1+T_0)/2\rceil}^{T_0}\mathbb{P}\left[\hat c_{1,t}<\hat c_{2,t}\right]\geq \frac{1}{2^{13}}\cdot \lfloor\frac{1+T_0}{2}\rfloor\geq \frac{1}{2^{14}}T_0.$$
Hence, we have just proved that~\eqref{eq:prop1-Ts} holds.

We number the $N$ jobs from $1$ to $N$ so that jobs $1,\ldots,N-1$ are the ones in class $1$ and job $N$ is the class $2$ job. For $n\in [N]$, let $T_n$ denote the number of time steps where job $n$ is processed by Algorithm~\ref{preempt-then-nonpreemptive} during the preemption phase. Then $0\leq T_n\leq \Ts$. Let $\sigma:[N]\to[N]$ be the permutation of $[N]$ that gives the sequence of jobs completed by the algorithm. As in the proof of Theorem~\ref{thm:first-lb}, we can argue that $W_n\geq \Ts-\sum_{\ell=1}^nT_{\sigma(\ell)}$. Then, by Lemma~\ref{lemma:regret-1},
\begin{equation}\label{eq:prop1-1}
	R^\pi\geq\sum_{n\in[N]}d_{\sigma(n)}\left(\Ts -\sum_{\ell=1}^{n}T_{\sigma(\ell)} \right)+  \sum_{n\in[N]}\sum_{\ell\in E_n}\left(d_{\sigma(\ell)}-d_{\sigma(n)}\right)\scale.
\end{equation}

Let $n_2$ be the number such that $\sigma(n_2)=n$, i.e., job $\sigma(n_2)$ is the one in class 2. Note that
$$\Ts-\sum_{\ell=1}^{n_2-1}T_{\sigma(\ell)}\geq T_{\sigma(n_2)} \geq Q(T_0)$$
where the last inequality is because $T_0\leq \Ts$.
Then, by~\eqref{eq:prop1-1}, we obtain
\begin{equation}\label{eq:prop1-2}
	R^\pi\geq (n_2-1)\cdot\frac{1}{8}\cdot Q(T_0)+(N-n_2)\cdot\frac{\epsilon}{8}\cdot \scale.
\end{equation}
If $n_2\geq N/2+1$, then we have \begin{equation}R^\pi\geq \frac{1}{16}NQ(T_0).\end{equation} If $n_2< N/2+1$, then $n_2\leq (N+1)/2$. Then it follows from~\eqref{eq:prop1-1} that 
\begin{equation}\label{eq:prop1-3}
	R^\pi\geq \frac{N-1}{16}\scale(\scale^{-1/3}\sqrt{\ln 2})\geq \frac{\sqrt{\ln2}}{32}N\scale^{2/3}\end{equation}
where the last inequality holds because $N-1\geq N/2$ and $\Ts\geq 1$. Combining~\eqref{eq:prop1-2} and~\eqref{eq:prop1-3}, we obtain
\begin{align}\label{eq:prop1-case2}
	\begin{aligned}
		\mathbb{E}\left[R^\pi\right]&=\mathbb{P}\left[n_2\geq N/2+1\right]\cdot\mathbb{E}\left[R^{\pi}\mid n_2\geq N/2+1\right]+\mathbb{P}\left[n_2< N/2+1\right]\cdot\mathbb{E}\left[R^{\pi}\mid n_2< N/2+1\right]\\
		&\geq \mathbb{P}\left[n_2\geq N/2+1\right]\cdot \frac{1}{16}N\mathbb{E}\left[Q(T_0)\right] + \mathbb{P}\left[n_2< N/2+1\right]\cdot\frac{\sqrt{\ln2}}{32}N\scale^{2/3}\\
		&\geq \frac{1}{2^{18}} NT_0 \\
		&\geq \frac{1}{2^{19}} N\scale^{2/3}
	\end{aligned}
\end{align}
where the second inequality comes from~\eqref{eq:prop1-Ts} and the last inequality holds because $T_0\geq \scale^{2/3}/2$.

Therefore, it follows from~\eqref{eq:prop1-case1} and~\eqref{eq:prop1-case1} that
$$\mathbb{E}\left[R^\pi\right]\geq  \frac{1}{2^{19}} N\scale^{2/3},$$
as required.

\subsection{Proof of Proposition~\ref{prop:ex2}}\label{sec:prop2}

Recall that $N_1=N-\lceil N^{\delta}\rceil$. Then, since $N\geq 8^{1/(1-\delta)}$, we get \begin{equation}\label{eq:prop2-N1}
	N_1- \frac{3N}{4}\geq (N- N^{\delta}-1)-\frac{3N}{4}=N^{\delta}\left(\frac{N^{1-\delta}}{4}-1\right)-1\geq N^{\delta}-1\geq 0
\end{equation}
where the second last inequality is due to $N^{1-\delta}\geq 8$. Therefore, $N_1-N/4\geq N/2$.

At each time $t$, we define $\R_t$ as follows: $$\R_t=\left\{i\in\{2,\ldots, I\}:\ \text{the job of class $i$ still waits to be served at time $t$}\right\}.$$
Hence, if $\R_t$ is not empty and Algorithm~\ref{preempt-then-nonpreemptive} decides which job to serve at time $t$, then Algorithm~\ref{preempt-then-nonpreemptive} compares $\hat c_{1,t}$ and $\max_{i\in \R_t} \hat c_{i,t}$. The following lemma provides a lower bound on the probability that class 1 is not selected at time $t$.
\begin{lemma}\label{lemma:ex2-error}
	Consider the instance $\mathcal{P}_2$ in Example 2 with $N\geq 8^{1/(1-\delta)}$ and $\sqrt{4\ln 2/N\scale}\leq \epsilon\leq 9/10$. 
	Then for any $\ln 2/2\epsilon^2\leq t\leq 2\ln 2/\epsilon^2$, if $\R_t$ is not empty,
	we have $\mathbb{P}\left[\hat c_{1,t}<\hat c_{i,t}\text{ for some $i\in \R_t$}\right]\geq 1/2^{23}$.
\end{lemma}
\begin{proof}
	Note that
	\begin{align}\label{ex2:1}
		\begin{aligned}
			\mathbb{P}\left[\hat c_{1,t}<\hat c_{i,t}\text{ for some $i\in \R_t$}\right]&\geq \mathbb{P}\left[\hat c_{1,t}\leq \frac{1+\epsilon}{8}<\hat c_{i,t}\text{ for some $i\in \R_t$}\right]\\
			&=\mathbb{P}\left[\hat c_{1,t}\leq\frac{1+\epsilon}{8}\right]\cdot \mathbb{P}\left[\hat c_{i,t}> \frac{1+\epsilon}{8} \text{ for some $i\in \R_t$}\right]\\
		\end{aligned}
	\end{align}
	where the equalities hold because $\hat c_{1,t},\ldots, \hat c_{I,t}$ are independent. 
	
	For any $i\in \R_t$, by Lemma~\ref{lemma:reverse-chernoff} and our assumption that $\epsilon\leq 9/10$ and $t\leq 2\ln 2/\epsilon^2$, we obtain
	$$\mathbb{P}\left[\hat c_{i,t}> \frac{1+\epsilon}{8}\right]\geq \frac{1}{4}\exp\left(-\frac{\epsilon^2}{1-\epsilon}t\right)\geq \frac{1}{2^{22}}.
	$$
	Then, if $\R_t$ is not empty, there exists some $j\in \R_t$, and therefore, 
	\begin{equation}\label{ex2:2}
		\mathbb{P}\left[\hat c_{i,t}> \frac{1+\epsilon}{8} \text{ for some $i\in \R_t$}\right]\geq \mathbb{P}\left[\hat c_{j,t}> \frac{1+\epsilon}{8}\right]\geq \frac{1}{2^{22}}.
	\end{equation}
	
	As in the proof of Lemma~\ref{lemma:ex1-error},
	it follows from $\epsilon \geq \sqrt{4\ln 2/N\scale}$ that there are at least $N_1- N/4$ jobs of class 1 remain in the system at time $t$. By~\eqref{eq:prop2-N1}, we have $N_1- N/4\geq N/2$, so  $\sum_{s=1}^t N_{1,s}\geq Nt/2$. By Hoeffding's inequality with $t\geq \ln 2/2\epsilon^2$, we obtain
	\begin{equation}\label{ex2:3}
		\mathbb{P}\left[\hat c_{1,t}\leq \frac{1+\epsilon}{8}\right]\geq 1- \exp\left(-\epsilon^2Nt\right)\geq 1-2^{-N/2}\geq \frac{1}{2}
	\end{equation}
	Therefore, by~\eqref{ex2:1},~\eqref{ex2:2}, and~\eqref{ex2:3}, we get $\mathbb{P}\left[\hat c_{1,t}<\hat c_{i,t}\text{ for some $i\in \R_t$}\right]\geq 1/2^{23}$, as required.
\end{proof}

To prove Proposition~\ref{prop:ex2}, we consider 
$$T_0=\lceil N^{\delta}\rceil\scale\quad\text{and}\quad \epsilon=N^{-\delta/2}\scale^{-1/2}\sqrt{\ln 2}.$$
By~\eqref{eq:prop2-N1}, we know that $N_1\geq 3N/4$. Then
\begin{equation}\label{eq:prop2-N1-remain}
	N_1-\lceil N^{\delta}\rceil \geq \frac{3N}{4} - N^{\delta}-1=\frac{N}{2}+ N^{\delta}\left(\frac{N^{1-\delta}}{4} -1\right)-1\geq \frac{N}{2}+N^{\delta}-1\geq \frac{N}{2}
\end{equation}
where the first inequality is because $N_1\geq 3N/4$ and the second inequality follows from $N\geq 8^{1/(1-\delta)}$. This implies that until the end of the $T_0$th time slot, there are at least $N/2$ jobs of class $1$.

Let $P(T_0)$ be the number of jobs from classes in $\{2,\ldots, I\}$ that are chosen for non-preemptive serving by Algorithm~\ref{preempt-then-nonpreemptive} until the end of the $T_0$th time slot. Let $k$ be the total number of jobs chosen for non-preemptive serving by Algorithm~\ref{preempt-then-nonpreemptive} until the end of the $T_0$th time slot. Note that $k\leq \lceil N^{\delta}\rceil$ because  finishing $\lceil N^{\delta}\rceil$ jobs requires $T_0$ units of service. In fact, $k\geq \lceil N^{\delta}\rceil$ since Algorithm~\ref{preempt-then-nonpreemptive} completes the first $\lceil N^{\delta}\rceil-1$ jobs by the end of the $\left(\Ts+(\lceil N^{\delta}\rceil-1)\scale\right)$th time slot and $\Ts+(\lceil N^{\delta}\rceil-1)\scale\leq \scale -1+(\lceil N^{\delta}\rceil-1)\scale= T_0-1$. Therefore, $k$ is precisely $\lceil N^{\delta}\rceil$.

Let $t_1,\ldots, t_{\lceil N^{\delta}\rceil}$ denote the moments when Algorithm~\ref{preempt-then-nonpreemptive} chooses a job for non-preemptive serving. We denote by $E_t$ the event that Algorithm~\ref{preempt-then-nonpreemptive} chooses a job from some class in $\{2,\ldots, I\}$, and we define $\bm{1}(E_t)$ as the indicator random variable for event $E_t$, i.e., $\bm{1}(E_t)$ takes value 1 when $E_t$ holds and value 0 when $E_t$ does not happen. Then
\begin{equation}\label{eq:prop2-1}
	\mathbb{E}\left[\bm{1}(E_t)\right]=\mathbb{P}\left[E_t\right]\geq \mathbb{P}\left[\hat c_{1,t}<\hat c_{i,t}\text{ for some $i\in \R_t$}\right]
\end{equation}
because Algorithm~\ref{preempt-then-nonpreemptive} must choose a job from some class in $\R_t$ if $\hat c_{1,t}<\hat c_{i,t}$ for some $i\in \R_t$. Moreover, by Lemma~\ref{lemma:ex2-error}, we have $\mathbb{P}\left[\hat c_{1,t}<\hat c_{i,t}\text{ for some $i\in \R_t$}\right]\geq 1/2^{23}$ for any $N^{\delta}\scale/2\leq t\leq 2N^{\delta}\scale$. In particular,
\begin{equation}\label{eq:prop2-2}
	\mathbb{P}\left[\hat c_{1,t}<\hat c_{i,t}\text{ for some $i\in \R_t$}\right]\geq \frac{1}{2^{23}}\quad\text{for any $\lceil\frac{ N^{\delta}}{2}\rceil\scale\leq t\leq \lceil{N^{\delta}}\rceil$}\scale.
\end{equation}
Recall that $1\leq t_1,\ldots, t_{\lceil N^{\delta}\rceil}\leq \lceil N^{\delta}\rceil\scale$. With the same argument by which we proved this, we can also argue that $1\leq t_1,\ldots, t_{\lceil N^{\delta}/2\rceil}\leq \lceil N^{\delta}/2\rceil\scale$. This means that
\begin{equation}\label{eq:prop2-3}
	\lceil\frac{N^{\delta}}{2}\rceil\scale \leq t_{\lceil N^{\delta}/2\rceil+1},\ldots, t_{\lceil N^{\delta}\rceil}\leq \lceil N^{\delta}\rceil\scale.
\end{equation}
Based on~\eqref{eq:prop2-1},~\eqref{eq:prop2-2}, and~\eqref{eq:prop2-3}, we can argue that the following holds.
\begin{align}\label{eq:prop2-4}
	\begin{aligned}
		\mathbb{E}\left[ P(T_0)\right]&=\mathbb{E}\left[\mathbb{E}\left[P(T_0)\mid E_{t_1},\ldots, E_{t_{\lceil N^{\delta}\rceil}}\right] \right]\\
		&= \mathbb{E}\left[\mathbb{E}\left[\sum_{\ell=1}^{\lceil N^\delta\rceil}\mathbb{E}\left[\bm{1}(E_{t_\ell})\right]\mid E_{t_1},\ldots, E_{t_{\lceil N^{\delta}\rceil}}\right] \right]\\
		&\geq \mathbb{E}\left[\mathbb{E}\left[\sum_{\ell=1}^{\lceil N^\delta\rceil}\mathbb{P}\left[\hat c_{1,t_\ell}<\hat c_{i,t_\ell}\text{ for some $i\in \R_t$}\right]\mid E_{t_1},\ldots, E_{t_{\lceil N^{\delta}\rceil}}\right] \right]\\
		&\geq \mathbb{E}\left[\mathbb{E}\left[\sum_{\ell=\lceil N^{\delta}/2\rceil+1}^{\lceil N^\delta\rceil}\mathbb{P}\left[\hat c_{1,t_\ell}<\hat c_{i,t_\ell}\text{ for some $i\in \R_t$}\right]\mid E_{t_1},\ldots, E_{t_{\lceil N^{\delta}\rceil}}\right] \right]\\
		&\geq \mathbb{E}\left[\mathbb{E}\left[\frac{1}{2^{23}}\left(\lceil N^\delta\rceil - \lceil \frac{N^\delta}{2}\rceil\right)\mid E_{t_1},\ldots, E_{t_{\lceil N^{\delta}\rceil}}\right] \right]\\
		&=\frac{1}{2^{23}}\left(\lceil N^\delta\rceil - \lceil \frac{N^\delta}{2}\rceil\right)
	\end{aligned}
\end{align}
where the first inequality is due to~\eqref{eq:prop2-1} and the third inequality is obtained from~\eqref{eq:prop2-2} and~\eqref{eq:prop2-3}. Since $\lceil N^\delta\rceil - \lceil {N^\delta}/{2}\rceil\geq {N^\delta}/{4}$, it follows from~\eqref{eq:prop2-4} that
\begin{equation}\label{eq:prop2-5}
	\mathbb{E}\left[ P(T_0)\right]\geq \frac{N^{\delta}}{2^{25}}.
\end{equation}

We number the $N$ jobs from $1$ to $N$ so that jobs $1,\ldots,N-\lceil N^{\delta}\rceil$ are the ones in class $1$ and job $N-\lceil N^{\delta}+i-1\rceil$ is the job of class $i$ for $i=2,\ldots, I$. Let $\sigma:[N]\to[N]$ be the permutation of $[N]$ that gives the sequence of jobs completed by Algorithm~\ref{preempt-then-nonpreemptive}. Then
\begin{align}\label{eq:prop2-6}
	\begin{aligned}
		R^\pi&\geq \sum_{n\in[N]}\sum_{\ell\in E_n}\left(d_{\sigma(\ell)}-d_{\sigma(n)}\right)\scale\\
		&\geq \sum_{n\leq\lceil N^{\delta}\rceil:\sigma(n)\not\in\J_1}\sum_{\ell\in E_n}\left(d_{\sigma(\ell)}-d_{\sigma(n)}\right)\scale\\
		&\geq \sum_{n\leq\lceil N^{\delta}\rceil:\sigma(n)\not\in\J_1}\frac{N}{2}\cdot \frac{\epsilon}{8}\cdot\scale\\
		&=P(T_0)\cdot\frac{N}{2}\cdot \frac{\epsilon}{8}\cdot\scale
	\end{aligned}
\end{align}
where the first inequality is given by Lemma~\ref{lemma:regret-1}, the second inequality is because $\lceil N^\delta\rceil\leq N$, and the third inequality is because there are at least $N_1-\lceil N^\delta\rceil$ jobs of class 1 remaining until completing job $\sigma(\lceil N^\delta\rceil)$, $N_1-\lceil N^\delta\rceil\geq N/2$ by~\eqref{eq:prop2-N1-remain}, and $c_1-c_i=\epsilon/8$ for any $i\geq 2$. Therefore, we can obtain the following:
\begin{align}\label{eq:prop2-7}
	\mathbb{E}\left[ R^\pi\right] &\geq \mathbb{E}\left[ P(T_0)\right]\cdot\frac{N}{2}\cdot \frac{\epsilon}{8}\cdot\scale\\
	&\geq \frac{N^{\delta}}{2^{25}}\cdot \frac{N}{2}\cdot \frac{\sqrt{\ln 2}}{8N^{\delta/2}\scale^{1/2}}\cdot \scale\\
	&=\frac{\sqrt{\ln 2}}{2^{29}}N^{1+\delta/2}\scale^{1/2}
\end{align}
where the first inequality follows from~\eqref{eq:prop2-6} and the second inequality comes from~\eqref{eq:prop2-5}. Consequently, we have just proved that
$$\mathbb{E}\left[ R^\pi\right]=\Omega(N^{1+\delta/2}\scale^{1/2}),$$
as required.

\section{Proof of Theorem~\ref{thm:uniform-ub-2}}\label{sec:proof:uniform-ub-2}

The proof of Theorem~\ref{thm:uniform-ub-2} is similar to that of Theorem~\ref{thm:uniform-ub-1}. Section~\ref{sec:gaps-refined} is about estimating the $c\mu$ index when the distribution of jobs over classes is not uniform. In Section~\ref{sec:uniform-ub-lemma}, we prove Lemma~\ref{lemma:uniform-ub-2}, and in Section~\ref{sec:pluuing-in-tau-second}, we show that setting $\tau$ as in~\eqref{eq:Ts-choice} we obtain the regret upper bound~\eqref{regret:uniform-ub-2}.

\subsection{Gaps between $c\mu$ values under the clean event}\label{sec:gaps-refined}
Recall that while running Algorithm~\ref{preempt-then-nonpreemptive-refined}, $i_{\max}$ and $\R$ are defined as follows:
\begin{itemize}
	\item $i_{\max}$ is some class in $\arg\max_{i\in\I}N_i$.
	\item $\R$ is the set of remaining classes; the classes that have at least one job waiting to be served.
\end{itemize}

Recall also that $\RP$ is a subfamily of $\R$ that is updated in every time slot by the following command:
$$\RP\leftarrow \RP\cup \left\{i\in \R\setminus \RP:\  \text{LCB}_{i,t}> \text{UCB}_{j,t}\text{ for some }j\in \RP\cup\{i_{\max}\}\right\}.$$
Basically, we newly add a class $i$ to $\RP$ in time slot $t$ if $\text{LCB}_{i,t}> \text{UCB}_{i_{\max},t}$ or $\text{LCB}_{i,t}> \text{UCB}_{j,t}$ for some $j\in \RP$.

\begin{lemma}\label{lemma:gap}
	Under the clean event, the following statements hold.
	\begin{enumerate}[($a$)]
		\item\label{gaps1} If $i_{\max}\in \R$ and $\RP$ is empty, then for any $i\in \R$,
		$$c_i\mu_i -c_{i_{\max}}\mu_{i_{\max}} \leq 2\mu_i\sqrt{\frac{3}{\sum_{s=1}^tN_{i,s}}\log \frac{N\scale}{\mu_{\min}}}+ 2\mu_{i_{\max}}\sqrt{\frac{3}{\sum_{s=1}^tN_{{i_{\max}},s}}\log \frac{N\scale}{\mu_{\min}}}.$$
		
		\item\label{gaps2} If $i_{\max}\in \R$ and $\RP$ is nonempty, then for any $i\in \RP$, 
		$$c_{i_{\max}}\mu_{i_{\max}}-c_i\mu_i<0.$$
		\item\label{gaps3} If $i_{\max}\in \R$, $\RP$ is nonempty, and $i\in\arg\max_{i\in \RP}\hat c_{i,t}\mu_i$, then for any $j\in \R$,
		$$c_j\mu_j- c_i\mu_i \leq 2\mu_j\sqrt{\frac{3}{\sum_{s=1}^tN_{j,s}}\log \frac{N\scale}{\mu_{\min}}}+2\mu_i\sqrt{\frac{3}{\sum_{s=1}^tN_{i,s}}\log \frac{N\scale}{\mu_{\min}}}.$$
		\item\label{gaps4} If $\R$ is nonempty and $i\in\arg\max_{i\in \R}\hat c_{i,t}\mu_i$, then for any $j\in \R$,
		$$c_j\mu_j-c_i\mu_i \leq \mu_j\sqrt{\frac{3}{\sum_{s=1}^tN_{j,s}}\log \frac{N\scale}{\mu_{\min}}}+\mu_i\sqrt{\frac{3}{\sum_{s=1}^tN_{i,s}}\log \frac{N\scale}{\mu_{\min}}}.$$
	\end{enumerate}
\end{lemma}
\begin{proof}
	\eqref{gaps1} Since $\RP$ is empty, for any $i\in \R$, we have $\text{UCB}_{i_{\max},t}\geq \text{LCB}_{i,t}$. We also know from~\eqref{conf-interval} that under the clean event, $$c_{i_{\max}}\mu_{i_{\max}}\geq \text{UCB}_{i_{\max},t}-2\mu_{i_{\max}}\sqrt{\frac{3}{\sum_{s=1}^tN_{{i_{\max}},s}}\log \frac{N\scale}{\mu_{\min}}}$$ and for all $i\in \R$, $$c_{i}\mu_{i}\leq \text{LCB}_{i,t}+2\mu_i\sqrt{\frac{3}{\sum_{s=1}^tN_{{i},s}}\log \frac{N\scale}{\mu_{\min}}}.$$
	Since $\text{LCB}_{i,t}-\text{UCB}_{i_{\max},t}\leq 0$, we obtain the statement in \eqref{gaps1}.
	
	\eqref{gaps2} Let $i\in \RP$. If $\text{UCB}_{i_{\max},t}<\text{LCB}_{i,t}$ for some $t$, then since $c_{i_{\max}}\mu_{i_{\max}}\leq \text{UCB}_{i_{\max},t}$ and $c_i\mu_i\geq \text{LCB}_{i,t}$ by~\eqref{conf-interval}, it follows that $c_{i_{\max}}\mu_{i_{\max}}-c_i\mu_i<0$, as required. If $\text{UCB}_{j,t}<\text{LCB}_{i,t}$ for some $j\in \RP$ and some $t$, then we can apply the same argument to show that $c_{j}\mu_{j}-c_i\mu_i<0$. Moreover, we have proved that $c_{i_{\max}}\mu_{i_{\max}}-c_j\mu_j<0$, implying in turn that $c_{i_{\max}}\mu_{i_{\max}}-c_i\mu_i<0$.
	
	\eqref{gaps3} Let $i\in\arg\max_{i\in \RP}\hat c_{i,t}\mu_i$ and $j\in \R$. If $j\in \RP$, then $\hat c_{i,t}\mu_i\geq \hat c_{j,t}\mu_j$ and thus $\hat c_{j,t}\mu_j-\hat c_{i,t}\mu_i\leq 0$. By~\eqref{clean-event:gap}, the statement in \eqref{gaps3} holds. Thus we may assume that $j\in \R\setminus \RP$. By~\eqref{gaps2}, we may assume that $j\neq i_{\max}$. Since $j\notin \RP$, we have $\text{UCB}_{i,t}\geq\text{LCB}_{j,t}$. Note that by~\eqref{conf-interval},
	$$c_{i}\mu_{i}\geq \text{UCB}_{i,t}-2\mu_i\sqrt{\frac{3}{\sum_{s=1}^tN_{{i},s}}\log \frac{N\scale}{\mu_{\min}}}$$
	and
	$$c_{j}\mu_{j}\leq \text{LCB}_{j,t}+2\mu_j\sqrt{\frac{3}{\sum_{s=1}^tN_{{j},s}}\log \frac{N\scale}{\mu_{\min}}}.$$
	Therefore, the statement in \eqref{gaps3} holds, as required.
	
	\eqref{gaps4} Let $i\in\arg\max_{i\in \R}\hat c_{i,t}\mu_i$ and $j\in \R$. Since $\hat c_{i,t}\mu_i\geq \hat c_{j,t}\mu_j$, we get $\hat c_{j,t}\mu_j-\hat c_{i,t}\mu_i\leq 0$. By~\eqref{clean-event:gap},~\eqref{gaps4} holds.
\end{proof}

\subsection{Proof of Lemma~\ref{lemma:uniform-ub-2}}\label{sec:uniform-ub-lemma}

As in Section~\ref{sec:regret-basic}, we assume that $c_1\geq c_2\geq \cdots \geq c_I$.
We number the $N$ jobs from $1$ to $N$ so that jobs $1+\sum_{j\in[i-1]}N_j,\ldots, \sum_{j\in[i]}N_j$ belong to class $i$. Now let $\sigma:[N]\rightarrow [N]$ be the permutation of $[N]$ that corresponds to the sequence of jobs completed by Algorithm~\ref{preempt-then-nonpreemptive-refined}. Let $C^\pi$ and $R^\pi$ denote the cumulative holding cost and the regret incurred up to $\Tc$, the time at which all jobs are completed under Algorithm~\ref{preempt-then-nonpreemptive-refined}, respectively.
\begin{lemma}\label{lemma:firstphase}
	For each class $i\in\I$, let $T_{s,i}$ denote the number of time slots where class $i$ is selected during the preemption period. Let $i_{\max}$ be the class in $\arg\max_{i\in \I}N_i$ that is selected by Algorithm~\ref{preempt-then-nonpreemptive-refined}. Then the following statements hold. 
	\begin{enumerate}[$(a)$]
		\item\label{algo2-lemma1-1} If $i$ is the class of job $\sigma(n)$, then $W_n\leq \Ts- T_{s,i}$.
		\item\label{algo2-lemma1-2} If the class of job $\sigma(n)$ is $i_{\max}$, then $W_n=0$.
	\end{enumerate}
\end{lemma}
\begin{proof}
	\eqref{algo2-lemma1-1}
	Let $t$ be some time slot in which the server gives service to a job other than $\sigma(1),\ldots, \sigma(n)$. After the preemption phase, until the completion of job $\sigma(n)$, the server processes only the jobs $\sigma(1),\ldots, \sigma(n)$. Hence, it follows that $t$ is some time slot in the preemption phase. Moreover, the designated job of class $i$ is completed before job $\sigma(n)$. This implies that the job served in time slot $t$ is not the designated job of class $i$. Therefore, $W_n$ is at most the number of time slots during the preemption phase where a class other than the class of job $\sigma(n)$ is chosen, so we obtain $W_n\leq \Ts-T_{s,i}$.
	
	\eqref{algo2-lemma1-2} For the sake of contradiction, suppose that there exists some time slot $t$ where a job other than $\sigma(1),\ldots, \sigma(n)$ is processed. With the same argument above, we may argue that $t$ is some time slot in the preemption phase. During the preemption phase, if Algorithm~\ref{preempt-then-nonpreemptive-refined} does not select class $i_{\max}$, it must select some class in $\RP$. Hence, if $T_{s,i}>0$, it means that either $i=i_{\max}$ or $i\in \RP$. This implies that some class $i\in \RP$ with $T_{s,i}>0$ is selected in time slot $t$. As $i\in \RP$ before the non-preemptive phase begin, we know that all jobs of class $i$ are completed before any job of class $i_{\max}$, and therefore, the jobs of class $i$ are among $\sigma(1),\ldots, \sigma(n)$. Since a class $i$ job is chosen at time $t$, one of $\sigma(1),\ldots, \sigma(n)$ is served at $t$, but this contradicts the supposition. Therefore, $W_n=0$ if $\sigma(n)$ is of class $i_{\max}$.
\end{proof}

Then it follows from Lemma~\ref{lemma:firstphase} that
$$\sum_{n\in[N]}d_{\sigma(n)}W_n\leq \sum_{n\in[N]}W_n=\sum_{i\in [N]\setminus\{i_{\max}\}} \sum_{n:\sigma(n)\in\J_i}W_n\leq \sum_{i\in [N]\setminus\{i_{\max}\}}N_i(\Ts-T_{s,i})\leq \barn \Ts.$$
Together with Lemma~\ref{lemma:regret-1}, this implies that
\begin{equation}\label{eq:regret-bound}
	R^\pi\leq \barn \Ts + \sum_{n\in[N]}\sum_{\ell\in E_n}(d_{\sigma(\ell)} - d_{\sigma(n)})\scale.
\end{equation}
The following is an important property about $E_n$'s.
\begin{lemma}\label{lemma:largest}
	Let $n\in[N]$ be such that job $\sigma(n)$ is not of class $i_{\max}$. Then
	$$E_n\cap \left\{\ell\in[N]:\ \text{$\sigma(\ell)$ is of class $i_{\max}$}\right\}=\emptyset.$$
	In words, $E_n$ contains no $\ell$ with $\sigma(\ell)$ being in class $i_{\max}$.
\end{lemma}
\begin{proof}
	Let $\ell\in[N]$ be such that job $\sigma(\ell)$ is of class $i_{\max}$. If job $\sigma(n)$ finishes before a job of class $i_{\max}$ by Algorithm~\ref{preempt-then-nonpreemptive-refined}, then set $\RP$ includes the class of $\sigma(n)$ until the class remains in the system. Then Lemma~\ref{lemma:gap}\eqref{gaps2} implies that $d_{\sigma(n)}>d_{\sigma(\ell)}$. Therefore, $\ell\notin E_n$.
\end{proof}

We now prove Lemma~\ref{lemma:uniform-ub-2} based on Lemmas~\ref{lemma:firstphase} and~\ref{lemma:largest}.

We have defined the notion of clean event in Appendix~\ref{sec:clean-event}. Let us consider the case where the clean event does not hold first. Algorithm~\ref{preempt-then-nonpreemptive-refined} is a work conserving policy, under which all jobs must be completed by the end of $N\scale$th time slot. An obvious implication of this is that the completion time of each job is bounded above by $N\scale$. Another straightforward fact is that the expected regret of Algorithm~\ref{preempt-then-nonpreemptive-refined} is upper bounded by it expected cumulative holding cost. As the completion time of each job under Algorithm~\ref{preempt-then-nonpreemptive-refined} is at most $N\scale$, the expected cumulative holding cost is at most $N^2\scale$, and so is the expected regret.

We next focus on the case where the clean event holds. In particular, the statements in Lemma~\ref{lemma:gap} hold. We will bound the second term on the right hand side of~\eqref{eq:regret-bound}. Take $n=1$ and consider $\sum_{\ell\in E_1}(d_{\sigma(\ell)} - d_{\sigma(1)}).$ If the class of $\sigma(1)$ is $i_{\max}$, then 
\begin{equation}\label{eq:1-bound1}
	\sum_{\ell\in E_1}(d_{\sigma(\ell)} - d_{\sigma(1)}) =\sum_{\ell\in E_1:\ \sigma(\ell)\not\in\J_{i_{\max}}}(d_{\sigma(\ell)} - d_{\sigma(1)})
\end{equation}
because $d_{\sigma(\ell)}-d_{\sigma(1)}=0$ for any $\ell\in E_1$ such that $\sigma(\ell)$ is in class $i_{\max}$.
Moreover, $\RP$ is empty, when job $\sigma(1)$ is chosen for non-preemptive serving, which is at the beginning of the $(\Ts+1)$th time slot. Then it follows from Lemma~\ref{lemma:gap}\eqref{gaps1} that for any $\ell\in E_1$, \begin{equation}\label{eq:1-bound2}
	d_{\sigma(\ell)} - d_{\sigma(1)}\leq \sqrt{\frac{12}{\sum_{s=1}^{\Ts+1}N_{i,s}}\log N\scale}+\sqrt{\frac{12}{\sum_{s=1}^{\Ts+1}N_{{i_{\max}},s}}\log N\scale}
\end{equation}
where $i$ denotes the class of $\sigma(\ell)$. Since each job requires $\scale$ units of service to finish, all $N$ jobs remain in the system until the end of the $\scale$th time slot. This means that as $\Ts<\scale$, all $N$ jobs are present in the system at the beginning of the $(\Ts+1)$th time slot. Then we have $N_{i,s}=N_i\geq N_{\min}$ for all $s\leq \Ts+1$ and $i\in \I$. This and \eqref{eq:1-bound1}--\eqref{eq:1-bound2} imply that
\begin{equation}\label{eq:1-bound3}
	\sum_{\ell\in E_1}(d_{\sigma(\ell)} - d_{\sigma(1)})\leq \sum_{\ell\in E_1:\ \sigma(\ell)\not\in\J_{i_{\max}}}\sqrt{\frac{48}{N_{\min}(\Ts+1)}\log N\scale}\leq \barn \sqrt{\frac{48}{N_{\min}(\Ts+1)}\log N\scale}.
\end{equation}
If the class of $\sigma(1)$ is not $i_{\max}$, then Lemma~\ref{lemma:largest} implies that for each $\ell \in E_1$, $\sigma(\ell)$ is not of class $i_{\max}$. Moreover, by Lemma~\ref{lemma:gap}, it follows that
\begin{equation}\label{eq:1-bound4}
	\sum_{\ell\in E_1}(d_{\sigma(\ell)} - d_{\sigma(1)})\leq \sum_{\ell\in E_1}\sqrt{\frac{48}{N_{\min}(\Ts+1)}\log N\scale}\leq \barn \sqrt{\frac{48}{N_{\min}(\Ts+1)}\log N\scale}.
\end{equation}
By~\eqref{eq:1-bound3} and~\eqref{eq:1-bound4}, we obtain
$$R^\pi\leq \barn \Ts+\bar N\scale\sqrt{\frac{48}{N_{\min}(\Ts+1)}\log N\scale}+\sum_{n\geq 2}\sum_{\ell\in E_n}(d_{\sigma(\ell)} - d_{\sigma(n)})\scale.$$

Now it remains to bound the third term on the right hand side of this inequality. For $n\geq 2$, let $t_n$ denote the time when job $\sigma(n)$ is selected by Algorithm~\ref{preempt-then-nonpreemptive-refined} after the preemption phase. As $t_n$ is a moment after jobs $\sigma(1),\ldots, \sigma(n-1)$ are completed, $t_n\geq (n-1)\scale$. If $\sigma(n)$ is of class $i_{\max}$, at time $t_n$, $\RP$ is empty. Then Lemma~\ref{lemma:gap}\eqref{gaps1} implies that
\begin{align}\label{gapeq}
	\begin{aligned}
		d_{\sigma(\ell)} - d_{\sigma(n)}&\leq \sqrt{\frac{12}{\sum_{s=1}^{t_n} N_{\text{class of $\sigma(\ell)$},s}}\log N\scale}+\sqrt{\frac{12}{\sum_{s=1}^{t_n} N_{\text{class of $\sigma(n)$},s}}\log N\scale}\\
		&\leq \sqrt{\frac{12}{\sum_{s=1}^{(n-1)\scale} N_{\text{class of $\sigma(\ell)$},s}}\log N\scale}+\sqrt{\frac{12}{\sum_{s=1}^{(n-1)\scale} N_{\text{class of $\sigma(n)$},s}}\log N\scale}
	\end{aligned}
\end{align}
for any $\ell\in E_n$ since $E_n\subseteq \R$. Moreover, if $\sigma(\ell)$ also belongs to class $i_{\max}$, then $d_{\sigma(\ell)} - d_{\sigma(n)}=0$. Therefore, if $\sigma(n)$ is of class $i_{\max}$, then we can argue the following holds based on~\eqref{gapeq}:
\begin{align}\label{bound1-'}
	\begin{aligned}
		&\sum_{\ell\in E_n}\left(d_{\sigma(\ell)} - d_{\sigma(n)}\right)\\
		&\leq \sum_{\ell\in E_n:\sigma(\ell)\not\in\J_{i_{\max}}}\left(\sqrt{\frac{12}{\sum_{s=1}^{(n-1)\scale} N_{\text{class of $\sigma(\ell)$},s}}\log N\scale}+\sqrt{\frac{12}{\sum_{s=1}^{(n-1)\scale} N_{\text{class of $\sigma(n)$},s}}\log N\scale}\right)\\
		&\leq\sum_{\ell\in E_n:\sigma(\ell)\not\in\J_{i_{\max}}}\sqrt{\frac{12}{\sum_{s=1}^{(n-1)\scale} N_{\text{class of $\sigma(\ell)$},s}}\log N\scale}+\barn\sqrt{\frac{12}{\sum_{s=1}^{(n-1)\scale} N_{\text{class of $\sigma(n)$},s}}\log N\scale}
	\end{aligned}
\end{align}
where the second inequality is because the number of indices $\ell\in E_n$ with $\sigma(\ell)$ not being in class $i_{\max}$ is at most $\barn$. If $\sigma(n)$ does not belong to class $i_{\max}$, then Lemma~\ref{lemma:largest} implies that for each $\ell \in E_n$, $\sigma(\ell)$ is not of class $i_{\max}$. Moreover, by Lemma~\ref{lemma:gap}, it follows that
\begin{align}\label{bound1-''}
	\begin{aligned}
		&\sum_{\ell\in E_n}\left(d_{\sigma(\ell)} - d_{\sigma(n)}\right)\\
		&\leq \sum_{\ell\in E_n}\left(\sqrt{\frac{12}{\sum_{s=1}^{(n-1)\scale} N_{\text{class of $\sigma(\ell)$},s}}\log N\scale}+\sqrt{\frac{12}{\sum_{s=1}^{(n-1)\scale} N_{\text{class of $\sigma(n)$},s}}\log N\scale}\right)\\
		&\leq\sum_{\ell\in E_n:\sigma(\ell)\not\in\J_{i_{\max}}}\sqrt{\frac{12}{\sum_{s=1}^{(n-1)\scale} N_{\text{class of $\sigma(\ell)$},s}}\log N\scale}+\barn\sqrt{\frac{12}{\sum_{s=1}^{(n-1)\scale} N_{\text{class of $\sigma(n)$},s}}\log N\scale}
	\end{aligned}
\end{align}
where the second inequality is due to Lemma~\ref{lemma:largest} which also implies that $|E_n|\leq \barn$. By~\eqref{bound1-'} and~\eqref{bound1-''}, we obtain
\begin{align}\label{bound1}
	\begin{aligned}
		&\sum_{n\geq 2}\sum_{\ell\in E_n}\left(d_{\sigma(\ell)}-d_{\sigma(n)}\right)\\
		&\leq \sum_{n\geq 2}\sum_{\ell\in E_n:\sigma(\ell)\not\in\J_{i_{\max}}}\sqrt{\frac{12}{\sum_{s=1}^{(n-1)\scale} N_{\text{class of $\sigma(\ell)$},s}}\log N\scale}\\
		&\quad+\sum_{n\geq 2}\barn\sqrt{\frac{12}{\sum_{s=1}^{(n-1)\scale} N_{\text{class of $\sigma(n)$},s}}\log N\scale}
	\end{aligned}
\end{align}
We look at the second sum at the right-hand side of inequality~\eqref{bound1} first.

\begin{equation}\label{bound2}
	\sum_{n\geq 2}\sqrt{\frac{12}{\sum_{s=1}^{(n-1)\scale} N_{\text{class of $\sigma(n)$},s}}}=\sum_{i\in\I}\sum_{n\geq2:\sigma(n)\in\J_i}\sqrt{\frac{12}{\sum_{s=1}^{(n-1)\scale} N_{\text{class of $\sigma(n)$},s}}}.
\end{equation}
Let $i\in\I$ be a class that $\sigma(1)$ does not belong to. Then for some $2\leq n_{i_1}\leq\cdots \leq n_{i_{N_i}}$, jobs $\sigma(n_{i_1}),\ldots, \sigma(n_{i_{N_i}})$ are in class $i$. Then
\begin{align}\label{bound3}
	\begin{aligned}
		&\sum_{n\geq2:\sigma(n)\in\J_i}\sqrt{\frac{12}{\sum_{s=1}^{(n-1)\scale} N_{\text{class of $\sigma(n)$},s}}}\\
		&=\sum_{k=1}^{N_i}\sqrt{\frac{12}{\sum_{s=1}^{(n_{i_k}-1)\scale} N_{i,s}}}\\
		&=\sum_{k=1}^{\lfloor N_i/2\rfloor}\sqrt{\frac{12}{\sum_{s=1}^{(n_{i_k}-1)\scale} N_{i,s}}}+\sqrt{\frac{12}{\sum_{s=1}^{(n_{i_{\lceil N_i/2\rceil}}-1)\scale} N_{i,s}}}+\sum_{k=\lceil N_i/2\rceil+1}^{N_i}\sqrt{\frac{12}{\sum_{s=1}^{(n_{i_k}-1)\scale} N_{i,s}}}\\
		&\leq \sum_{k=1}^{\lfloor N_i/2\rfloor}\sqrt{\frac{12}{\sum_{s=1}^{(n_{i_k}-1)\scale} N_{i,s}}}+\sqrt{\frac{12}{\sum_{s=1}^{(n_{i_{\lceil N_i/2\rceil}}-1)\scale} N_{i,s}}}+\lfloor\frac{N_i}{2}\rfloor\sqrt{\frac{12}{\sum_{s=1}^{n_{i_{\lceil N_i/2\rceil}}\scale} N_{i,s}}}\\
		&\leq 3\sum_{k=1}^{\lfloor N_i/2\rfloor}\sqrt{\frac{12}{\sum_{s=1}^{(n_{i_k}-1)\scale} N_{i,s}}}\\
		&\leq 3\sum_{k=1}^{\lfloor N_i/2\rfloor}\sqrt{\frac{24}{(n_{i_k}-1)\scale N_i}}
	\end{aligned}
\end{align}
where the first inequality is due to $\sum_{s=1}^{(n_{i_{k}}-1)\scale} N_{i,s}\geq \sum_{s=1}^{n_{i_{\lceil N_i/2\rceil}}\scale} N_{i,s}$ for any $k\geq \lceil N_i/2\rceil+1$, the second inequality comes from  $\sum_{s=1}^{(n_{i_{k}}-1)\scale} N_{i,s}\leq \sum_{s=1}^{n_{i_{\lceil N_i/2\rceil}}\scale} N_{i,s}$ for any $k\leq \lceil N_i/2\rceil$, and the last inequality is because at least $N_i/2$ jobs of class $i$ remain in the system until choosing the $\lceil N_i/2\rceil$th job of class $i$. 

If $\sigma(1)$ is of class $i\in \I$, then for some $2\leq n_{i_1}\leq\cdots \leq n_{i_{N_i-1}}$, jobs $\sigma(n_{i_1}),\ldots, \sigma(n_{i_{N_i-1}})$ are in class $i$. Here, if $N_i=1$, then
\begin{equation}\label{bound4}
	\sum_{n\geq2:\sigma(n)\in\J_i}\sqrt{\frac{12}{\sum_{s=1}^{(n-1)\scale} N_{\text{class of $\sigma(n)$},s}}}=0.
\end{equation}
Now assume that $N_i\geq 2$. Then
\begin{align}\label{bound5}
	\begin{aligned}
		\sum_{n\geq2:\sigma(n)\in\J_i}\sqrt{\frac{12}{\sum_{s=1}^{(n-1)\scale} N_{\text{class of $\sigma(n)$},s}}}
		&=\sum_{k=1}^{N_i-1}\sqrt{\frac{12}{\sum_{s=1}^{(n_{i_k}-1)\scale} N_{i,s}}}\\
		&=\sum_{k=1}^{\lfloor N_i/2\rfloor}\sqrt{\frac{12}{\sum_{s=1}^{(n_{i_k}-1)\scale} N_{i,s}}}+\sum_{k=\lceil N_i/2\rceil}^{N_i-1}\sqrt{\frac{12}{\sum_{s=1}^{(n_{i_k}-1)\scale} N_{i,s}}}\\
		&\leq \sum_{k=1}^{\lfloor N_i/2\rfloor}\sqrt{\frac{12}{\sum_{s=1}^{(n_{i_k}-1)\scale} N_{i,s}}}+\lfloor\frac{N_i}{2}\rfloor\sqrt{\frac{12}{\sum_{s=1}^{n_{i_{\lfloor N_i/2\rfloor}}\scale} N_{i,s}}}\\
		&\leq 2\sum_{k=1}^{\lfloor N_i/2\rfloor}\sqrt{\frac{12}{\sum_{s=1}^{(n_{i_k}-1)\scale} N_{i,s}}}\\
		&\leq 2\sum_{k=1}^{\lfloor N_i/2\rfloor}\sqrt{\frac{48}{(n_{i_k}-1)\scale N_i}}
	\end{aligned}
\end{align}
where the first inequality is due to $\sum_{s=1}^{(n_{i_{k}}-1)\scale} N_{i,s}\geq \sum_{s=1}^{n_{i_{\lfloor N_i/2\rfloor}}\scale} N_{i,s}$ for any $k\geq \lceil N_i/2\rceil$, the second inequality comes from  $\sum_{s=1}^{(n_{i_{k}}-1)\scale} N_{i,s}\leq \sum_{s=1}^{n_{i_{\lfloor N_i/2\rfloor}}\scale} N_{i,s}$ for any $k\leq \lfloor N_i/2\rfloor$, and the last inequality is because at least $\lfloor N_i/2\rfloor \geq N_i/4$ jobs of class $i$ remain in the system until choosing the $\lfloor N_i/2\rfloor$th job of class $i$. 

Then it follows from~\eqref{bound2}--\eqref{bound5} that 
\begin{align}\label{bound6}
	\begin{aligned}
		\sum_{n\geq 2}\barn\sqrt{\frac{12}{\sum_{s=1}^{(n-1)\scale} N_{\text{class of $\sigma(n)$},s}}\log N\scale}\leq \sum_{i\in\I}6\barn\sqrt{\log N\scale}\sum_{k=1}^{\lfloor N_i/2\rfloor}\sqrt{\frac{6}{(n_{i_k}-1)\scale N_i}}.
	\end{aligned}
\end{align}
Next, we turn our attention to the first sum at the right-hand side of inequality~\eqref{bound1}. Note that
\begin{align}\label{bound7}
	\begin{aligned}
		&\sum_{n\geq 2}\sum_{\ell\in E_n:\sigma(\ell)\not\in\J_{i_{\max}}}\sqrt{\frac{12}{\sum_{s=1}^{(n-1)\scale} N_{\text{class of $\sigma(\ell)$},s}}}\\
		&=\sum_{n\geq 2}\sum_{i\in \I\setminus\{i_{\max}\}}\left|\left\{\ell\in E_n:\sigma(\ell)\in\J_i\right\}\right|\sqrt{\frac{12}{\sum_{s=1}^{(n-1)\scale} N_{i,s}}}\\
		&=\sum_{i\in \I\setminus\{i_{\max}\}}\sum_{n\geq 2}\left|\left\{\ell\in E_n:\sigma(\ell)\in\J_i\right\}\right|\sqrt{\frac{12}{\sum_{s=1}^{(n-1)\scale} N_{i,s}}}\\
		&\leq\sum_{i\in \I\setminus\{i_{\max}\}}\sum_{n\geq 2:\text{$\exists \ell\in E_n$ with $\sigma(\ell)\in\J_i$}}N_i\sqrt{\frac{12}{\sum_{s=1}^{(n-1)\scale} N_{i,s}}}
	\end{aligned}
\end{align}
because $\left|\left\{\ell\in E_n:\sigma(\ell)\in\J_i\right\}\right|=0$ or $\left|\left\{\ell\in E_n:\sigma(\ell)\in\J_i\right\}\right|\leq N_i$.
Let $i\in \I\setminus\{i_{\max}\}$. If $\sigma(1)$ is not in class $i$, then as before, for some $2\leq n_{i_1}\leq\cdots \leq n_{i_{N_i}}$, jobs $\sigma(n_{i_1}),\ldots, \sigma(n_{i_{N_i}})$ are in class $i$. Moreover, 
\begin{align}\label{bound8'}
	\begin{aligned}
		&\sum_{n\geq 2:\text{$\exists \ell\in E_n$ with $\sigma(\ell)$ in class $i$}}N_i\sqrt{\frac{12}{\sum_{s=1}^{(n-1)\scale} N_{i,s}}}\\
		&\leq N_i\sum_{n=2}^{n_{i_{\lfloor N_i/2\rfloor}}}\sqrt{\frac{12}{\sum_{s=1}^{(n-1)\scale} N_{i,s}}}+\sum_{n\geq n_{i_{\lfloor N_i/2\rfloor}}+1:\text{$\exists \ell\in E_n$ with $\sigma(\ell)\in\J_i$}}N_i\sqrt{\frac{12}{\sum_{s=1}^{(n-1)\scale} N_{i,s}}}\\
		&\leq N_i\sum_{n=2}^{n_{i_{\lfloor N_i/2\rfloor}}}\sqrt{\frac{24}{(n-1)\scale N_i}}+\sum_{n\geq n_{i_{\lfloor N_i/2\rfloor}}+1:\text{$\exists \ell\in E_n$ with $\sigma(\ell)\in\J_i$}}N_i\sqrt{\frac{12}{\sum_{s=1}^{(n-1)\scale} N_{i,s}}}
	\end{aligned}
\end{align}
where the second inequality is because there are at least $N_i/2$ jobs waiting until the selection of the $\lfloor N_i/2\rfloor$th job of class $i$. Next, let $n\geq n_{i_{\lfloor N_i/2\rfloor}}+1$. If $\sigma(n)$ is of class $i_{\max}$, then as $\sigma(n_{i_{\lfloor N_i/2\rfloor}})$ finishes before $\sigma(n)$, class $i$ belongs to $\RP$ when job $\sigma(n_{i_{\lfloor N_i/2\rfloor}})$ is selected. This means that Algorithm~\ref{preempt-then-nonpreemptive-refined} completes all the remaining jobs in class $i$ before job $\sigma(n)$, so $E_n$ does not contain $\ell$ such that $\sigma(\ell)$ is in class $i$. Then it follows that 
\begin{align}\label{bound8''}
	\begin{aligned}
		&\sum_{n\geq n_{i_{\lfloor N_i/2\rfloor}}+1:\text{$\exists \ell\in E_n$ with $\sigma(\ell)\in\J_i$}}N_i\sqrt{\frac{12}{\sum_{s=1}^{(n-1)\scale} N_{i,s}}}\\
		&\leq \sum_{n\geq n_{i_{\lfloor N_i/2\rfloor}}+1:\sigma(n)\not\in\J_{i_{\max}}}N_i\sqrt{\frac{12}{\sum_{s=1}^{(n-1)\scale} N_{i,s}}}\\
		&\leq \barn N_i\sqrt{\frac{12}{\sum_{s=1}^{n_{i_{\lfloor N_i/2\rfloor}}\scale} N_{i,s}}}.
	\end{aligned}
\end{align}
Moreover,
\begin{align}\label{bound8'''}
	\begin{aligned}
		\frac{N_i}{2}\cdot\sqrt{\frac{12}{\sum_{s=1}^{n_{i_{\lfloor N_i/2\rfloor}}\scale} N_{i,s}}}\leq 2\sum_{k=1}^{\lfloor N_i/2\rfloor}\sqrt{\frac{12}{\sum_{s=1}^{(n_{i_{k}}-1)\scale} N_{i,s}}}\leq2\sum_{k=1}^{\lfloor N_i/2\rfloor}\sqrt{\frac{24}{(n_{i_k}-1)\scale N_i}}
	\end{aligned}
\end{align}
where the first inequality holds true because $n_{i_k}\leq n_{i_{\lfloor N_i/2\rfloor}}$ for $k\leq \lfloor N_i/2\rfloor$ and $N_i/2\leq 2\lfloor N_i/2\rfloor$ and the second inequality holds because there are at least $N_i/2$ jobs remaining until choosing the $\lfloor N_i/2\rfloor$th job is chosen. Combining~\eqref{bound8'}--\eqref{bound8'''}, we obtain
\begin{align}\label{bound8}
	\begin{aligned}
		\sum_{n\geq 2:\text{$\exists \ell\in E_n$ with $\sigma(\ell)\in\J_i$}}N_i\sqrt{\frac{12}{\sum_{s=1}^{(n-1)\scale} N_{i,s}}}\leq N_i\sum_{n=2}^{n_{i_{\lfloor N_i/2\rfloor}}}\sqrt{\frac{24}{(n-1)\scale N_i}}+4\barn\sum_{k=1}^{\lfloor N_i/2\rfloor}\sqrt{\frac{24}{(n_{i_k}-1)\scale N_i}}
	\end{aligned}
\end{align}

Now let $i$ be the class of $\sigma(1)$. If $N_i=1$, then 
\begin{equation}\label{bound9}
	\sum_{n\geq 2:\text{$\exists \ell\in E_n$ with $\sigma(\ell)\in\J_i$}}N_i\sqrt{\frac{12}{\sum_{s=1}^{(n-1)\scale} N_{i,s}}}=0.
\end{equation}
In this case, we set $n_{i_1}=\bar N+1$.
If $N_i\geq 2$, as before, for some $2\leq n_{i_1}\leq\cdots \leq n_{i_{N_i}-1}$, jobs $\sigma(n_{i_1}),\ldots, \sigma(n_{i_{N_i}-1})$ are in class $i$. Then we can similarly argue that
\begin{align}\label{bound10}
	\sum_{n\geq 2:\text{$\exists \ell\in E_n$ with $\sigma(\ell)\in\J_i$}}N_i\sqrt{\frac{12}{\sum_{s=1}^{(n-1)\scale} N_{i,s}}}
	\leq N_i\sum_{n=2}^{n_{i_{\lfloor N_i/2\rfloor}}}\sqrt{\frac{48}{(n-1)\scale N_i}}+2\barn\sum_{k=1}^{\lfloor N_i/2\rfloor}\sqrt{\frac{48}{(n_{i_k}-1)\scale N_i}}.
\end{align}
Then~\eqref{bound7}--\eqref{bound10} imply that
\begin{align}\label{bound11}
	\begin{aligned}
		&\sum_{n\geq 2}\sum_{\ell\in E_n:\sigma(\ell)\not\in\J_{i_{\max}}}\sqrt{\frac{12}{\sum_{s=1}^{(n-1)\scale} N_{\text{class of $\sigma(\ell)$},s}}\log N\scale}\\
		&\leq \sum_{i\in\I\setminus\{i_{\max}\}}\left(4N_i\sqrt{\log N\scale}\sum_{n=2}^{n_{i_{\lfloor N_i/2\rfloor}}}\sqrt{\frac{3}{(n-1)\scale N_i}}+8\barn\sqrt{\log N\scale}\sum_{k=1}^{\lfloor N_i/2\rfloor}\sqrt{\frac{6}{(n_{i_k}-1)\scale N_i}}\right).
	\end{aligned}
\end{align}
Since~\eqref{bound6} and~\eqref{bound11} provide upper bounds on the first and second terms at the rightmost side of~\eqref{bound1}, we obtain
\begin{align}\label{bound12}
	\begin{aligned}
		&\sum_{n\geq 2}\sum_{\ell\in E_n}\left(d_{\sigma(\ell)}-d_{\sigma(n)}\right)\scale\\
		&\leq \sum_{i\in\I\setminus\{i_{\max}\}}4\scale N_i\sqrt{\log N\scale}\sum_{n=2}^{n_{i_{\lfloor N_i/2\rfloor}}}\sqrt{\frac{3}{(n-1)\scale N_i}}+ \sum_{i\in\I}14\barn\scale\sqrt{\log N\scale}\sum_{k=1}^{\lfloor N_i/2\rfloor}\sqrt{\frac{6}{(n_{i_k}-1)\scale N_i}}.
	\end{aligned}
\end{align}
Consequently, it remains to bound the two terms on the right-hand side of inequality~\eqref{bound12}. We will show that both terms are at most
$$\kappa\cdot \left(\left(\sqrt{I N\bar N}+\min\left\{I\barn,\sqrt{I}\barn\sqrt{\log \barn},\frac{\barn^{3/2}}{N_{\min}^{1/2}}\right\}\right)\sqrt{\scale\log N\scale}\right)$$
for some constant $\kappa>0$, completing the proof of Lemma~\ref{lemma:uniform-ub-2}.

Let us first consider the second sum on the right-hand side of~\eqref{bound12}, for which we provide three different bounds. First, the following holds for some constants $\kappa_0,\kappa_1>0$:
\begin{align}\label{final-bound1}
	\begin{aligned}
		&\sum_{i\in\I}14\barn\scale\sqrt{\log N\scale}\sum_{k=1}^{\lfloor N_i/2\rfloor}\sqrt{\frac{6}{(n_{i_k}-1)\scale N_i}}\\
		&= 14\barn\sqrt{\scale\log N\scale}\left(\sum_{k=1}^{\lfloor N_{i_{\max}}/2\rfloor}\sqrt{\frac{6}{(n_{(i_{max})_k}-1)N_{i_{\max}}}}+\sum_{i\in\I\setminus\{i_{\max}\}}\sum_{k=1}^{\lfloor N_i/2\rfloor}\sqrt{\frac{6}{(n_{i_k}-1)N_i}}\right)\\
		&\leq 14\barn\sqrt{\scale\log N\scale}\left(\sum_{n=2}^{\lfloor N_{i_{\max}}/2\rfloor+1}\sqrt{\frac{6}{(n-1)N_{i_{\max}}}}+\sum_{i\in\I\setminus\{i_{\max}\}}\sum_{k=1}^{\lfloor N_i/2\rfloor}\sqrt{\frac{6}{(n_{i_k}-1)N_{\min}}}\right)\\
		&\leq 14\barn\sqrt{\scale\log N\scale}\left(\sum_{n=2}^{\lfloor N_{i_{\max}}/2\rfloor+1}\sqrt{\frac{6}{(n-1)N_{i_{\max}}}}+\sum_{n=2}^{\barn+1}\sqrt{\frac{6}{(n-1)N_{\min}}}\right)\\
		&\leq \kappa_0\barn\sqrt{\scale\log N\scale}\left(\sqrt{N_{i_{\max}}}\cdot\frac{1}{\sqrt{N_{i_{\max}}}}+\sqrt{\barn}\cdot\frac{1}{\sqrt{N_{\min}}}\right)\\
		&\leq \kappa_1\barn\sqrt{\scale\log N\scale}\sqrt{\barn}\cdot\frac{1}{\sqrt{N_{\min}}}\\
		&=\kappa_1\cdot \frac{\barn^{3/2}}{N_{\min}^{1/2}}\sqrt{\scale\log N\scale}
	\end{aligned}
\end{align}
where the first inequality is by $n_{(i_{max})_k}\geq k+1$ and $N_i\geq N_{\min}$, the second inequality is because each $n_{i_k}\geq2$ with $i\neq i_{\max}$ is at least $2$ and the number of such $n_{(i_{max})_k}$'s is at most $\barn$, and the third inequality is because $\sum_{n=1}^\ell 1/\sqrt{n}=O(\sqrt{\ell})$.

Second, for some constant $\kappa_2>0$, the following holds:
\begin{align}\label{final-bound2}
	\begin{aligned}
		\sum_{i\in\I}14\barn\scale\sqrt{\log N\scale}\sum_{k=1}^{\lfloor N_i/2\rfloor}\sqrt{\frac{6}{(n_{i_k}-1)\scale N_i}}
		&\leq 14\barn\sqrt{\scale\log N\scale}\sum_{i\in\I}\sum_{k=1}^{\lfloor N_i/2\rfloor}\sqrt{\frac{6}{kN_i}}\\
		&\leq 14\barn\sqrt{\scale\log N\scale}\sum_{i\in\I}\frac{1}{\sqrt{N_i}}\sum_{k=1}^{\lfloor N_i/2\rfloor}\sqrt{\frac{6}{k}}\\
		&\leq \kappa_2 \cdot \barn\sqrt{\scale\log N\scale}\sum_{i\in\I}\frac{1}{\sqrt{N_i}}\sqrt{N_i}\\
		&=\kappa_2 \cdot I\barn\sqrt{\scale\log N\scale}
	\end{aligned}
\end{align}
where the first inequality is because $n_{i_k}\geq k+1$.

Lastly, for some constant $\kappa_3,\kappa_4>0$,
\begin{align}\label{final-bound3}
	\begin{aligned}
		&\sum_{i\in\I}14\barn\scale\sqrt{\log N\scale}\sum_{k=1}^{\lfloor N_i/2\rfloor}\sqrt{\frac{6}{(n_{i_k}-1)\scale N_i}}\\
		&= 14\barn\sqrt{\scale\log N\scale}\left(\sum_{k=1}^{\lfloor N_{i_{\max}}/2\rfloor}\sqrt{\frac{6}{(n_{(i_{max})_k}-1)N_{i_{\max}}}}+\sum_{i\in\I\setminus\{i_{\max}\}}\sum_{k=1}^{\lfloor N_i/2\rfloor}\sqrt{\frac{12}{(n_{i_k}-1)N_i}}\right)\\
		&\leq \kappa_3\barn\sqrt{\scale\log N\scale}\sum_{i\in\I\setminus\{i_{\max}\}}\sum_{k=1}^{\lfloor N_i/2\rfloor}\sqrt{\frac{6}{(n_{i_k}-1)N_i}}\\
		&\leq \kappa_3\barn\sqrt{\scale\log N\scale}\sqrt{\sum_{i\in\I\setminus\{i_{\max}\}}\sum_{k=1}^{\lfloor N_i/2\rfloor}\frac{1}{n_{i_k}-1}}\sqrt{\sum_{i\in\I\setminus\{i_{\max}\}}\sum_{k=1}^{\lfloor N_i/2\rfloor}\frac{6}{N_i}}\\
		&\leq \kappa_3\barn\sqrt{\scale\log N\scale}\sqrt{\sum_{n=1}^{\barn}\frac{1}{n}}\sqrt{\sum_{i\in\I}N_i\cdot\frac{6}{N_i}}\\
		&\leq \kappa_4\cdot \sqrt{I}\barn\sqrt{\log \barn}\sqrt{\scale\log N\scale}
	\end{aligned}
\end{align}
where the first inequality is because
$$\sum_{k=1}^{\lfloor N_{i_{\max}}/2\rfloor}\sqrt{\frac{6}{(n_{(i_{max})_k}-1)N_{i_{\max}}}}\leq \sum_{n=2}^{\lfloor N_{i_{\max}}/2\rfloor+1}\sqrt{\frac{6}{(n-1)N_{i_{\max}}}}=O\left(\sqrt{N_{i_{\max}}}\cdot\frac{1}{\sqrt{N_{i_{\max}}}}\right)$$
which we observed when considering~\eqref{final-bound1}, the second inequality is given by the Cauchy-Schwarz inequality, and the last inequality is because ${\sum_{n\geq 2}1/n}=O(\log N)$. 
Hence,~\eqref{final-bound1}--\eqref{final-bound3} imply the desired bound on the second sum:
\begin{align}\label{final-bound4}
	\begin{aligned}
		&\sum_{i\in\I}14\barn\scale\sqrt{\log N\scale}\sum_{k=1}^{\lfloor N_i/2\rfloor}\sqrt{\frac{6}{(n_{i_k}-1)\scale N_i}}\\
		&\leq \max\{\kappa_1,\kappa_2,\kappa_4\}\cdot\min\left\{I\barn, \sqrt{I}\barn\sqrt{\log \barn}, \frac{\barn^{3/2}}{N_{\min}^{1/2}}\right\}\sqrt{\scale\log N\scale}
	\end{aligned}
\end{align}
Next we consider the first sum. We show that
\begin{align}\label{final-bound5}
	\begin{aligned}
		\sum_{i\in\I\setminus\{i_{\max}\}}4\scale N_i\sqrt{\log N\scale}\sum_{n=2}^{n_{i_{\lfloor N_i/2\rfloor}}}\sqrt{\frac{3}{(n-1)\scale N_i}}
		&=\sum_{i\in\I\setminus\{i_{\max}\}}4\sqrt{N_i\scale\log N\scale}\sum_{n=2}^{n_{i_{\lfloor N_i/2\rfloor}}}\sqrt{\frac{3}{(n-1)}}\\
		&\leq \sum_{i\in\I\setminus\{i_{\max}\}}2\sqrt{N_i\scale\log N \scale}\sum_{n=2}^{N}\sqrt{\frac{3}{(n-1)}}\\
		&\leq \kappa_5\cdot \sqrt{N\scale\log N\scale}\sum_{i\in\I\setminus\{i_{\max}\}} \sqrt{N_i}\\
		&\leq \kappa_5\cdot \sqrt{N\scale\log N\scale}\cdot \sqrt{I\barn}\\
		&=\kappa_5\cdot \sqrt{I N\bar N}\cdot \sqrt{\scale \log N\scale}
	\end{aligned}
\end{align}
holds for some constant $\kappa_5>0$
where the first inequality is due to $n_{i_{\lfloor N_i/2\rfloor}}\leq N$, the second inequality is because $\sum_{n=2}^N\sqrt{1/(n-1)} =O(\sqrt{N})$, and the last inequality is by the Cauchy–Schwarz inequality. 
Finally, combining~\eqref{bound12},~\eqref{final-bound4}, and~\eqref{final-bound5}, we show that the two terms on the right-hand side of~\eqref{bound12} are bounded above by
$$\kappa\cdot \left(\sqrt{I N\bar N\scale \log N\scale}+\min\left\{I\barn,\sqrt{I}\barn\sqrt{\log \barn},\frac{\barn^{3/2}}{N_{\min}^{1/2}}\right\}\sqrt{\scale\log N\scale}\right)$$
for some $\kappa>0$ as required.

Therefore, note that
\begin{align*}
	&\mathbb{E}[R^\pi]\\&= \mathbb{E}[R^\pi\mid \neg~\text{clean event}]\cdot \mathbb{P}[\neg~\text{clean event}] +\mathbb{E}[R^\pi \mid \text{clean event}]\cdot \mathbb{P}[\text{clean event}]\\
	&=O\left(\frac{2}{N^4\scale}N^2\scale\right.\\
	&\qquad+\left(1-\frac{2}{N^4\scale}\right)\left(\barn \Ts + \barn\scale\sqrt{\frac{48}{N_{\min} \Ts}\log N\scale}+\sqrt{IN\barn}\sqrt{\scale \log N\scale}\right.\\
	&\qquad\left.\left.+\min\left\{I\barn,\sqrt{I}\barn\sqrt{\log \barn},\frac{\barn^{3/2}}{N_{\min}^{1/2}}\right\}\sqrt{\scale \log N\scale}\right)\right)\\
	&=O\left(\barn \Ts + \barn\scale\sqrt{\frac{48}{N_{\min} (\Ts+1)}\log N\scale}+\left(\sqrt{IN\barn}+\min\left\{I\barn,\sqrt{I}\barn\sqrt{\log \barn},\frac{\barn^{3/2}}{N_{\min}^{1/2}}\right\}\right)\sqrt{\scale \log N\scale}\right),
\end{align*}
completing the proof of Lemma~\ref{lemma:uniform-ub-2}. 

\subsection{Completing the proof: the length of the preemption phase}\label{sec:pluuing-in-tau-second}

By Lemma~\ref{lemma:Ts-choice}, if $\Ts$ is given as in~\eqref{eq:Ts-choice}, then
$$\barn \Ts + \frac{\barn \scale(\log N\scale)^{1/2}}{N_{\min}^{1/2}(\Ts+1)^{1/2}}=O\left(\max\left\{\frac{\barn\scale^{2/3}(\log N\scale)^{1/3}}{N_{\min}^{1/3}},\ \frac{\barn\scale^{1/2}({\log N\scale})^{1/2}}{N_{\min}^{1/2}},\ \barn({\log N\scale})^{1/2}\right\}\right).$$
Then~\eqref{regret:uniform-ub-2'} is upper bounded by~\eqref{regret:uniform-ub-2''}. As~\eqref{regret:uniform-ub-2'} is an upper bound on the expected regret of Algorithm~\ref{preempt-then-nonpreemptive-refined} by Lemma~\ref{lemma:uniform-ub-2},~\eqref{regret:uniform-ub-2''} is also an upper bound on the expected regret. As we explained in Section~\ref{sec:refined}, the bound~\eqref{regret:uniform-ub-2''} implies~\eqref{regret:uniform-ub-2} because $N_{\min}\geq 1$.

\section{Proofs for the extension results in Section~\ref{sec:extensions}}

\subsection{Regret under heterogeneous service times}\label{sec:extension-tool}

Assume that $$c_1\mu_1\geq c_2\mu_2\geq \cdots \geq c_I\mu_I.$$ As in Section~\ref{sec:regret-basic}, we number the $N$ jobs from $1$ to $N$ so that jobs $1+\sum_{j\in[i-1]}N_j,\ldots, \sum_{j\in[i]}N_j$ belong to class~$i$. We use notation $d_n$ to denote the mean per-time holding cost of job $n\in\J$, so $d_n=c_i$ if $i$ is the class of job $n$. Moreover, we introduce notation $\lambda_n$ to denote the service rate of job $n\in\J$, so if $i$ is the class of job $n$, then we have $\lambda_n=\mu_i$.

Now let $\sigma:[N]\rightarrow [N]$ be the permutation of $[N]$ that corresponds to the sequence of jobs completed by an algorithm $\pi$. Let $C^\pi$ and $R^\pi$ denote the cumulative holding cost and the regret incurred up to $\Tc$, the time at which all jobs are completed under $\pi$, respectively. Let $W_n$ denote the number of time steps that the server spends to serve jobs other than $\sigma(1),\ldots, \sigma(n)$ before completing job $\sigma(n)$. Then job $\sigma(n)$ stays in the system for precisely $W_n+n\scale$ time steps. Then $C^\pi$ is given by 
$$
C^\pi=\sum_{n\in[N]}d_{\sigma(n)}W_n+\sum_{n\in[N]}d_{\sigma(n)} \sum_{\ell\in[n]}\frac{\scale}{\lambda_{\sigma(\ell)}}.$$
The minimum holding cost is $\sum_{n\in[N]} d_n \sum_{\ell\in[n]}{\scale}/{\lambda_{\ell}}$, so we have
\begin{equation}\label{prelim-regret'-hetero}
	R^\pi=\sum_{n\in[N]}d_{\sigma(n)}W_n+\sum_{n\in[N]}\left(d_{\sigma(n)}\sum_{\ell\in[n]}\frac{\scale}{\lambda_{\sigma(\ell)}}-d_n\sum_{\ell\in[n]}\frac{\scale}{\lambda_{\ell}}\right).\end{equation}
Based on the following lemma, we can rewrite the regret expression for $R^\pi$ given in~\eqref{prelim-regret'-hetero} so that $R^\pi$ can be written as a sum of some nonnegative terms only. If we define $E_n$ as in~\eqref{eq:E_n}, then we know that
$$E_n\supseteq\left\{\ell\in[N]:\ \text{$\sigma(n)$ finishes before $\sigma(\ell)$ and $d_{\sigma(\ell)}\lambda_{\sigma(\ell)}> d_{\sigma(n)}\lambda_{\sigma(n)}$}\right\}.$$
Note that for any $n\in[N]$ and $\ell\in E_n$, we know that $d_{\sigma(\ell)}\lambda_{\sigma(\ell)}-d_{\sigma(n)}\lambda_{\sigma(n)}\geq 0$. The following lemma is a generalization of Lemma~\ref{lemma:regret-1} to the case of heterogeneous service times.
\begin{lemma}\label{lemma:regret-1-hetero}
	Let $E_n$ be defined as in~\eqref{eq:E_n}. Then
	\begin{equation}\label{prelim-regret-hetero}
		R^\pi=\sum_{n\in[N]}d_{\sigma(n)}W_n+\sum_{n\in[N]}\sum_{\ell\in E_n}\left(d_{\sigma(\ell)}\lambda_{\sigma(\ell)}-d_{\sigma(n)}\lambda_{\sigma(n)}\right)\frac{\scale}{\lambda_{\sigma(\ell)}\lambda_{\sigma(n)}}.
	\end{equation}
\end{lemma}
\begin{proof}
	It is sufficient to show that the following relation holds.
	\begin{equation}\label{eq:hetero-lemma} \sum_{n\in[N]}\left(d_{\sigma(n)}\sum_{\ell\in[n]}\frac{\scale}{\lambda_{\sigma(n)}}-d_{n}\sum_{\ell\in[n]}\frac{\scale}{\lambda_\ell}\right)=\sum_{n\in[N]}\sum_{\ell\in E_n}\left(d_{\sigma(\ell)}\lambda_{\sigma(\ell)}-d_{\sigma(n)}\lambda_{\sigma(n)}\right)\frac{\scale}{\lambda_{\sigma(\ell)}\lambda_{\sigma(n)}}.
	\end{equation}
	First, by rearranging terms, we obtain
	\begin{align*}
		\sum_{n\in[N]}d_{\sigma(n)}\sum_{\ell\in[n]}\frac{\scale}{\lambda_{\sigma(\ell)}}&=\sum_{n\in[N]}\frac{\scale}{\lambda_{\sigma(n)}}\sum_{\ell\geq n}d_{\sigma(\ell)}\\
		&=\sum_{n\in[N]}d_{\sigma(n)}\frac{\scale}{\lambda_{\sigma(n)}}+\sum_{n\in[N]}\frac{\scale}{\lambda_{\sigma(n)}}\sum_{\ell\in E_n}d_{\sigma(\ell)}+\sum_{n\in[N]}\frac{\scale}{\lambda_{\sigma(n)}}\sum_{\ell>n: \sigma(\ell)>\sigma(n)}d_{\sigma(\ell)}.
	\end{align*}
	Next, we also obtain the following by rearranging terms:
	\begin{align*}
		\sum_{n\in[N]}d_{n}\sum_{\ell\in[n]}\frac{\scale}{\lambda_\ell}&=\sum_{n\in[N]}d_{\sigma(n)}\sum_{\ell\in[\sigma(n)]}\frac{\scale}{\lambda_\ell}\\
		&=\sum_{n\in[N]}d_{\sigma(n)}\left(\frac{\scale}{\lambda_{\sigma(n)}}+\sum_{\sigma(\ell)<\sigma(n):\ell<n}\frac{\scale}{\lambda_{\sigma(\ell)}}+\sum_{ \sigma(\ell)<\sigma(n):\ell>n}\frac{\scale}{\lambda_{\sigma(\ell)}}\right)\\
		&=\sum_{n\in[N]}d_{\sigma(n)}\frac{\scale}{\lambda_{\sigma(n)}}+\sum_{n\in[N]}d_{\sigma(n)}\sum_{\ell<n: \sigma(\ell)<\sigma(n)}\frac{\scale}{\lambda_{\sigma(\ell)}}+\sum_{n\in[N]}d_{\sigma(n)}\sum_{\ell\in E_n}\frac{\scale}{\lambda_{\sigma(\ell)}}.
	\end{align*}
	Therefore, the first sum in~\eqref{eq:hetero-lemma} is equal to 
	\begin{equation*}
		\sum_{n\in[N]}\sum_{\ell\in E_n}\left(\frac{d_{\sigma(\ell)}}{\lambda_{\sigma(n)}} - \frac{d_{\sigma(n)}}{\lambda_{\sigma(\ell)}}\right)\scale + \sum_{n\in[N]}\frac{\scale}{\lambda_{\sigma(n)}}\sum_{\ell>n: \sigma(\ell)>\sigma(n)}d_{\sigma(\ell)}-\sum_{n\in[N]}d_{\sigma(n)}\sum_{\ell<n: \sigma(\ell)<\sigma(n)}\frac{\scale}{\lambda_{\sigma(\ell)}}.
	\end{equation*}
	Notice that the second term in this equation can be rewritten as
	\begin{align*}
		\sum_{n\in[N]}\frac{\scale}{\lambda_{\sigma(n)}}\sum_{\ell>n: \sigma(\ell)>\sigma(n)}d_{\sigma(\ell)}&=\sum_{\ell\in[N]}d_{\sigma(\ell)}\sum_{n<\ell:\sigma(n)<\sigma(\ell)} \frac{\scale}{\lambda_{\sigma(n)}},
	\end{align*}
	and therefore, it is equivalent to the third term of the equation.
	Hence, the first sum in~\eqref{eq:hetero-lemma}  is indeed equal to the second term in~\eqref{eq:hetero-lemma}, as required.
\end{proof}

Recall that we introduced notation
$\bar \scale ={\scale}/{\mu_{\min}}$.
The following lemma is a direct consequence of Lemma~\ref{lemma:gap}.
\begin{lemma}\label{lemma:gap'}
	Under the clean event, the following statements hold.
	\begin{enumerate}[($a$)]
		\item\label{gaps'1} If $i_{\max}\in \R$ and $\RP$ is empty, then for any $i\in \R$,
		$$(c_i\mu_i -c_{i_{\max}}\mu_{i_{\max}})\frac{\scale}{\mu_i\mu_{i_{\max}}} \leq 2\bar\scale \sqrt{\log N\bar\scale}\left(\sqrt{\frac{3}{\sum_{s=1}^tN_{i,s}}}+ \sqrt{\frac{3}{\sum_{s=1}^tN_{{i_{\max}},s}}}\right).$$
		
		\item\label{gaps'2} If $i_{\max}\in \R$ and $\RP$ is nonempty, then for any $i\in \RP$, 
		$$(c_{i_{\max}}\mu_{i_{\max}}-c_i\mu_i)\frac{\scale}{\mu_{i_{\max}}\mu_i} <0.$$
		\item\label{gaps'3} If $i_{\max}\in \R$, $\RP$ is nonempty, and $i\in\arg\max_{i\in \RP}\hat c_{i,t}\mu_i$, then for any $j\in \R$,
		$$(c_j\mu_j- c_i\mu_i)\frac{\scale}{\mu_{j}\mu_i}  \leq 2\bar\scale\sqrt{\log N\bar\scale}\left(\sqrt{\frac{3}{\sum_{s=1}^tN_{j,s}}}+\sqrt{\frac{3}{\sum_{s=1}^tN_{i,s}}}\right).$$
		\item\label{gaps'4} If $\R$ is nonempty and $i\in\arg\max_{i\in \R}\hat c_{i,t}\mu_i$, then for any $j\in \R$,
		$$(c_j\mu_j-c_i\mu_i )\frac{\scale}{\mu_{j}\mu_i} \leq \bar\scale\sqrt{\log N\bar\scale}\left(\sqrt{\frac{3}{\sum_{s=1}^tN_{j,s}}}+\sqrt{\frac{3}{\sum_{s=1}^tN_{i,s}}}\right).$$
	\end{enumerate}
\end{lemma}

\subsection{Proof of Theorem~\ref{thm:hetero-ub}}

Our proof of Theorem~\ref{thm:hetero-ub} is an adaptation of the proof of Theorem~\ref{thm:uniform-ub-2} given in Appendix~\ref{sec:proof:uniform-ub-2}. There are several distinctions to consider. First, a regret bound has dependence on individual $\mu_i$'s.  We will shortly argue 
\begin{equation}\label{eq:regret-bound-hetero}
	R^\pi\leq \barn \Ts + \sum_{n\in[N]}\sum_{\ell\in E_n}\left(d_{\sigma(\ell)}\lambda_{\sigma(\ell)}-d_{\sigma(n)}\lambda_{\sigma(n)}\right)\frac{\scale}{\lambda_{\sigma(\ell)}\lambda_{\sigma(n)}},
\end{equation}
which considers different values for $\lambda_n$'s.
To bound the second sum on the right-hand side of~\eqref{eq:regret-bound-hetero}, we can use Lemma~\ref{lemma:gap'} to bound each term in the second sum. Note that on the right-hand side of inequalities given in Lemma~\ref{lemma:gap'}, we have $\bar\scale = \scale/\mu_{\min}$ on the right-hand side. Lastly, the amount of time required to complete a job is bounded below by $\scale/\mu_{\max}$. In the end, we will provide a regret upper bound which is blown up by a factor of $\sqrt{\mu_{\max}/\mu_{\min}}$ compared to the regret upper bound for the uniform case provided in Theorem~\ref{thm:uniform-ub-2}.

First, let us take care of the edge case where the clean event does not hold. As the policy is work-conserving, all jobs would be completed after $M\scale$ time steps, where $M=\sum_{n\in[N]}1/\lambda_n$, implying in turn that each job stays in the system for at most $M\scale$ time steps. Hence, the worst case holding cost is always bounded from above by $\sum_{n\in[N]}d_n\cdot M\scale$, which is less than or equal to $NM\scale$ as $d_n\leq 1$ for $n\in[N]$. That means that $R^\pi\leq MN\scale$.

Now let us assume that the clean event does hold. Note that Lemma~\ref{lemma:firstphase} holds regardless of whether $\mu_1,\ldots,\mu_I$ are heterogeneous. Then it follows from Lemma~\ref{lemma:firstphase} that
$$\sum_{n\in[N]}d_{\sigma(n)}W_n\leq \sum_{n\in[N]}W_n=\sum_{i\in [N]\setminus\{i_{\max}\}} \sum_{n:\sigma(n)\in\J_i}W_n\leq \sum_{i\in [N]\setminus\{i_{\max}\}}N_i\Ts= \barn \Ts.$$
Together with Lemma~\ref{lemma:regret-1-hetero}, this implies that~\eqref{eq:regret-bound-hetero} holds.

We bound the sum on the right-hand side of~\eqref{eq:regret-bound-hetero}. Recall that~\eqref{hetero-assumption'} implies that no job finishes until the end of $(\Ts+1)$th time slot. We can consider the terms with $n=1$ first and the other terms next. As in the uniform case, the first job $\sigma(1)$ for non-preemptive serving is determined at the beginning of $(\Ts+1)$th time slot. Hence, we can apply the same argument used in the proof of Theorem~\ref{thm:uniform-ub-2} up to~\eqref{eq:1-bound4}, thereby we argue that
$$\sum_{\ell\in E_1}\left(d_{\sigma(\ell)}\lambda_{\sigma(\ell)}-d_{\sigma(1)}\lambda_{\sigma(1)}\right)\frac{\scale}{\lambda_{\sigma(\ell)}\lambda_{\sigma(1)}}\leq 4\barn\bar\scale\sqrt{\frac{3}{N_{\min}(\Ts+1)}\log N\bar\scale},$$
which implies
\begin{equation}\label{eq:regret-bound-hetero-2}
	R^\pi\leq \barn \Ts +4\barn\bar\scale\sqrt{\frac{3}{N_{\min}(\Ts+1)}\log N\bar\scale}+ \sum_{n\geq2}\sum_{\ell\in E_n}\left(d_{\sigma(\ell)}\lambda_{\sigma(\ell)}-d_{\sigma(n)}\lambda_{\sigma(n)}\right)\frac{\scale}{\lambda_{\sigma(\ell)}\lambda_{\sigma(n)}}.
\end{equation}
Next we bound the terms on the right-hand side of~\eqref{eq:regret-bound-hetero-2} with $n\geq2$. For $n\geq 2$, let $t_n$ denote the time when job $\sigma(n)$ is selected by Algorithm~\ref{preempt-then-nonpreemptive-refined} after the preemption phase. At time $t_n$, the jobs $\sigma(1),\ldots,\sigma(n-1)$ have already been completed, we have that $t_n\geq (n-1)\scale/\mu_{\max}$. Based on this, we can argue the following holds as in the proof of Theorem~\ref{thm:uniform-ub-2} up to~\eqref{bound1}.
\begin{align}\label{eq:regret-bound-hetero-4}
	\begin{aligned}
		&\sum_{n\geq 2}\sum_{\ell\in E_n}\left(d_{\sigma(\ell)}\lambda_{\sigma(\ell)}-d_{\sigma(n)}\lambda_{\sigma(n)}\right)\frac{\scale}{\lambda_{\sigma(\ell)}\lambda_{\sigma(n)}}\\
		&\leq2\bar\scale\sqrt{\log N\bar\scale}\sum_{n\geq 2}\sum_{\ell\in E_n:\sigma(\ell)\not\in\J_{i_{\max}}}\sqrt{\frac{3}{\sum_{s=1}^{{(n-1)\scale}/{\mu_{\max}}} N_{\text{class of $\sigma(\ell)$},s}}}\\
		&\quad +2\barn\bar\scale\sqrt{\log N\bar\scale}\sum_{n\geq 2}\sqrt{\frac{3}{\sum_{s=1}^{{(n-1)\scale}/{\mu_{\max}}} N_{\text{class of $\sigma(n)$},s}}}.
	\end{aligned}
\end{align}
Following the argument in the proof of Theorem~\ref{thm:uniform-ub-2} up to~\eqref{bound6}, we can provide a bound on the second term on the right-hand side of~\eqref{eq:regret-bound-hetero-4} as follows.
\begin{align}\label{eq:regret-bound-hetero-6}
	\begin{aligned}
		2\barn\bar\scale\sqrt{\log N\bar\scale}\sum_{n\geq 2}\sqrt{\frac{3}{\sum_{s=1}^{{(n-1)\scale}/{\mu_{\max}}} N_{\text{class of $\sigma(n)$},s}}}
		&\leq \sum_{i\in\I}6\barn\bar\scale\sqrt{\log N\bar\scale}\sum_{k=1}^{\lfloor N_i/2\rfloor}\sqrt{\frac{6\mu_{\max}}{(n_{i_k}-1)\scale N_i}}\\
		&=6\barn\sqrt{\frac{\mu_{\max}}{\mu_{\min}}\bar\scale\log N\bar\scale}\sum_{i\in\I}\sum_{k=1}^{\lfloor N_i/2\rfloor}\sqrt{\frac{6}{(n_{i_k}-1) N_i}}
	\end{aligned}
\end{align}
where $n_{i_1},\ldots, n_{i_{\lfloor N_i/2\rfloor}}$ are the indices such that $\sigma(n_{i_1}),\ldots, \sigma(n_{i_{\lfloor N_i/2\rfloor}})$ are the first $\lfloor N_i/2\rfloor$ jobs in class $i\in\I$ after job $\sigma(1)$.
Next, we turn our attention to the first sum at the right-hand side of inequality~\eqref{eq:regret-bound-hetero-4}. Following the corresponding argument in Theorem~\ref{thm:uniform-ub-2} up to~\eqref{bound11}, we obtain
\begin{align}\label{eq:regret-bound-hetero-7}
	\begin{aligned}
		&2\bar\scale\sum_{n\geq 2}\sum_{\ell\in E_n:\sigma(\ell)\not\in\J_{i_{\max}}}\sqrt{\frac{3}{\sum_{s=1}^{(n-1)\scale/\mu_{\max}} N_{\text{class of $\sigma(\ell)$},s}}}\\
		&\leq \sum_{i\in\I\setminus\{i_{\max}\}}4N_i\bar\scale\sum_{n=2}^{n_{i_{\lfloor N_i/2\rfloor}}}\sqrt{\frac{3\mu_{\max}}{(n-1)\scale N_i}}+ \sum_{i\in\I\setminus\{i_{\max}\}}8\barn\bar\scale\sum_{k=1}^{\lfloor N_i/2\rfloor}\sqrt{\frac{6\mu_{\max}}{(n_{i_k}-1)\scale N_i}}\\
		&=\sqrt{\frac{\mu_{\max}}{\mu_{\min}}\bar\scale}\left(\sum_{i\in\I\setminus\{i_{\max}\}}4N_i\sum_{n=2}^{n_{i_{\lfloor N_i/2\rfloor}}}\sqrt{\frac{3}{(n-1)N_i}}+ \sum_{i\in\I\setminus\{i_{\max}\}}8\barn\sum_{k=1}^{\lfloor N_i/2\rfloor}\sqrt{\frac{6}{(n_{i_k}-1) N_i}}\right)
	\end{aligned}
\end{align}
Combining~\eqref{eq:regret-bound-hetero-4},~\eqref{eq:regret-bound-hetero-6}, and~\eqref{eq:regret-bound-hetero-7}, the third term on the right-hand side of~\eqref{eq:regret-bound-hetero-2} can be bounded above by
\begin{align}\label{eq:regret-bound-hetero-8}
	\begin{aligned}
		&\sum_{n\geq 2}\sum_{\ell\in E_n}\left(d_{\sigma(\ell)}\lambda_{\sigma(\ell)}-d_{\sigma(n)}\lambda_{\sigma(n)}\right)\frac{\scale}{\lambda_{\sigma(\ell)}\lambda_{\sigma(n)}}\\
		&\leq \sqrt{\frac{\mu_{\max}}{\mu_{\min}}\bar\scale\log N\bar \scale}\left( \sum_{i\in\I\setminus\{i_{\max}\}}4 N_i\sum_{n=2}^{n_{i_{\lfloor N_i/2\rfloor}}}\sqrt{\frac{3}{(n-1)N_i}}+ \sum_{i\in\I}14\barn\sum_{k=1}^{\lfloor N_i/2\rfloor}\sqrt{\frac{6}{(n_{i_k}-1)N_i}}\right).
	\end{aligned}    
\end{align}
Following the argument up to~\eqref{final-bound5}, we can show that the two terms on the right-hand side of~\eqref{eq:regret-bound-hetero-8} are bounded above by
\begin{equation}\label{eq:regret-bound-hetero-9}
	\kappa\cdot \sqrt{\frac{\mu_{\max}}{\mu_{\min}}}\left(\sqrt{I N\bar N\bar \scale \log N\bar\scale}+\min\left\{I\barn,\sqrt{I}\barn\sqrt{\log \barn},\frac{\barn^{3/2}}{N_{\min}^{1/2}}\right\}\sqrt{\bar \scale\log N\bar\scale}\right)
\end{equation}
for some constant $\kappa>0$. Since $I\leq \barn+1$, we have $I\leq 2\barn$, which implies that~\eqref{eq:regret-bound-hetero-9} is bounded above by
\begin{equation}\label{eq:regret-bound-hetero-10}
	2\kappa\cdot \sqrt{\frac{\mu_{\max}}{\mu_{\min}}} N^{1/2}\barn\bar\scale^{1/2}(\log N \bar\scale)^{1/2}
\end{equation}
Finally, by~\eqref{eq:regret-bound-hetero-2}, we obtain
\begin{equation}\label{eq:regret-bound-hetero-11}
	\mathbb{E}[R^\pi\mid \text{clean event}]=O\left(\barn\Ts +4\barn\bar\scale\sqrt{\frac{3}{N_{\min}(\Ts+1)}\log N\bar\scale}+\sqrt{\frac{\mu_{\max}}{\mu_{\min}}}N^{1/2}\barn\bar\scale^{1/2}(\log N \bar\scale)^{1/2}\right).
\end{equation}
As in~\eqref{eq:Ts-choice-hetero}, we set $\Ts$ as $\Ts=\lfloor N_{\min}^{-1/3} \bar\scale^{2/3}\left(\log N\bar\scale\right)^{1/3}\rfloor$. If $\Ts\geq 1$, then 
\begin{equation}\label{eq:regret-bound-hetero-12}
	\mathbb{E}[R^\pi\mid \text{clean event}]=O\left(\barn\bar\scale^{2/3}(\log N\bar\scale)^{1/3}+\sqrt{\frac{\mu_{\max}}{\mu_{\min}}}N^{1/2}\barn\bar\scale^{1/2}(\log N \bar\scale)^{1/2}\right)
\end{equation}
since $N_{\min}\geq 1$. If $\Ts=0$, then it means that $N_{\min}^{-1/3}\bar \scale^{2/3}<1$, in which case, $\bar\scale /\sqrt{N_{\min}}<1$. Then it follows from~\eqref{eq:regret-bound-hetero-11} that~\eqref{eq:regret-bound-hetero-12} holds even when $\Ts=0$.

Since $N\geq 2$
We have previously argued that
$$
\mathbb{E}[R^\pi\mid \neg \text{clean event}]= MN\scale\leq N^2\bar\scale.
$$
Then it follows from~\eqref{clean-event:bound} that
$$
\mathbb{E}[R^\pi]= O\left(\barn\bar\scale^{2/3}(\log N\bar\scale)^{1/3}+\sqrt{\frac{\mu_{\max}}{\mu_{\min}}}N^{1/2}\barn\bar\scale^{1/2}(\log N \bar\scale)^{1/2}\right),
$$
as required.

\subsection{Proof of Theorem~\ref{thm:hetero-lb}}

In this section, we prove Theorem~\ref{thm:hetero-lb}. The proof of Theorem~\ref{thm:hetero-lb} is similar to that of Theorem~\ref{thm:uniform-lb}. Let $\I_1$ and $\I_2$ be some nonempty sets partitioning $\I$, the set of all classes. Let $M_1$ and $M_2$ be defined as
\begin{equation}\label{eq:Ms-hetero}
	M_1:=\sum_{i\in \I_1}\frac{N_i}{\mu_i},\quad M_2:=\sum_{i\in \I_2}\frac{N_i}{\mu_i},
\end{equation}
and assume that $M_1\geq M_2$. Let us consider the following family of two problem instances, with parameter $\epsilon>0$ to be decided later:
\begin{equation}\label{lb-instance-1-hetero}
	\mathcal{P}_1 =\begin{cases}
		c_i=\mu_{\min}/2\mu_i&\text{for each class}\ i\in \I_1\\
		c_i = (1+\epsilon)\mu_{\min}/2\mu_i& \text{for each class}\ i\in \I_2
	\end{cases}
\end{equation}
and
\begin{equation}\label{lb-instance-2-hetero}
	\mathcal{P}_2 =\begin{cases}
		c_i=\mu_{\min}/2\mu_i&\text{for each class}\ i\in \I_1\\
		c_i = (1-\epsilon)\mu_{\min}/2\mu_i& \text{for each class}\ i\in \I_2.
	\end{cases}
\end{equation} 
Moreover, we consider an additional problem instance 
$$\mathcal{P}_0 = \left\{c_i = {\mu_{\min}}/{2\mu_i}\quad \text{for every class}\ i\in\I\right..$$ 
For each job $n\in [N]$, define the $t$-round sample space $\Omega_n^t=\{0,\lambda_n\}^t$, where each outcome corresponds to a particular realization of the random cost values $X_{n,1},\ldots, X_{n,t}$ of job $n$ for the first $t$ time steps and $\lambda_n$ is the mean holding cost of job $n$, i.e., $\lambda_n$ is the mean of $X_{n,t}$ for $t\geq 1$. We focus on
$$\Omega = \prod_{n\in[N]}\Omega_n^t$$
so that the random costs of the $N$ jobs for the first $t$ time steps can be considered. For $k\in\{0,1,2\}$, we define distribution $\mathbb{P}_k$ on $\Omega$ as
$$
\mathbb{P}_k[A]= \mathbb{P}[ A \mid \mathcal{P}_k]\quad\text{for each}\ A\subseteq \Omega.
$$
Note that, for $k\in\{0,1,2\}$, $\mathbb{P}_k$ can be expressed as
$$
\mathbb{P}_k=\prod_{i\in[N],s\in[t]}\mathbb{P}_k^{n,s}
$$
where $\mathbb{P}_k^{n,s}$ is the distribution of the random cost of job $n$ at time step $t$. Based on the notion of KL-divergence and Lemma~\ref{KL:variant}, we obtain the following for each event $A\subseteq \Omega$:
\begin{equation}\label{eq:tv-3-hetero}
	2\left(\mathbb{P}_0[A]-\mathbb{P}_k[A]\right)^2\leq \mathrm{KL}(\mathbb{P}_0,\mathbb{P}_k)\leq \sum_{n\in[N]}\sum_{s\in[t]} \mathrm{KL}(\mathbb{P}_0^{n,s},\mathbb{P}_k^{n,s})\leq \mu_{\min}{t\epsilon^2}\sum_{i\in \I_2}\frac{N_i}{\mu_i}=\mu_{\min}M_2t\epsilon^2.
\end{equation}

\begin{theorem}\label{thm:first-lb-hetero}
	Fix any (randomized) scheduling algorithm $\pi$. Choose $k$ from $\{1,2\}$ uniformly at random, and run the algorithm on instance $\mathcal{P}_k$. Assume that $\mu_{\min}^{-1/3}\mu_{\max}^{-2/3}M_2^{-1/3}\scale^{2/3}\geq1$ where $M_2=\sum_{i\in \I_2}N_i/\mu_i$.
	Then
	$$
	\mathbb{E}[R^\pi]=\Omega\left(\mu_{\min}^{4/3}\mu_{\max}^{-2/3}M_2^{2/3}\bar \scale^{2/3}\right)$$
	where the expectation is taken over the choice of $k$ and the randomness in holding costs and the algorithm.
\end{theorem}
\begin{proof}
	We set
	$$T_0=\lfloor \mu_{\min}^{-1/3}\mu_{\max}^{-2/3}M_2^{-1/3}\scale^{2/3}\rfloor\quad\text{and}\quad\epsilon=
	\frac{\mu_{\min}^{-1/3}\mu_{\max}^{1/3}M_2^{-1/3}\scale^{-1/3}}{3}.
	$$
	As $S/\mu_{\max}\geq 1$, it is clear that $\epsilon\leq 1/3$, and therefore, $\epsilon$ is sufficiently small to apply Lemma~\ref{KL:variant}.
	Since $\mu_{\min}^{-1/3}\mu_{\max}^{-2/3}M_2^{-1/3}\scale^{2/3}\geq 1$, we have
	\begin{equation}\label{eq:first-T0-hetero}
		\frac{\mu_{\min}^{-1/3}\mu_{\max}^{-2/3}M_2^{-1/3}\scale^{2/3}}{2}\leq T_0   \leq \mu_{\min}^{-1/3}\mu_{\max}^{-2/3}M_2^{-1/3}\scale^{2/3}.
	\end{equation}
	Then we consider the $T_0$-round sample space $\Omega_n^{T_0}=\{0,1\}^{T_0}$ of each job $n\in[N]$, and we define $\Omega$ as before. Then it follows from~\eqref{eq:tv-3-hetero} that for any event $A\subseteq \Omega$, 
	\begin{equation}\label{eq:KLgap-1-hetero}
		\left|\mathbb{P}[ A \mid \mathcal{P}_0]-\mathbb{P}[ A \mid \mathcal{P}_k]\right|\leq \frac{1}{3}\quad\text{for}~k\in\{1,2\}.
	\end{equation}
	Let $B\subseteq\Omega$ be the event that algorithm $\pi$ chooses a job from some class in $\I_2$ in at least $T_0/2$ time slots until the end of the $T_0$th time slot. Then under $\neg B\subseteq \Omega$, algorithm $\pi$ chooses a job from some class in $\I_1$ in at least $T_0/2$ time slots until the end of the $T_0$th time slot. 
	If $\mathbb{P}[ B \mid \mathcal{P}_0]\geq {1}/{2}$, then by~\eqref{eq:KLgap-1-hetero}, we have $\mathbb{P}[ B \mid \mathcal{P}_k]\geq {1}/{6}$ for $k\in\{1,2\}$. In this case, we obtain
	\begin{equation}\label{eq:first-lb-case1-hetero}
		\mathbb{E}[R^\pi]\geq\mathbb{P}[\mathcal{P}_2]\cdot \mathbb{P}[B\mid \mathcal{P}_2]\cdot\mathbb{E}[ R^\pi \mid B, \mathcal{P}_2]\geq \frac{1}{12}\mathbb{E}[ R^\pi \mid B, \mathcal{P}_2].
	\end{equation}
	If not, we have $\mathbb{P}[ \neg B \mid \mathcal{P}_0]\geq {1}/{2}$, and therefore, $\mathbb{P}[ \neg B \mid \mathcal{P}_k]\geq {1}/{6}$ for $k\in\{1,2\}$ by~\eqref{eq:KLgap-1-hetero}. In this case, we similarly obtain
	\begin{equation}\label{eq:first-lb-case2-hetero}
		\mathbb{E}[R^\pi]\geq\frac{1}{12}\mathbb{E}[ R^\pi \mid \neg B, \mathcal{P}_1].
	\end{equation}
	By~\eqref{eq:first-lb-case1-hetero} and~\eqref{eq:first-lb-case2-hetero}, it is sufficient to bound the terms $\mathbb{E}[ R^\pi \mid B, \mathcal{P}_2]$ and $\mathbb{E}[ R^\pi \mid \neg B, \mathcal{P}_1]$.

	Let $\sigma:[N]\to[N]$ be the permutation of $[N]$ that gives the sequence of jobs completed by the algorithm. Next, let $T_n$ denote the number of time steps where job $n$ is processed by the scheduling algorithm during the period of the first $T_0$ time steps. Notice that $T_0\leq \scale$, so no job finishes until the $T_0$th time slot. This means that $T_0-\sum_{\ell=1}^n T_{\sigma(\ell)}$ time slots are used to serve jobs other than $\sigma(1),\ldots,\sigma(n)$, and therefore, $W_n\geq T_0-\sum_{\ell=1}^n T_{\sigma(\ell)}$. Then it follows from  Lemma~\ref{lemma:regret-1-hetero} that
	\begin{equation}\label{eq:lb-first-bound1-hetero}
		R^\pi\geq\sum_{n\in[N]}d_{\sigma(n)}\left(T_0 -\sum_{\ell=1}^{n}T_{\sigma(\ell)} \right)+ \sum_{n\in[N]}\sum_{\ell\in E_n}\left(d_{\sigma(\ell)}\lambda_{\sigma(\ell)}-d_{\sigma(n)}\lambda_{\sigma(n)}\right)\frac{\scale}{\lambda_{\sigma(\ell)}\lambda_{\sigma(n)}}.
	\end{equation}
	
	Consider the case where we are under the instance $\mathcal{P}_2$ and the event $B$. Let $n_2$ be the smallest number such that $\sigma(n_2)$ belongs to a class in $\I_2$. Then $d_{\sigma(n_2)}\lambda_{\sigma(n_2)}=(1-\epsilon)\mu_{\min}/2$ and $d_{\sigma(\ell)}\lambda_{\sigma(\ell)}=\mu_{\min}/2$ for all $\ell\in E_{n_2}$. Then it follows from~\eqref{eq:lb-first-bound1-hetero} that
	\begin{equation}\label{eq:lb-first-bound2-hetero'}
		R^\pi\geq\sum_{\ell\in E_{n_2}}\left(d_{\sigma(\ell)}\lambda_{\sigma(\ell)}-d_{\sigma(n_2)}\lambda_{\sigma(n_2)}\right)\frac{\scale}{\lambda_{\sigma(\ell)}\lambda_{\sigma(n_2)}}=\frac{\mu_{\min}\epsilon S}{2\lambda_{\sigma(n_2)}}\sum_{\ell\in E_{n_2}}\frac{1}{\lambda_{\sigma(\ell)}}.
	\end{equation}
	If $\sum_{\ell \in E_{n_2}} 1/\lambda_{\sigma(\ell)}\geq M_1/2$, then as $\lambda_{\sigma(n_2)}\leq \mu_{\max}$, we obtain from~\eqref{eq:lb-first-bound2-hetero'} that
	\begin{equation}\label{eq:lb-first-bound2-hetero}
		R^\pi\geq \frac{\mu_{\min}}{4\mu_{\max}}\epsilon S M_1=\frac{\mu_{\min}^{2/3}}{12\mu_{\max}^{2/3}} M_1M_2^{-1/3}S^{2/3}\geq \frac{\mu_{\min}^{4/3}}{12\mu_{\max}^{2/3}}M_2^{2/3}\bar\scale^{2/3}.
	\end{equation}
	If $\sum_{\ell \in E_{n_2}} 1/\lambda_{\sigma(\ell)}<M_1/2$, then as we are under $\mathcal{P}_2$ and $\sigma(n)$ for $n<n_2$ belongs to a class in $\mathcal{I}_1$ by our choice of $n_2$, it follows that
	\begin{equation}\label{eq:lb-first-bound3'}
		\sum_{n<n_2}d_{\sigma(n)}=\sum_{n<n_2}\frac{\mu_{\min}}{2\lambda_{\sigma(n)}}=\frac{\mu_{\min}}{2}\left(\sum_{i\in\mathcal{I}_1}\frac{N_i}{\mu_i}-\sum_{\ell \in E_{n_2}} \frac{1}{\lambda_{\sigma(\ell)}}\right)>\frac{\mu_{\min}}{4}M_1
	\end{equation}
	Moreover, since we are under the event $B$, 
	$\sum_{i\in \I_1}\sum_{n\in\J_i}T_{n}\leq{T_0}/{2}$.
	Note that for $n<n_2$, $\sigma(n)$ belongs to a class in $\mathcal{I}_1$ by the choice of $n_2$. This implies that for $n<n_2$,
	$$T_0-\sum_{\ell=1}^nT_{\sigma(\ell)}\geq T_0-\sum_{i\in \I_1}\sum_{n\in\J_i}T_{n}\geq \frac{T_0}{2}.$$
	Hence, from~\eqref{eq:lb-first-bound1-hetero} and~\eqref{eq:lb-first-bound3'}, we obtain
	\begin{equation}\label{eq:lb-first-bound3-hetero}
		R^\pi\geq \sum_{n<n_2}d_{\sigma(n)}\left(T_0-\sum_{\ell=1}^nT_{\sigma(\ell)}\right)\geq \frac{\mu_{\min}}{8}T_0 M_1\geq \frac{\mu_{\min}^{2/3}}{16\mu_{\max}^{2/3}}M_1 M_2^{-1/3}S^{2/3}\geq \frac{\mu_{\min}^{4/3}}{16\mu_{\max}^{2/3}}M_2^{2/3}\bar \scale^{2/3}.
	\end{equation}
	where the last inequality is from~\eqref{eq:first-T0-hetero}.
	Based on~\eqref{eq:lb-first-bound2-hetero} and~\eqref{eq:lb-first-bound3-hetero}, we get
	\begin{equation}\label{eq:lb-first-bound4-hetero}
		\mathbb{E}\left[R^\pi \mid B,\mathcal{P}_2\right]\geq \frac{\mu_{\min}^{4/3}}{16\mu_{\max}^{2/3}}M_2^{2/3}\bar \scale^{2/3}.
	\end{equation}
	
	Next assume that we are under the instance $\mathcal{P}_1$ and the event $\neg B$. Let ${n_1}$ be the smallest number such that $\sigma({n_1})$ belongs to a class in $\I_1$. Then it follows from~\eqref{eq:lb-first-bound1-hetero} that
	\begin{equation}\label{eq:lb-first-bound5-hetero'}
		R^\pi\geq\sum_{\ell\in E_{n_1}}\left(d_{\sigma(\ell)}\lambda_{\sigma(\ell)}-d_{\sigma(n_1)}\lambda_{\sigma(n_1)}\right)\frac{\scale}{\lambda_{\sigma(\ell)}\lambda_{\sigma(n_1)}}=\frac{\mu_{\min}\epsilon S}{2\lambda_{\sigma(n_1)}}\sum_{\ell\in E_{n_1}}\frac{1}{\lambda_{\sigma(\ell)}}.
	\end{equation}
	If $\sum_{\ell \in E_{n_1}} 1/\lambda_{\sigma(\ell)}\geq M_2/2$,
	\begin{equation}\label{eq:lb-first-bound5-hetero}
		R^\pi\geq\frac{\mu_{\min}}{4\mu_{\max}}\epsilon S M_2=\frac{\mu_{\min}^{2/3}}{12\mu_{\max}^{2/3}} M_2^{2/3}S^{2/3}\geq \frac{\mu_{\min}^{4/3}}{12\mu_{\max}^{2/3}}M_2^{2/3}\bar\scale^{2/3}.
	\end{equation}
	If $\sum_{\ell \in E_{n_1}} 1/\lambda_{\sigma(\ell)}<M_2/2$,
	then as we are under $\mathcal{P}_1$ and $\sigma(n)$ for $n<n_1$ belongs to a class in $\mathcal{I}_2$ by our choice of $n_1$, it follows that
	\begin{equation}\label{eq:lb-first-bound6'-hetero}
		\sum_{n<n_1}d_{\sigma(n)}=\sum_{n<n_1}\frac{\mu_{\min}}{2\lambda_{\sigma(n)}}=\frac{\mu_{\min}}{2}\left(\sum_{i\in\mathcal{I}_2}\frac{N_i}{\mu_i}-\sum_{\ell \in E_{n_1}} \frac{1}{\lambda_{\sigma(\ell)}}\right)>\frac{\mu_{\min}}{4}M_2.
	\end{equation}
	Furthermore, as we are under the event $\neg B$, 
	$\sum_{i\in \I_2}\sum_{n\in\J_i}T_{n}\leq{T_0}/{2}$.
	Note that for $n<n_1$,
	$$T_0-\sum_{\ell=1}^nT_{\sigma(\ell)}\geq T_0-\sum_{i\in \I_2}\sum_{n\J_i}T_{n}\geq \frac{T_0}{2}$$
	because $\sigma(n)$ for $n<n_1$ belongs to a class in $\mathcal{I}_2$ by the choice of $n_1$.
	Therefore, we obtain from~\eqref{eq:lb-first-bound1-hetero} that
	\begin{equation}\label{eq:lb-first-bound6-hetero}
		R^\pi\geq \sum_{n<n_1}d_{\sigma(n)}\left(T_0-\sum_{\ell=1}^nT_{\sigma(\ell)}\right)\geq \frac{\mu_{\min}}{8}T_0 M_2\geq\frac{\mu_{\min}^{4/3}}{16\mu_{\max}^{2/3}}M_2^{2/3}\bar \scale^{2/3}.
	\end{equation}
	where the last inequality is from~\eqref{eq:first-T0-hetero}.
	Based on~\eqref{eq:lb-first-bound5-hetero} and~\eqref{eq:lb-first-bound6-hetero},
	\begin{equation}\label{eq:lb-first-bound7-hetero}
		\mathbb{E}\left[R^\pi \mid \neg B,\mathcal{P}_1\right]\geq \frac{\mu_{\min}^{4/3}}{16\mu_{\max}^{2/3}}M_2^{2/3}\bar \scale^{2/3}.
	\end{equation}
	By~\eqref{eq:first-lb-case1-hetero},~\eqref{eq:first-lb-case2-hetero},~\eqref{eq:lb-first-bound4-hetero}, and~\eqref{eq:lb-first-bound7-hetero}, we have finally proved that
	$\mathbb{E}[R^\pi]=\Omega\left(\mu_{\min}^{4/3}\mu_{\max}^{-2/3}M_2^{2/3}\bar \scale^{2/3}\right)$,
	as required.
\end{proof}

We next provide the second lower bound. 

\begin{theorem}\label{thm:second-lb-hetero}
	Fix any (randomized) scheduling algorithm $\pi$. Choose $k$ from $\{1,2\}$ uniformly at random, and run the algorithm on instance $\mathcal{P}_k$. Let $M_1=\sum_{i\in \I_1}N_i/\mu_i$ and $M_2=\sum_{i\in \I_2}N_i/\mu_i$.
	Then
	$$
	\mathbb{E}[R^\pi]=\Omega\left(\mu_{\min}M_1^{1/2}M_2^{1/2}\bar \scale^{1/2}\right)$$
	where the expectation is taken over the choice of $k$ and the randomness in holding costs and the algorithm.
\end{theorem}
\begin{proof}
	Without loss of generality, assume that $M_1\geq M_2$. We set 
	$$\epsilon=\frac{\mu_{\min}^{-1/2}M_1^{-1/2}M_2^{-1/2}\scale^{-1/2}}{2}.$$
	We consider
	$$T_0=\lfloor\frac{M_1\scale}{2}\rfloor$$
	and the $T_0$-round sample space $\Omega_n^{T_0}=\{0,1\}^{T_0}$ of each job $n\in[N]$, and we define $\Omega$ as before. With our choice of $\epsilon$ and $T_0$, it follows from~\eqref{eq:tv-3-hetero} that for any event $A\subseteq \Omega$, 
	\begin{equation}\label{eq:KLgap-hetero}
		\left|\mathbb{P}[ A \mid \mathcal{P}_0]-\mathbb{P}[ A \mid \mathcal{P}_k]\right|\leq \frac{1}{4}\quad\text{for}~k\in\{1,2\}
	\end{equation}
	Let $B\subseteq\Omega$ be the event that algorithm $\pi$ chooses a job from classes in $\I_2$ in at least $M_2\scale/4$ time slots until the end of the $T_0$th time slot. Then under $\neg B\subseteq \Omega$, algorithm $\pi$ chooses a job from classes in $\I_2$ in at most $M_2\scale/4$ time slots until the end of the $T_0$th time slot. 
	If $\mathbb{P}[ B \mid \mathcal{P}_0]\geq {1}/{2}$, then by~\eqref{eq:KLgap-hetero}, we have $\mathbb{P}[ B \mid \mathcal{P}_k]\geq {1}/{4}$ for $k\in\{1,2\}$. If not, we have $\mathbb{P}[ \neg B \mid \mathcal{P}_0]\geq {1}/{2}$, and therefore, $\mathbb{P}[ \neg B \mid \mathcal{P}_k]\geq {1}/{4}$ for $k\in\{1,2\}$ by~\eqref{eq:KLgap-hetero}. Therefore, we know that one of $\mathbb{P}[ B \mid \mathcal{P}_k]\geq {1}/{4}$ and $\mathbb{P}[ \neg B \mid \mathcal{P}_k]\geq {1}/{4}$ must hold, implying in turn that
	\begin{equation}\label{eq:second-lb-cases-hetero}
		\mathbb{E}[R^\pi]\geq\frac{1}{8}\mathbb{E}[ R^\pi \mid B, \mathcal{P}_2]~~\text{or}~~\mathbb{E}[R^\pi]\geq\frac{1}{8}\mathbb{E}[ R^\pi \mid \neg B, \mathcal{P}_1]
	\end{equation}
	since $\mathbb{P}[\mathcal{P}_1]=\mathbb{P}[\mathcal{P}_2]=1/2$.
	
	We first consider the case where $M_1=1$ and $\scale =1$. Since $M_1\geq M_2$, we also have $M_2=1$. In this case, there are precisely 2 jobs in the system, and the service time of each job is just 1. This means that $\mu_1=\mu_2=\mu_{\min}=1$. Under the event $B$ and instance $\mathcal{P}_2$, the algorithm serves the job of mean holding cost $(1-\epsilon)/2$ and then the job of mean holding cost $1/2$ next, but the optimal sequence is the opposite. Hence, we obtain
	$$\mathbb{E}[ R^\pi \mid B, \mathcal{P}_2]=\left(\frac{1-\epsilon}{2} + \frac{1}{2}\cdot 2\right)-\left(\frac{1}{2} + \frac{1-\epsilon}{2}\cdot 2\right)=\frac{\epsilon}{2}=\frac{1}{4}.$$
	Similarly, under the event $\neg B$ and instance $\mathcal{P}_1$, the algorithm serves the job of mean holding cost $1/2$ and then the job of mean holding cost $(1+\epsilon)/2$ next. Therefore,
	$$\mathbb{E}[ R^\pi \mid \neg B, \mathcal{P}_1]=\left(\frac{1}{2} + \frac{1+\epsilon}{2}\cdot 2\right)-\left(\frac{1+\epsilon}{2} + \frac{1}{2}\cdot 2\right)=\frac{\epsilon}{2}=\frac{1}{4}.$$
	Since $\mu_{\min}=M_1=M_2=\scale =1$, we have $\mu_{\min}M_1^{1/2}M_2^{1/2} \scale^{1/2}=1$. Thus we may assume that $M_1\geq 2$ or $\scale\geq 2$, so $M_1 \scale\geq 2$. This means that
	\begin{equation}\label{eq:second-lb-T0-hetero}
		\frac{1}{3}M_1\scale \leq T_0 \leq \frac{1}{2} M_1\scale.
	\end{equation}

	Consider the case where we are under the instance $\mathcal{P}_2$ and the event $B$. Let $n_1$ be the smallest number such that  $\sigma(n_1)$ belongs to a class in $\mathcal{I}_1$ and 
	\begin{equation}\label{eq:second-lb-condition1-hetero}
		\sum_{n\leq n_1:\sigma(n)\text{ is of a class in }\mathcal{I}_1}\frac{1}{\lambda_{\sigma(n)}}>\frac{M_1}{2}.
	\end{equation}
	By our choice of $n_1$, 
	$$
	\sum_{n<n_1:\sigma(n)\text{ is of a class in }\mathcal{I}_1}\frac{1}{\lambda_{\sigma(n)}}\leq\frac{M_1}{2},$$
	implying in turn that
	\begin{equation}\label{eq:second-lb-condition1'-hetero}
		\sum_{\ell\geq n_1:\sigma(n)\text{ is of a class in }\mathcal{I}_1}\frac{1}{\lambda_{\sigma(n)}}= M_1-\sum_{n<n_1:\sigma(n)\text{ is of a class in }\mathcal{I}_1}\frac{1}{\lambda_{\sigma(n)}}\geq \frac{M_1}{2},
	\end{equation}
	Then we obtain
	\begin{align}\label{eq:lb-second-bound2-hetero}
		\begin{aligned}
			R^{\pi}&\geq \sum_{n<n_1:\sigma(n)\text{ is of a class in }\mathcal{I}_2}\sum_{\ell\in E_n}\left(d_{\sigma(\ell)}\lambda_{\sigma(\ell)}-d_{\sigma(n)}\lambda_{\sigma(n)}\right)\frac{\scale}{\lambda_{\sigma(\ell)}\lambda_{\sigma(n)}}\\
			&\geq \sum_{n<n_1:\sigma(n)\text{ is of a class in }\mathcal{I}_2}\frac{\mu_{\min}\epsilon S}{2\lambda_{\sigma(n)}}\sum_{\ell\in E_n}\frac{1}{\lambda_{\sigma(\ell)}}\\
			&\geq \sum_{n<n_1:\sigma(n)\text{ is of a class in }\mathcal{I}_2}\frac{\mu_{\min}\epsilon S}{2\lambda_{\sigma(n)}}\sum_{\ell\geq n_1:\sigma(\ell)\text{ is of a class in }\mathcal{I}_1}\frac{1}{\lambda_{\sigma(\ell)}}\\
			&\geq \sum_{n<n_1:\sigma(n)\text{ is of a class in }\mathcal{I}_2}\frac{\mu_{\min}\epsilon S}{2\lambda_{\sigma(n)}}\cdot \frac{M_1}{2}\\
			&=\frac{\mu_{\min}\epsilon S M_1}{4}\sum_{n<n_1:\sigma(n)\text{ is of a class in }\mathcal{I}_2}\frac{1}{\lambda_{\sigma(n)}}
		\end{aligned}
	\end{align}
	where the first inequality is from Lemma~\ref{lemma:regret-1-hetero} and the last inequality is implied by~\eqref{eq:second-lb-condition1'-hetero}. If $$\sum_{n<n_1:\sigma(n)\text{ is of a class in }\mathcal{I}_2}\frac{1}{\lambda_{\sigma(n)}}\geq\frac{M_2}{8},$$
	then it follows from~\eqref{eq:lb-second-bound2-hetero} that
	\begin{equation}\label{eq:lb-second-bound3-hetero}
		R^{\pi}\geq \frac{\mu_{\min}\epsilon}{32}M_1M_2S=\frac{\mu_{\min}^{1/2}}{64}M_1^{1/2}M_2^{1/2} S^{1/2}.
	\end{equation}
	If $$\sum_{n<n_1:\sigma(n)\text{ is of a class in }\mathcal{I}_2}\frac{1}{\lambda_{\sigma(n)}}<\frac{M_2}{8},$$
	then as we are under event $B$, at least $M_2S/8$ time slots are used to serve jobs other than $\sigma(1),\ldots, \sigma(n_1)$. This means that for $n\leq n_1$, we have $W_n\geq N_2S/8$, which implies that 
	\begin{equation}\label{eq:lb-second-bound4-hetero}
		R^{\pi}\geq \sum_{n\leq n_1:\sigma(n)\text{ is of a class in } \mathcal{I}_1}d_{\sigma(n)}W_n\geq \sum_{n\leq n_1:\sigma(n)\text{ is of a class in } \mathcal{I}_1}\frac{\mu_{\min}}{2\lambda_{\sigma(n)}}\cdot \frac{M_2S}{8}\geq \frac{\mu_{\min}}{32}M_1M_2S
	\end{equation}
	where the last inequality is due to~\eqref{eq:second-lb-condition1-hetero}.
	Based on~\eqref{eq:lb-second-bound3-hetero} and~\eqref{eq:lb-second-bound4-hetero}, we obtain
	\begin{equation}\label{eq:lb-second-bound5-hetero}
		\mathbb{E}[ R^\pi \mid B, \mathcal{P}_2]\geq \frac{\mu_{\min}^{1/2}}{64}\cdot M_1^{1/2}M_2^{1/2}\scale^{1/2}=\frac{\mu_{\min}}{64}M_1^{1/2}M_2^{1/2} \bar S^{1/2}
	\end{equation}
	since $\mu_{\min},M_1,M_2,S\geq 1$.
	
	Next, assume that we are under the instance $\mathcal{P}_1$ and the event $\neg B$. Let $n_2$ be the number such that $\sigma(n_2)$ is of a class in $\I_2$ and \begin{equation}\label{eq:second-lb-condition2-hetero}
		\sum_{n\leq n_2:\sigma(n)\text{ is of a class in }\mathcal{I}_2}\frac{1}{\lambda_{\sigma(n)}}>\frac{M_2}{2}.
	\end{equation}
	By our choice of $n_2$, 
	$$
	\sum_{n<n_2:\sigma(n)\text{ is of a class in }\mathcal{I}_1}\frac{1}{\lambda_{\sigma(n)}}\leq\frac{M_2}{2},$$
	implying in turn that
	\begin{equation}\label{eq:second-lb-condition2'-hetero}
		\sum_{\ell\geq n_2:\sigma(n)\text{ is of a class in }\mathcal{I}_2}\frac{1}{\lambda_{\sigma(n)}}= M_2-\sum_{n<n_2:\sigma(n)\text{ is of a class in }\mathcal{I}_2}\frac{1}{\lambda_{\sigma(n)}}\geq \frac{M_2}{2},
	\end{equation}
	Then we deduce that
	\begin{align}\label{eq:lb-second-bound6-hetero}
		\begin{aligned}
			R^{\pi}&\geq \sum_{n<n_2:\sigma(n)\text{ is of a class in }\mathcal{I}_1}\sum_{\ell\in E_n}\left(d_{\sigma(\ell)}\lambda_{\sigma(\ell)}-d_{\sigma(n)}\lambda_{\sigma(n)}\right)\frac{\scale}{\lambda_{\sigma(\ell)}\lambda_{\sigma(n)}}\\
			&\geq \sum_{n<n_2:\sigma(n)\text{ is of a class in }\mathcal{I}_1}\frac{\mu_{\min}\epsilon S}{2\lambda_{\sigma(n)}}\sum_{\ell\in E_n}\frac{1}{\lambda_{\sigma(\ell)}}\\
			&\geq \sum_{n<n_2:\sigma(n)\text{ is of a class in }\mathcal{I}_1}\frac{\mu_{\min}\epsilon S}{2\lambda_{\sigma(n)}}\sum_{\ell\geq n_2:\sigma(\ell)\text{ is of a class in }\mathcal{I}_2}\frac{1}{\lambda_{\sigma(\ell)}}\\
			&\geq \sum_{n<n_2:\sigma(n)\text{ is of a class in }\mathcal{I}_1}\frac{\mu_{\min}\epsilon S}{2\lambda_{\sigma(n)}}\cdot \frac{M_2}{2}\\
			&=\frac{\mu_{\min}\epsilon S M_2}{4}\sum_{n<n_2:\sigma(n)\text{ is of a class in }\mathcal{I}_1}\frac{1}{\lambda_{\sigma(n)}}
		\end{aligned}
	\end{align}
	where the first inequality is from Lemma~\ref{lemma:regret-1-hetero} and the last inequality is due to~\eqref{eq:second-lb-condition2'-hetero}. If $$\sum_{n<n_2:\sigma(n)\text{ is of a class in }\mathcal{I}_1}\frac{1}{\lambda_{\sigma(n)}}\geq\frac{M_1}{24},$$
	then it follows from~\eqref{eq:lb-second-bound6-hetero} that
	\begin{equation}\label{eq:lb-second-bound7-hetero}
		R^{\pi}\geq \frac{\mu_{\min}\epsilon}{96}M_1M_2S=\frac{\mu_{\min}^{1/2}}{192}M_1^{1/2}M_2^{1/2} S^{1/2}.
	\end{equation}
	If $$\sum_{n<n_2:\sigma(n)\text{ is of a class in }\mathcal{I}_1}\frac{1}{\lambda_{\sigma(n)}}<\frac{M_1}{24},$$
	then less than $M_1\scale / 24$ time slots are used to complete the jobs from $\I_1$ that are sequenced before job $\sigma(n_2)$. However, we are under the event $\neg B$, so at least $T_0-M_2\scale /4$ time slots are allocated for serving jobs from $\I_1$. Here, we know that
	$$T_0-\frac{M_2\scale}{4}\geq \frac{M_1\scale}{3} - \frac{M_2\scale}{4}\geq \frac{M_1\scale}{12}$$
	where the first inequality is from~\eqref{eq:second-lb-T0-hetero}.
	This in turn implies that at least $M_1\scale /12-M_1\scale / 24=M_1\scale/24$ time slots are used to serve jobs other than the ones before $\sigma(n_2)$. Thus, it follows that $W_{n_2}\geq M_1\scale/24$, which implies that
	\begin{equation}\label{eq:lb-second-bound8-hetero}
		R^{\pi}\geq \sum_{n\leq n_2:\sigma(n)\text{ is of a class in } \mathcal{I}_2}d_{\sigma(n)}W_n\geq \sum_{n\leq n_2:\sigma(n)\text{ is of a class in } \mathcal{I}_2}\frac{\mu_{\min}}{2\lambda_{\sigma(n)}}\cdot \frac{M_1S}{24}\geq \frac{\mu_{\min}}{96}M_1M_2S
	\end{equation}
	where the last inequality is due to~\eqref{eq:second-lb-condition2-hetero}.
	Based on~\eqref{eq:lb-second-bound3-hetero} and~\eqref{eq:lb-second-bound4-hetero}, we obtain
	\begin{equation}\label{eq:lb-second-bound9-hetero}
		\mathbb{E}[ R^\pi \mid \neg B, \mathcal{P}_1]\geq \frac{\mu_{\min}^{1/2}}{192}\cdot M_1^{1/2}M_2^{1/2}\scale^{1/2}=\frac{\mu_{\min}}{192}M_1^{1/2}M_2^{1/2} \bar S^{1/2}
	\end{equation}
	since $\mu_{\min},M_1,M_2,S\geq 1$.
	
	Combining~\eqref{eq:second-lb-cases-hetero}, ~\eqref{eq:lb-second-bound5-hetero}, and~\eqref{eq:lb-second-bound9-hetero}, it follows that $\mathbb{E}[ R^\pi]=\Omega\left(\mu_{\min}M_1^{1/2}M_2^{1/2}\bar \scale^{1/2}\right)$, as required.
\end{proof}

We have proved that
$\Omega\left(\mu_{\min}^{4/3}\mu_{\max}^{-2/3}M_2^{2/3}\bar \scale^{2/3}\right)$
is a lower bound on the expected regret of any (randomized) scheduling algorithm, under the condition that $\mu_{\min}^{-1/3}\mu_{\max}^{-2/3}M_2^{-1/3}\scale^{2/3}\geq 1$. Moreover, $\Omega\left(\mu_{\min}M_1^{1/2}M_2^{1/2}\bar \scale^{1/2}\right)$
is a lower bound on the expected regret by Theorem~\ref{thm:second-lb-hetero}. Let us argue that the second lower bound is stronger than the first one if $\mu_{\min}^{-1/3}\mu_{\max}^{-2/3}M_2^{-1/3}\scale^{2/3}<1$ as we can check from
$$\frac{\mu_{\min}M_1^{1/2}M_2^{1/2}\bar \scale^{1/2}}{\mu_{\min}^{4/3}\mu_{\max}^{-2/3}M_2^{2/3}\bar \scale^{2/3}}= \mu_{\min}^{-1/6}\mu_{\max}^{2/3}M_1^{1/2}M_2^{-1/6}\bar\scale^{-1/6}\geq \mu_{\min}^{-1/6}\mu_{\max}^{2/3}M_2^{1/3}\scale^{-1/6}>\mu_{\min}^{-1/2}S^{1/2}=\bar S^{1/2}$$
where the first inequality is because $M_1\geq M_2$ and the second inequality follows from $\mu_{\min}^{-1/3}\mu_{\max}^{-2/3}M_2^{-1/3}\scale^{2/3}<1$. Therefore, both $\Omega\left(\mu_{\min}^{4/3}\mu_{\max}^{-2/3}M_2^{2/3}\scale^{2/3}\right)$ and $\Omega\left(\mu_{\min}M_1^{1/2}M_2^{1/2}\scale^{1/2}\right)$ are correct lower bounds on the expected regret.

By Lemma~\ref{lemma:lb-best-partition}, there always exists a partition $(\I_1,\I_2)$ such that $\sum_{i\in \I_2}N_i=\Omega(\barn)$ and thus $\sum_{i\in \I_2}N_i/\mu_i=\Omega(\barn/\mu_{\max})$ Moreover, it also implies that there is a partition $(\I_1,\I_2)$ such that $\left(\sum_{i\in \I_1}N_i\right)\left(\sum_{i\in \I_2}N_i\right)=\Omega(N\barn)$ and thus $\left(\sum_{i\in \I_1}N_i/\mu_i\right)\left(\sum_{i\in \I_2}N_i/\mu_i\right)=\Omega(N\barn/\mu_{\max}^2)$. As a result, it follows that 
$$\Omega\left(\max\left\{(\mu_{\max}/\mu_{\min})^{-4/3}\barn^{2/3}\bar\scale^{2/3},\ (\mu_{\max}/\mu_{\min})^{-1}N^{1/2}\barn^{1/2}\bar\scale^{1/2}\right\}\right)$$
is a correct lower bound on the expected regret, as required.

\subsection{Proof of Theorem~\ref{instance:N}}

As in the proof of Theorem~\ref{thm:hetero-ub}, if the clean event does not hold, then $R^\pi\leq MN\scale\leq N^2\bar\scale$. 

Now assume that the clean event holds. Let us consider the case when $\sigma(1)$ is in class 1. Let $\Ts_1$ denote the number of time slots where job $\sigma(1)$ is selected during the preemption period. Then $W_n\leq \Ts-\Ts_1$ for all $n\in[N]$. As Lemma~\ref{lemma:firstphase} holds even when $\mu_1,\ldots,\mu_I$ are heterogeneous, it follows from Lemmas~\ref{lemma:regret-1-hetero} and~\ref{lemma:firstphase} that
\begin{equation}\label{eq:instance-1}
	R^\pi\leq \sum_{i\in [N]\setminus\{i_{\max}\}}N_i(\Ts-\Ts_1)+\sum_{n\in[N]}\sum_{\ell\in E_n}\left(d_{\sigma(\ell)}\lambda_{\sigma(\ell)}-d_{\sigma(n)}\lambda_{\sigma(n)}\right)\frac{\scale}{\lambda_{\sigma(\ell)}\lambda_{\sigma(n)}}.
\end{equation}
To bound the first term on the right-hand side of~\eqref{eq:instance-1}, we will obtain an upper bound on $\Ts-\Ts_1$. Let $t$ be the last time slot where job $\sigma(1)$ is not served among the first $\Ts+1$ time slots. Since job $\sigma(1)$ is of class 1 and is chosen at time $\Ts+1$, $ \Ts-\Ts_1\leq t\leq \Ts$. Let $i$ be the class of the job chosen at time $t$, i.e., $\hat c_{1,t}\mu_1\leq\hat c_{i,t}\mu_i$ while $c_1\mu_1\geq c_i\mu_i$. Since $t\leq \Ts+1$, we have $t\leq \scale/\mu_{\max}$ which means that $N_{i,t}=N_i$ and $N_{1,t}=N_1$. Then, by~\eqref{clean-event:gap},
\begin{equation}\label{eq:instance-2}
	t\leq \frac{(\mu_1+\mu_i)^2}{(c_1\mu_1-c_i\mu_i)^2}\cdot\frac{3}{N_{\min}}\log N\bar\scale\leq \frac{3}{{N_{\min}}\Delta^2}\log N\bar\scale.
\end{equation}
As $\Ts-\Ts_1\leq t$, it follows from~\eqref{eq:instance-1} and~\eqref{eq:instance-2} that
\begin{equation}\label{eq:instance-3}
	R^\pi\leq \frac{3\barn}{{N_{\min}}\Delta^2}\log N\bar\scale+\sum_{n\in[N]}\sum_{\ell\in E_n}\left(d_{\sigma(\ell)}\lambda_{\sigma(\ell)}-d_{\sigma(n)}\lambda_{\sigma(n)}\right)\frac{\scale}{\lambda_{\sigma(\ell)}\lambda_{\sigma(n)}}.
\end{equation}

Next we consider the case when $\sigma(1)$ is not in class 1. Let $i$ be the class of job $\sigma(1)$. Then at time $\Ts_1$, job $\sigma(1)$ is chosen instead of class 1 jobs for, which means that $\hat c_{i,\Ts+1}\mu_{i}\geq \hat c_{1,\Ts+1}\mu_{1}$. Then~\eqref{clean-event:gap} implies that
\begin{equation}\label{eq:instance-4}
	\Ts+1\leq \frac{(\mu_1 +\mu_i)^2}{(c_1\mu_1-c_i\mu_i)^2}\cdot\frac{3}{N_{\min}}\log N\bar\scale\leq\frac{3}{{N_{\min}}\Delta^2}\log N\bar\scale
\end{equation}
because $\Ts+1\leq \scale/\mu_{\max}$ and thus $N_{1,\Ts+1}=N_1\geq N_{\min}$ and $N_{i,\Ts+1}=N_i\geq N_{\min}$. Since~\eqref{eq:regret-bound-hetero} holds,~\eqref{eq:instance-4} implies that~\eqref{eq:instance-3} holds even for the case when $\sigma(1)$ is not in class 1.

Next we consider the second term on the right-hand side of~\eqref{eq:instance-3}. We start by bounding the terms with $n=1$. If $\sigma(1)$ is of class 1, then $E_1$ is empty, so the corresponding sum equals 0. Thus we may assume that $\sigma(1)$ is not in class 1. For this case, we argued earlier that $\Ts+1$ can be bounded as in~\eqref{eq:instance-4}. Recall that $\Ts$ is given in~\eqref{eq:Ts-choice-hetero}. Assume that $\Ts\geq 1$. Then $\Ts=C_1 N_{\min}^{-1/3}\bar \scale^{2/3}\left(\log N\bar\scale\right)^{1/3}$ for some constant $C_1$. Moreover, note that $\hat d_{\sigma(1),\Ts+1}\lambda_{\sigma(1)}\geq\hat d_{\sigma(\ell),\Ts+1}\lambda_{\sigma(\ell)}$ for $\ell\in E_1$, so by~\eqref{clean-event:gap},
\begin{equation}\label{eq:instance-4'}
	\Ts+1\leq \frac{(\lambda_{\sigma(\ell)} +\lambda_{\sigma(1)})^2}{\left(d_{\sigma(\ell)}\lambda_{\sigma(\ell)}-d_{\sigma(1)}\lambda_{\sigma(1)}\right)^2}\cdot\frac{3}{N_{\min}}\log N\bar\scale
\end{equation}
for any $\ell\in E_1$.~\eqref{eq:instance-4'} implies that
\begin{equation}\label{eq:instance-5}
	\bar\scale\leq \frac{(\lambda_{\sigma(\ell)} +\lambda_{\sigma(1)})^3}{\left(d_{\sigma(\ell)}\lambda_{\sigma(\ell)}-d_{\sigma(1)}\lambda_{\sigma(1)}\right)^3}\cdot\frac{C_2}{N_{\min}}\log N\bar\scale
\end{equation}
for any $\ell\in E_1$ for some constant $C_2$.
In that case,
\begin{align}\label{eq:instance-6}
	\begin{aligned}
		&\sum_{\ell\in E_1}\left(d_{\sigma(\ell)}\lambda_{\sigma(\ell)}-d_{\sigma(1)}\lambda_{\sigma(1)}\right)\frac{\scale}{\lambda_{\sigma(\ell)}\lambda_{\sigma(1)}}\\
		&=\sum_{\ell\in E_1:\text{$\sigma(\ell)$ and $ \sigma(1)$ are in different classes}}\left(d_{\sigma(\ell)}\lambda_{\sigma(\ell)}-d_{\sigma(1)}\lambda_{\sigma(1)}\right)\frac{\scale}{\lambda_{\sigma(\ell)}\lambda_{\sigma(1)}}\\
		&\leq\sum_{\ell\in E_1:\text{$\sigma(\ell)$ and $ \sigma(1)$ are in different classes}}
		\frac{(\lambda_{\sigma(\ell)} +\lambda_{\sigma(1)})^2}{\left(d_{\sigma(\ell)}\lambda_{\sigma(\ell)}-d_{\sigma(1)}\lambda_{\sigma(1)}\right)^2}\cdot \frac{(\lambda_{\sigma(\ell)}+\lambda_{\sigma(1)})\mu_{\min}}{\lambda_{\sigma(\ell)}\lambda_{\sigma(1)}}\cdot\frac{C_2}{N_{\min}}\log N\bar\scale\\
		&\leq\sum_{\ell\in E_1:\text{$\sigma(\ell)$ and $ \sigma(1)$ are in different classes}} \frac{1}{\Delta^2}\cdot \frac{2C_2}{N_{\min}}\log N\bar\scale\\
		&\leq \barn\cdot \frac{1}{\Delta^2}\cdot \frac{2C_2}{N_{\min}}\log N\bar\scale
	\end{aligned}
\end{align}
where the first inequality follows from~\eqref{eq:instance-5}, the second inequality is because $(\lambda_{\sigma(\ell)}+\lambda_{\sigma(1)})/\lambda_{\sigma(\ell)}\lambda_{\sigma(1)}\leq 2/\mu_{\min}$, and the thrid inequality is due to Lemma~\ref{lemma:largest}. If $\Ts=0$, then $N_{\min}^{-1/3}\bar \scale^{2/3}\left(\log N\bar\scale\right)^{1/3}\leq 1$, in which case, 
$$N_{\min}^{-1/3}\bar \scale^{2/3}\left(\log N\bar\scale\right)^{1/3}\leq \Ts+1\leq \frac{(\lambda_{\sigma(\ell)} +\lambda_{\sigma(1)})^2}{\left(d_{\sigma(\ell)}\lambda_{\sigma(\ell)}-d_{\sigma(1)}\lambda_{\sigma(1)}\right)^2}\cdot\frac{3}{N_{\min}}\log N\bar\scale$$
by~\eqref{eq:instance-4'}. Then we can similarly argue that~\eqref{eq:instance-5} holds for some constant $C_2$. Therefore, for some sufficiently large constant $C_2$,~\eqref{eq:instance-6} holds even when $\Ts=0$.

Next we consider and bound 
$$\sum_{\ell\in E_n}\left(d_{\sigma(\ell)}\lambda_{\sigma(\ell)}-d_{\sigma(n)}\lambda_{\sigma(n)}\right)\frac{\scale}{\lambda_{\sigma(\ell)}\lambda_{\sigma(n)}}$$ 
for $n\geq 2$. If $\sigma(n)$ is of class 1Let $t$ be the moment when job $\sigma(n)$ is chosen. Since $\sigma(n)$ is selected after jobs $\sigma(1),\ldots,\sigma(n-1)$ are completed, it follows that $t\geq (i-1)\scale/\mu_{\max}$ as the service time of each job is at least $\scale/\mu_{\max}$. Hence, by~\eqref{clean-event:gap}, we have $$(n-1)\frac{\scale}{\mu_{\max}}\leq \frac{3(\lambda_{\lambda(\ell)}+\lambda_{\sigma(n)})^2}{(d_{\sigma(\ell)}\lambda_{\sigma(\ell)}-d_{\sigma(n)}\lambda_{\sigma(n)})^2}\log N\bar\scale,$$ 
implying in turn that
\begin{align*}
	&\sum_{\ell\in E_n}\left(d_{\sigma(\ell)}\lambda_{\sigma(\ell)}-d_{\sigma(n)}\lambda_{\sigma(n)}\right)\frac{\scale}{\lambda_{\sigma(\ell)}\lambda_{\sigma(n)}}\\
	&=\sum_{\ell\in E_n:\text{$\sigma(\ell)$ and $\sigma(n)$ are in different classes}}\left(d_{\sigma(\ell)}\lambda_{\sigma(\ell)}-d_{\sigma(n)}\lambda_{\sigma(n)}\right)\frac{\scale}{\lambda_{\sigma(\ell)}\lambda_{\sigma(n)}}\\
	&\leq \sum_{\ell\in E_n:\text{$\sigma(\ell)$ and $\sigma(n)$ are in different classes}}\frac{1}{n-1}\cdot \frac{\lambda_{\sigma(\ell)}+\lambda_{\sigma(n)}}{d_{\sigma(\ell)}\lambda_{\sigma(\ell)}-d_{\sigma(n)}\lambda_{\sigma(n)}}\cdot\frac{6\mu_{\max}}{\mu_{\min}}\cdot \log N\bar\scale\\
	&\leq\frac{\barn}{n-1}\cdot \frac{\lambda_{\sigma(\ell)}+\lambda_{\sigma(n)}}{d_{\sigma(\ell)}\lambda_{\sigma(\ell)}-d_{\sigma(n)}\lambda_{\sigma(n)}}\cdot\frac{6\mu_{\max}}{\mu_{\min}}\cdot \log N\bar\scale
\end{align*}
where the last inequality follows from Lemma~\ref{lemma:largest}.
Therefore,
\begin{align*}
	\begin{aligned}
		\sum_{n\geq 2}\sum_{\ell\in E_n}\left(d_{\sigma(\ell)}\lambda_{\sigma(\ell)}-d_{\sigma(n)}\lambda_{\sigma(n)}\right)\frac{\scale}{\lambda_{\sigma(\ell)}\lambda_{\sigma(n)}} &\leq \sum_{n\geq2}\frac{\barn}{n-1}\cdot\frac{1}{\Delta}\cdot\frac{6\mu_{\max}}{\mu_{\min}}\cdot \log N\bar\scale\\
		&=O\left(\frac{\mu_{\max}}{\mu_{\min}}\cdot \frac{1}{\Delta}\cdot \barn \log N\log  N\bar \scale\right).
	\end{aligned}
\end{align*}
Therefore, we have just proved the claim.

\subsection{Proof of Theorem~\ref{thm:stochastic}}

As before, we assume that $c_1\mu_1\geq c_2\mu_2\geq \cdots \geq c_I\mu_I$ and order the $N$ jobs from $1$ to $N$ so that jobs $1+\sum_{j\in[i-1]}N_j,\ldots, \sum_{j\in[i]}N_j$ belong to class $i$. Moreover, let $d_n$ denote the mean per-time holding cost of job $n\in[N]$. Then, if job $n$ is of class $i$, then we have $d_n=c_i$. Moreover, we introduce notation $\hat d_{n,t}$ for $n\in[N]$ and $t\geq1$ which is equivalent to $\hat c_{i,t}$ assuming that job $n$ is of class $i$. Let $\sigma:[N]\rightarrow [N]$ be the permutation of $[N]$ that corresponds to the sequence of jobs completed by Algorithm~\ref{preempt-then-nonpreemptive-refined}. Then the order $\sigma$ depends on the random holding costs and stochastic service times of the $N$ jobs.

Let $Y_n$ for denote the stochastic service time of job $n$ for $n\in[N]$. As before, let $W_n$ denote the number of time slots in which Algorithm~\ref{preempt-then-nonpreemptive-refined} serves a job other than $\sigma(1),\ldots,\sigma(n)$ until job $\sigma(n)$ is completed. Then the (random) completion time of job $\sigma(n)$ is given by
$W_n+\sum_{\ell\in[n]}Y_{\sigma(\ell)}$, and therefore, the total completion time equals
\begin{equation}\label{eq:stochastic-1}
	C^{\pi}=\sum_{n\in[N]}d_{\sigma(n)}W_n + \sum_{n\in[N]}d_{\sigma(n)}\sum_{\ell\in[n]}Y_{\sigma(\ell)}.
\end{equation}
Let $t$ be some time slot in which Algorithm~\ref{preempt-then-nonpreemptive-refined} selects a job other than $\sigma(1),\ldots,\sigma(n)$ while job $\sigma(n)$ still waits to be served. Let $k$ be the largest index among $1,\ldots,n$ such that job $\sigma(k)$ finishes until the preemption phase is over. Note that after the preemption phase, Algorithm~\ref{preempt-then-nonpreemptive-refined} gives service to only the jobs $\sigma(k+1),\ldots,\sigma(n)$. Therefore, the time slot $t$ is within the preemption phase, so $t\leq\Ts$. This implies that $W_n\leq \Ts$. Moreover, if $\sigma(n)$ is in class $i_{\max}$ and such time slot $t$ exists, then a job whose class is not $i_{\max}$ gets service at time $t$. However, this implies that the class of the job served at $t$ is in priority class $\mathcal{P}$, and therefore, Algorithm~\ref{preempt-then-nonpreemptive-refined} completes the job before job $\sigma(n)$. This means that if job $\sigma(n)$ is in class $i_{\max}$, no such $t$ exists and thus $W_n=0$. Hence, it follows from~\eqref{eq:stochastic-1} that
\begin{equation}\label{eq:stochastic-2}
	C^{\pi}\leq \barn \tau + \sum_{n\in[N]}d_{\sigma(n)}\sum_{\ell\in[n]}Y_{\sigma(\ell)}.
\end{equation}

Note that once permutation $\sigma$ is fixed, then the distribution of $Y_s$ is determined. In particular,
$$\mathbb{E}\left[ \sum_{n\in[N]}d_{\sigma(n)}\sum_{\ell\in[n]}Y_{\sigma(\ell)}\mid \sigma\right]=\sum_{n\in[N]}d_{\sigma(n)}\sum_{\ell\in[n]}\frac{\scale}{\lambda_{\sigma(\ell)}}.$$
Then by the law of iterated expectations, it follows that
$$\mathbb{E}\left[ \sum_{n\in[N]}d_{\sigma(n)}\sum_{\ell\in[n]}Y_{\sigma(\ell)}\right]=\mathbb{E}\left[\mathbb{E}\left[ \sum_{n\in[N]}d_{\sigma(n)}\sum_{\ell\in[n]}Y_{\sigma(\ell)}\mid \sigma \right]\right]=\mathbb{E}\left[\sum_{n\in[N]}d_{\sigma(n)}\sum_{\ell\in[n]}\frac{\scale}{\lambda_{\sigma(\ell)}}\right].$$
Then it follows from~\eqref{eq:stochastic-2} that
\begin{align}\label{eq:stochastic-3}
	\begin{aligned}
		\mathbb{E}\left[C^{\pi}\right]\leq\mathbb{E}\left[\barn \tau + \sum_{n\in[N]}d_{\sigma(n)}\sum_{\ell\in[n]}Y_{\sigma(\ell)}\right]=\mathbb{E}\left[\barn \tau + \sum_{n\in[N]}d_{\sigma(n)}\sum_{\ell\in[n]}\frac{\scale}{\lambda_{\sigma(\ell)}}\right].
	\end{aligned}
\end{align}
Since serving jobs in the order of $1,\ldots,N$ without preemption is optimal, $\sum_{n\in[N]}d_n\sum_{\ell\in[n]}\scale/\lambda_\ell$ is the minimum expected cumulative holding cost. Then~\eqref{eq:stochastic-3} implies that
\begin{align}\label{eq:stochastic-4}
	\begin{aligned}
		\mathbb{E}\left[R^{\pi}\right]&= \mathbb{E}\left[C^{\pi}\right]-\sum_{n\in[N]}d_n\sum_{\ell\in[n]}\frac{\scale}{\lambda_\ell}\\
		&\leq\mathbb{E}_{\sigma}\left[\barn \tau + \sum_{n\in[N]}d_{\sigma(n)}\sum_{\ell\in[n]}\frac{\scale}{\lambda_{\sigma(\ell)}}-\sum_{n\in[N]}d_n\sum_{\ell\in[n]}\frac{\scale}{\lambda_\ell}\right]
	\end{aligned}
\end{align}
\eqref{eq:stochastic-4} implies that to bound $\mathbb{E}\left[R^{\pi}\right]$, it suffices to consider 
$$\barn \tau + \sum_{n\in[N]}d_{\sigma(n)}\sum_{\ell\in[n]}\frac{\scale}{\lambda_{\sigma(\ell)}}-\sum_{n\in[N]}d_n\sum_{\ell\in[n]}\frac{\scale}{\lambda_\ell}$$
for a fixed permutation $\sigma$.
By Lemma~\ref{lemma:regret-1-hetero}, \begin{equation}\label{eq:stochastic-5}
	\sum_{n\in[N]}\left(d_{\sigma(n)}\sum_{\ell\in[n]}\frac{\scale}{\lambda_{\sigma(n)}}-d_{n}\sum_{\ell\in[n]}\frac{\scale}{\lambda_\ell}\right)=\sum_{n\in[N]}\sum_{\ell\in E_n}\left(d_{\sigma(\ell)}\lambda_{\sigma(\ell)}-d_{\sigma(n)}\lambda_{\sigma(n)}\right)\frac{\scale}{\lambda_{\sigma(\ell)}\lambda_{\sigma(n)}}
\end{equation}
where $E_n$ is defined as in~\eqref{eq:E_n} for $n\in[N]$. Moreover, by Lemma~\ref{lemma:largest} and~\eqref{eq:stochastic-5}, 
\begin{align}\label{eq:stochastic-6}
	\begin{aligned}
		&\sum_{n\in[N]}\left(d_{\sigma(n)}\sum_{\ell\in[n]}\frac{\scale}{\lambda_{\sigma(n)}}-d_{n}\sum_{\ell\in[n]}\frac{\scale}{\lambda_\ell}\right)\\
		&=\sum_{n\in[N]}\sum_{\ell\in E_n:\sigma(\ell)\not\in\J_{i_{\max}}}\left(d_{\sigma(\ell)}\lambda_{\sigma(\ell)}-d_{\sigma(n)}\lambda_{\sigma(n)}\right)\frac{\scale}{\lambda_{\sigma(\ell)}\lambda_{\sigma(n)}}
	\end{aligned}
\end{align}

We also slightly modify the definition of clean event so that we can focus on only the first $\Ts$ time slots. We say that the clean holds when the following condition is satisfied:
\begin{equation*}
	\left|c_i - \frac{1}{m}\sum_{s=1}^mX_{i,s}\right| \leq x_m \ \text{for} \ m\in[N\Ts]\ \text{and for} \ i\in \I,\ \text{where}\ x_m=\sqrt{\frac{3}{m}\log N\bar\scale}.
\end{equation*}
Since $X_{i,m}$ for all $m$ are sub-Gaussian with parameter 1 with mean $c_i$, by Hoeffding's inequality, 
$$\mathbb{P}\left[\left|c_i - \frac{1}{m}\sum_{s=1}^mX_{i,s}\right| > x_m\right]\leq 2\exp(-2m x_m^2)$$
for any $m\geq 1$ and $x_m>0$ since $c_i\in[0,1]$. Then we obtain the following by using the union bound:
\begin{align}\label{clean-event-stochastic}
	\begin{aligned}
		\mathbb{P}\left [\text{clean event}\right] &=\mathbb{P}\left[\left|c_i - \frac{1}{m}\sum_{s=1}^mX_{i,s}\right| \leq x_m\ \text{for} \ m\in[N\Ts]\ \text{and for} \ i\in \I\right]\\&\geq 1-\sum_{i\in \I}\sum_{m\in[N\Ts]}\mathbb{P}\left[\left|c_i - \frac{1}{m}\sum_{s=1}^mX_{i,s}\right| > x_m\right]\\
		&\geq 1-2I\sum_{m\in[N\Ts]}\exp(-2m x_m^2)\\
		&=1-2I\cdot N\Ts\cdot \frac{1}{N^6\bar\scale^6}\\
		&\geq 1-\frac{2}{{N^3\bar\scale^5}}
	\end{aligned}
\end{align}
where the last inequality follows from $I\leq N$ and $\Ts\leq N\bar\scale$.
Hence, under the clean event, we have
$$
c_i\mu_i\in\left[\hat c_{i,t}\mu_i-\mu_i\sqrt{\frac{3}{\Ts+1}\log N\bar\scale},\ \hat c_{i,t}\mu_i+\mu_i\sqrt{\frac{3}{\Ts+1}\log N\bar\scale}\right]$$
for all $i\in \I$ and $t\in[\Ts+1]$.
Now consider two classes $i$ and $j$ such that $c_i\mu_i\geq c_j\mu_j$. If clean event holds, $\hat c_{i,t}\mu_i\leq \hat c_{j,t}\mu_j$, and there is at least one remaining job in each of classes $i$ and $j$, then
\begin{equation}\label{eq:stochastic-7}
	c_i\mu_i-c_j\mu_j \leq  (\mu_i+\mu_j)\sqrt{\frac{3}{\Ts+1}\log N\bar\scale}.
\end{equation}

We first consider the case where the clean event holds. Then it follows from~\eqref{eq:stochastic-6} and~\eqref{eq:stochastic-7} that
\begin{align}\label{eq:stochastic-8}
	\begin{aligned}
		&\sum_{n\in[N]}\left(d_{\sigma(n)}\sum_{\ell\in[n]}\frac{\scale}{\lambda_{\sigma(n)}}-d_{n}\sum_{\ell\in[n]}\frac{\scale}{\lambda_\ell}\right)\\
		&=\sum_{n\in[N]}\sum_{\ell\in E_n:\sigma(\ell)\not\in\J_{i_{\max}}}\frac{(\lambda_{\sigma(\ell)}+\lambda_{\sigma(n)})\scale}{\lambda_{\sigma(\ell)}\lambda_{\sigma(n)}}\cdot \sqrt{\frac{3}{\Ts+1}\log N\bar\scale}\\
		&\leq \sum_{n\in[N]}\sum_{\ell\in E_n:\sigma(\ell)\not\in\J_{i_{\max}}}2\bar \scale \sqrt{\frac{3}{\Ts+1}\log N\bar\scale}\\
		&\leq N\cdot \barn\cdot 2\bar \scale \sqrt{\frac{3}{\Ts+1}\log N\bar\scale}.
	\end{aligned}
\end{align}
Then, by~\eqref{eq:stochastic-4} and~\eqref{eq:stochastic-8},
$$
\mathbb{E}\left[R^\pi\mid  \text{clean event}\right]\leq \barn \Ts+ 2N\barn\bar \scale \sqrt{\frac{3}{\Ts+1}\log N\bar\scale}.$$
By our choice of $\Ts=\lfloor N^{2/3} \bar\scale^{2/3}(\log N\bar\scale)^{1/3}\rfloor$,
\begin{equation}\label{eq:stochastic-9}
	\mathbb{E}\left[R^\pi\mid  \text{clean event}\right]=O(N^{2/3}\barn \bar\scale^{2/3}(\log N\bar\scale)^{1/3}).
\end{equation}

Furtherfore, it is straightforward that
$$\mathbb{E} \left[R^\pi\mid \neg \text{clean event} \right] \leq N^2\bar\scale$$
since the longest expected service time of a job is $\bar\scale$ and there are initially $N$ jobs. Therefore,
\begin{align*}
	\mathbb{E} \left[R^\pi\right] &=\mathbb{P}\left [\text{clean event}\right]\cdot \mathbb{E} \left[R^\pi\mid \text{clean event} \right]+ \mathbb{P}\left [\neg\text{clean event}\right]\cdot \mathbb{E} \left[R^\pi\mid \neg \text{clean event} \right]\\
	&\leq  \mathbb{E} \left[R^\pi\mid \text{clean event} \right]+\frac{2}{N^3\scale^5}\cdot N^2\bar\scale\\
	&=O(N^{2/3}\barn \bar\scale^{2/3}(\log N\bar\scale)^{1/3}),
\end{align*}
implying in turn that
$$\mathbb{E} \left[R^\pi \right]= O(N^{2/3}\barn \bar\scale^{2/3}(\log N\bar\scale)^{1/3}),$$
as required.

\end{document}